\documentclass{article}

 \usepackage[main, final]{neurips_2026}

\usepackage[utf8]{inputenc} 
\usepackage[T1]{fontenc}    
\usepackage[backref=page, colorlinks = true, linkcolor = mydarkblue, citecolor = mydarkblue, urlcolor = mydarkblue]{hyperref}       
\usepackage{url}            
\usepackage{booktabs}       
\usepackage{amsfonts}       
\usepackage{nicefrac}       
\usepackage{microtype}      
\usepackage{xcolor}         
\usepackage{changes}

\usepackage{amsthm,amsmath,amssymb}
\usepackage{mathtools}
\usepackage{thmtools}
\usepackage{bm}
\usepackage{minitoc}
\usepackage{svg}
\usepackage{graphicx}
\usepackage{algorithm}
\usepackage{algpseudocode}
\usepackage{enumitem}
\usepackage{todonotes}
\usepackage{lipsum}
\usepackage{array}
\usepackage{cleveref}
\usepackage{pgfgantt}
\usepackage[official]{eurosym}
\usepackage{subcaption}
\usepackage{tikz}
\usetikzlibrary{patterns}
\definecolor{mygreen}{HTML}{2CA02C}
\usepackage{wrapfig}
\usepackage{float}
\usepackage{bm}

\newcommand{\x}{\bm{x}}
\newcommand{\h}{\mathbf{h}}
\newcommand{\g}{\mathbf{g}}

\newcommand{\z}{\bm{z}}

\newcommand{\X}{\mathcal{X}}

\newcommand{\Z}{\mathcal{Z}}
\newcommand{\Y}{\mathcal{Y}}
\newcommand{\Q}{\mathcal{Q}}
\newcommand{\R}{\mathcal{R}}

\newcommand{\M}{\mathcal{M}}

\newcommand{\btheta}{\bm{\theta}}
\newcommand{\ellNPE}{\ell_{\mathrm{NPE}}}

\newcommand{\softplus}{\mathrm{softplus}}
\newcommand{\sech}{\mathrm{sech}}
\newcommand{\KL}{\mathrm{KL}}
\newcommand{\kl}{\mathrm{kl}}
\newcommand{\HPDR}{R_\alpha^q(\x)}

\newcommand{\NPEloss}{\mathcal{L}_{\mathrm{NPE}}}

\definecolor{npecolor}{HTML}{1f77e9}
\definecolor{drocolor}{HTML}{2ca02c}
\definecolor{stopnpecolor}{HTML}{FF5F1F}
\definecolor{calnpecolor}{HTML}{AA3377}
\definecolor{balancecolor}{HTML}{CCBB44}
\definecolor{dronlpdcolor}{HTML}{BBBBBB}
\definecolor{drokldistcolor}{HTML}{EE6677}
\definecolor{drocoverror}{HTML}{66CCEE}

\newcommand{\legendbox}[1]{%
  \textcolor{#1}{\rule{0.55em}{0.55em}}%
}

\definecolor{mydarkblue}{rgb}{0,0.08,0.45}

\newenvironment{talign*}
 {\let\displaystyle\textstyle\csname align*\endcsname}
 {\endalign}
\newenvironment{talign}
{\let\displaystyle\textstyle\align}
{\endalign}

\newtheorem{definition}{Definition}[section]

\newtheorem{theorem}{Theorem}[section]
\newtheorem*{theorem*}{Theorem}
\newtheorem{assumption}{Assumption}[section]
\newtheorem{proposition}[theorem]{Proposition}
\newtheorem*{proposition*}{Proposition}
\newtheorem{lemma}[theorem]{Lemma}
\newtheorem*{lemma*}{Lemma}

\crefname{lemma}{lemma}{lemmas}
\Crefname{lemma}{Lemma}{Lemmas}

\title{Conservative neural posterior estimation via distributionally robust training}

\author{%
  William Laplante\thanks{Equal contribution.}\\  
  University College London, The Alan Turing Institute\\
  \texttt{william.laplante.24@ucl.ac.uk}
  \And
  Yuga Hikida\footnotemark[1] \\
  Aalto University\\
  \texttt{yuga.hikida@aalto.fi} 
    \And
  Charita Dellaporta\\
    University College London\\
  \texttt{h.dellaporta@ucl.ac.uk}
  \And
  Fran\c{c}ois-Xavier Briol \\
    University College London \\
  \texttt{f.briol@ucl.ac.uk}
  \And
  Ayush Bharti \\
  Aalto University \\
  \texttt{ayush.bharti@aalto.fi} 
}

\begin{document}

\maketitle

\begin{abstract}

Simulation-based inference with neural posterior estimation (NPE) often yields overconfident and unreliable posteriors under limited simulation budgets.
To address this, we propose DRO-NPE, a distributionally robust approach that replaces the standard NPE objective with a worst-case loss over a Wasserstein ambiguity set.
We introduce KL-based metrics for miscoverage and miscalibration, and use these to show that the DRO-NPE objective controls overfitting and reduces posterior overconfidence.
Our method is tractable, parallelisable, and readily integrates with standard normalising flows. Across benchmark SBI tasks, DRO-NPE consistently improves coverage and calibration, while narrowing the gap between empirical and population NPE loss, leading to more reliable inference in low-simulation regimes.
\end{abstract}

\section{Introduction}
\label{sec:introduction}

Simulation-based inference \citep[SBI;][]{Cranmer2020} is a powerful framework for inferring parameters of scientific models whose likelihood functions are unavailable or computationally prohibitive to evaluate, but for which simulating data is straightforward. The use of flexible neural conditional density estimators has substantially expanded the applicability of SBI to challenging problems, especially in fields such as particle physics \citep{brehmer2021simulation}, cognitive neuroscience \citep{fengler2021likelihood}, economics \citep{dyer2024black} and cosmology \citep{alsing2018massive, Jeffrey2021}. 
Neural SBI methods rely on simulations from the scientific model to approximate intractable quantities such as the posterior, the likelihood, the likelihood-to-evidence ratio, or the score function; see \citet{zammitmangion2024neural} for a recent review. In this work, we focus on the widely used neural posterior estimation (NPE) method \citep{Papamakarios2016, Radev2022}. 

A central practical limitation of NPE is the simulation budget required to train the conditional density estimator. As many scientific simulators are expensive to run, generating a sufficiently large training set is often the main computational bottleneck. Prior work has sought to reduce simulation cost using sequential training schemes \citep{Papamakarios2016, Lueckmann2017, Greenberg2019}, multi-fidelity simulations \citep{Hikida2025,Krouglova2025, tatsuoka2025multi, Saoulis2025}, or by exploiting heterogeneous simulation costs \citep{Bharti2024}. 
However, these strategies do not directly address the challenge of ensuring reliable uncertainty estimates in low-simulation regimes.
In practice, the adequacy of a given simulation budget is rarely clear \emph{a priori}, and neural SBI methods trained on limited simulations can exhibit poor uncertainty quantification, leading \citet{Hermans2022} to warn of the potential for a ``trust crisis''. 

Under finite simulations, NPE can yield \emph{overconfident} posterior approximations whose credible regions contain the true parameter less often than their nominal coverage suggests. 
Such undercoverage is especially problematic in scientific applications, where overstated certainty can lead to misleading conclusions. By contrast, conservative posteriors are often safer as they widen uncertainty statements rather than overstating confidence. This has motivated several works on conservative uncertainty quantification in neural SBI: 
\citet{delaunoy2022towards} proposed a regularised version of neural ratio estimation \citep[NRE;][]{Hermans2020} to inflate estimator uncertainty, and \citet{delaunoy2023balancing} extended this idea to contrastive NRE \citep{miller2022contrastive} and NPE.  
\citet{falkiewicz2023calibrating} instead regularise NPE directly with a coverage error term. 
However, these methods remain largely heuristic: their regularisation terms encourage conservativeness without guarantees on uncertainty quantification error, potentially leading to overly conservative and uninformative posteriors. 
Instead of regularising  NPE, \citet{delaunoy2024low} use Bayesian neural networks for conservative inference, at the cost of Bayesian inference over network weights.
Complementarily, conformal and Neyman-inversion methods \citep{Masserano2023_Waldo, patel2023variational, cabezas2025cp4sbi} post-hoc calibrate credible sets with coverage guarantees, but do not directly address poor conditional density estimation. 

\begin{wrapfigure}{r}{0.3\textwidth}
    \centering
    \vspace{-1.5em}
    \includegraphics[width=\linewidth]{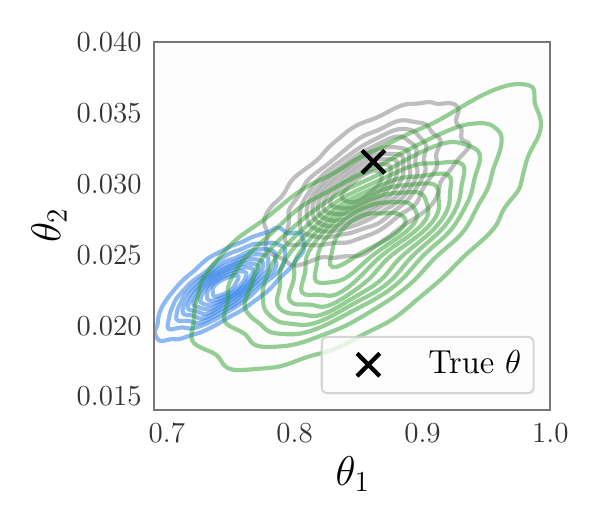}
    \vspace{-4ex}
    \caption{Lotka--Volterra posteriors in a low-simulation regime; details in \Cref{sec:experiments}.}
    \label{fig:intro-posterior-comparison}
    \vspace{-0.5em}
\end{wrapfigure}

In this paper, we view NPE through the lens of \emph{empirical risk minimisation}. Standard NPE minimises the average loss over a finite set of simulated parameter--data pairs (the empirical risk), which approximates the expected loss (the \emph{population risk}). When the simulation budget is limited, the empirical risk can be a poor proxy for the population risk, leading to overfitting and poorly calibrated posteriors \citep{Hermans2022}. This is illustrated in \Cref{fig:intro-posterior-comparison} on Lotka--Volterra model: standard NPE (\legendbox{npecolor}) is concentrated away from the true parameter, while a wider reference NPE posterior trained with many more simulations (\legendbox{gray}) assigns mass around it.
We make NPE conservative using \emph{distributionally robust optimisation} (DRO; \citealp{kuhn2025distributionally}), a framework for optimisation under distributional uncertainty. Our method (DRO-NPE; \legendbox{drocolor}) yields a broader posterior that covers the true parameter, reflecting uncertainty from the limited simulation budget.

Our contributions are as follows. First, we introduce Kullback--Leibler (KL) divergence-based uncertainty-quantification metrics (\Cref{sec:miscov-miscal}) that generalise standard notions of coverage and calibration errors \citep{Hermans2022, talts2018validating} and connect them directly to the NPE objective. 
Second, we propose a Wasserstein DRO objective that upper-bounds these metrics and the population NPE loss, so that, under suitable conditions, minimising it controls both posterior-fitting and uncertainty-quantification error (\Cref{theorem:W-DRO-upper-bound-chain}).
We also derive a computationally cheaper regularised objective that upper-bounds the DRO objective (\Cref{proposition:W-DRO-upper-bound}), thus retaining desirable theoretical properties, and verify its assumptions for standard normalising-flow architectures used in NPE (\Cref{lem:flows-assumptions}).
Third, we establish excess-risk, miscoverage, and miscalibration guarantees for DRO-NPE, and compare them to those of standard NPE (Appendix~\ref{appendix:excess-miscov-miscal}).
Finally, we show that DRO-NPE improves uncertainty quantification and generalisation across challenging benchmarks (\Cref{sec:benchmarking}) and a real-world cosmology application (\Cref{sec:experiments_cosmology}).

\section{Background}\label{sec:background}

We introduce NPE and uncertainty-quantification metrics revealing its finite-simulation failure modes. We use subscripts on distributions and CDFs to denote the random variables in question and any conditioning variables, and superscript on CDFs to indicate the underlying distribution.
\paragraph{Simulation-based inference with neural posterior estimation.}

Let $\X \subseteq \mathbb{R}^{d_{\X}}$ and $\Theta \subseteq \mathbb{R}^{d_{\Theta}}$ denote the data and parameter spaces, and let $\Z := \X \times \Theta \subseteq \mathbb{R}^{d_{\Z}}$, where $d_{\Z} = d_{\X} + d_{\Theta}$.
Consider a prior $\Pi$ on $\Theta$ with density $\pi$ and a conditional distribution $P_{\x \mid \btheta}$ from which one can simulate, but whose density $p(\x \mid \btheta)$, and therefore the likelihood, cannot be evaluated pointwise.
Together, $\Pi$ and $P_{\x \mid \btheta}$ induce a joint distribution $P_{\x,\btheta}$, denoted by $P$ for simplicity, with marginal $P_{\x}$ and posterior $P_{\btheta \mid \x}$.
The goal of SBI is to approximate this posterior distribution $P_{\btheta \mid \x_0}$ given observed data $\x_o$.

NPE is an amortised SBI method that approximates the posterior using flexible conditional density estimators such as mixture-density networks \citep{bishop1994mixture} and normalising flows \citep{papamakarios2021normalizing}. Let $q_\phi(\btheta \mid \x)$ denote such an estimator with parameter $\phi \in \Phi \subseteq \mathbb{R}^{d_{\Phi}}$, and let $Q_{\btheta \mid \x}$ be the corresponding distribution (we remove the explicit dependence on $\phi$ for notational simplicity). Assuming $P_{\btheta \mid \x} $ is absolutely continuous with respect to $Q_{\btheta \mid \x}$, the NPE objective is obtained by minimising the expected KL divergence between $P_{\btheta \mid \x} $ and $Q_{\btheta \mid \x}$ over $\phi$: 
\begin{talign}
     \mathbb{E}_{\x \sim P_{\x}}\left [\KL \Big(P_{\btheta \mid \x} \;\|\; Q_{\btheta \mid \x} \Big)\right] \overset{+ K(P)}{=}  \mathbb{E}_{(\btheta, \x) \sim P}[- \log q_\phi(\btheta \mid \x)]:=\NPEloss(q_\phi;P),
     \label{eq:NPE-KL-objective}
\end{talign}
where the additive constant $K(P)$ does not depend on $\phi$. Note that $\NPEloss(q_\phi; P)$ can be thought of as a \emph{population risk}, as it is defined as an expectation under $P$, induced by the \emph{pointwise loss} $\ellNPE(q_\phi; \z):= -\log q_{\phi}(\btheta \mid \x)$, where $\z = (\x, \btheta)$.
In practice, $\NPEloss(q_\phi; P)$ is approximated using independent and identically distributed (i.i.d.) samples $\z_{1:n} \sim P$, where each $\z_i$ is obtained by first sampling $\btheta_i \sim \Pi$, and then simulating $\x_i \sim P_{\x \mid \btheta_i}$.
These samples induce the empirical measure $P_n := \nicefrac{1}{n}\sum_{i=1}^n \delta_{\z_i}$, such that $\NPEloss(q_\phi; P_n) := \nicefrac{1}{n}\sum_{i=1}^n \ellNPE(q_\phi; \z_i)$ is the associated \emph{empirical risk}.
The conditional density model is then trained by minimising the empirical risk, $\hat{\phi}_n := \arg\min_{\phi \in \Phi} \NPEloss(q_\phi;P_n)$.
As a result, the estimator $q_{\hat{\phi}_n}$ is amortised: for new observed data $\x_o$, the posterior is approximated by a simple forward pass, namely $q_{\hat{\phi}_n}(\btheta \mid \x_o) \approx p(\btheta \mid \x_o)$, without requiring retraining.

\paragraph{Uncertainty quantification metrics in SBI.}
Despite its widespread success, NPE remains prone to failures in uncertainty quantification. We focus on the two most widely used diagnostics for assessing this: expected coverage of high-probability density region (HPDR)  \citep{Ward2022, Hermans2022} and simulation-based calibration \citep{talts2018validating, sailynoja2022graphical, modrak2023simulation}.
Other diagnostics have also been studied \citep{zhao2021diagnostics, linhart2022validation, linhart2023c2st, lemos2023sampling, chung2024sampling, dey2025towards}, but are less common in SBI.

\begin{definition}[Expected coverage of HPDR]
\label{def:hpdr-coverage}
The $(1-\alpha)$ HPDR induced by $q$ at $\x$ is the region where $q(\btheta \mid \x)$ is highest while containing $1-\alpha$ of its mass.
Formally, \( \HPDR := \left\{\btheta \in \Theta : q(\btheta \mid \x) \ge \tau_\alpha^q(\x)\right\}\), where $\tau_\alpha^q(\x)$ is a threshold such that \( Q_{\btheta \mid \x}\!\left(\HPDR\right) = 1-\alpha \).
The expected coverage of the HPDR is then: 
\begin{talign}
\label{eq:expected-cov-def}
C_\alpha^P(q)
:= \mathbb{E}_{\z\sim P}
\left[\mathbf{1}\Big\{\btheta \in \HPDR\Big\}\right].
\end{talign}
\end{definition}
Expected coverage of HPDR is estimated by repeatedly drawing $\z_i=(\x_i,\btheta_i)\sim P$. For each $\x_i$, $M$ posterior samples $\tilde{\btheta}_{i,1},\dots,\tilde{\btheta}_{i,M}\sim Q_{\btheta\mid\x_i}$ are used to evaluate the densities $q(\tilde{\btheta}_{i,m}\mid \x_i)$ for $m=1,\dots,M$, and the empirical $\alpha$-quantile of these values approximates the threshold $\tau_\alpha^q(\x_i)$ defining the $(1-\alpha)$ HPDR. Expected coverage of HPDR is then assessed by whether the ground-truth parameter $\btheta_i$ satisfies $q(\btheta_i\mid \x_i)\ge \tau_\alpha^q(\x_i)$.
We say $Q_{\btheta\mid\x}$ is conservative when $C_\alpha^P(q) >1-\alpha$ and overconfident when $C_\alpha^P(q)<1-\alpha$.

\begin{definition}[Calibration w.r.t. $S$]\label{def:cal}
Let $S:\Theta\times\X\to\mathbb{R}$ be a summary, and $Q_{S \mid \x}$, $P_{S \mid \x}$ be the distributions of $S(\btheta;\x)$ under $\btheta \sim Q_{\btheta \mid \x}$ and $\btheta \sim P_{\btheta \mid \x}$, respectively.
Let the CDF of $Q_{S \mid \x}$ be denoted as \(F^Q_{S \mid \x}\).
We say that $Q_{\btheta \mid \x}$ (or equivalently $q$) is calibrated with respect to $S$ if, for $S^P_{\x} \sim P_{S \mid \x}$, the random variable $F^Q_{S\mid\x}(S_{\x}^P)$ is distributed as $\mathcal U[0,1]$ under $\x\sim P_{\x}$.
\end{definition}

Calibration with respect to \(S\) means that, under mild conditions on the CDF of \(Q_{S\mid \x}\), the distribution of $S$ under $P_{\btheta \mid \x}$ is the same as that of $S$ under $Q_{\btheta \mid \x}$, for \(P_{\x}\)-almost every \(\x\).
In practice, simulation-based calibration (SBC) is used to approximate calibration w.r.t. $S$ by repeatedly drawing $(\btheta_i,\x_i)\sim P$ and computing $u_i := F^Q_{S\mid\x_i}\big(S(\btheta_i;\x_i)\big), i=1,\dots,n$.
The empirical distribution of $u_{1:n}$ is then compared to $\mathcal U[0,1]$, for example via histograms, empirical CDFs, or summary statistics.
The function $S$ allows SBC to be applied in higher-dimensional settings.
For example, \(S(\btheta; \x)=\|\btheta-\btheta_0\|_2\) for a reference point \(\btheta_0\) is used in the \textit{BayesFlow} library \citep{kuhmichel2026bayesflow}. Note that expected coverage of HPDR is a special case of SBC when \(S(\btheta;\x)=q(\btheta\mid\x)\) \citep{delaunoy2022towards}; see also Appendix~\ref{appendix:theory-calib-cov}, \Cref{lemma:cov-special-case-cal}. Intuitively, the connection arises because coverage is simply calibration applied to a single summary evaluating whether $\btheta$ falls in the HPDR.

\section{A new KL-based perspective on uncertainty quantification metrics}\label{sec:miscov-miscal}

We consider errors in the expected coverage of HPDR and in calibration w.r.t. $S$, which we refer to as \emph{miscoverage} of HPDR and \emph{miscalibration} w.r.t. $S$, respectively. 
Superscripts on the miscoverage and miscalibration metrics indicate the miscoverage level $\alpha$ and the summary statistic $S$, respectively.
Miscoverage is often measured by the absolute deviation \citep{falkiewicz2023calibrating, yao2024simulation}:
\begin{talign*}
    \Delta_{\mathrm{cov}}^\alpha(q; P) := \left | C_\alpha^P(q) - (1 - \alpha) \right |.
\end{talign*}
However, this definition is blind to the prior, as the latter attains perfect expected coverage of HPDR, i.e., for any $\alpha \in (0,1)$,
\(
C_\alpha^P(\pi) = \mathbb{E}_{\z \sim P} [ \mathbf{1}\{\btheta \in R_{\alpha}^\pi \}] = \Pi(R_\alpha^\pi) = 1 - \alpha;
\)
therefore, $\Delta_{\mathrm{cov}}^\alpha(\pi;P)=0$ for all valid $\alpha$. This motivates an alternative notion of miscoverage. We adopt a KL-based approach because it removes prior blindness, links miscoverage and miscalibration to the NPE objective (\Cref{theorem:W-DRO-upper-bound-chain}), and provides a natural target for hyperparameter selection.
\begin{definition}[KL-based miscoverage of HPDR]
\label{def:kl-miscoverage}
Fix $\alpha \in (0,1)$, and let $T_A(\btheta):= \mathbf{1}\{\btheta \in A\} \in \{0,1\}$ denote the indicator of a set $A \subseteq \Theta$. 
Let $(T_A)_{\#}P_{\btheta \mid \x}$ denote the pushforward of $P_{\btheta\mid\x}$ through $T_A$, i.e., the Bernoulli distribution of $T_A(\btheta)$ when $\btheta\sim P_{\btheta\mid\x}$.
Define $(T_A)_{\#}Q_{\btheta \mid \x}$ analogously.
Then, KL-based miscoverage of HPDR is:
\begin{talign*}
        \kl_{\mathrm{cov}}^\alpha(q; P) :=  \mathbb{E}_{\x \sim P_{\x}}\left [\KL \left ( (T_{\HPDR})_{\#} P_{\btheta \mid \x} \;\Big \| \; (T_{\HPDR})_{\#} Q_{\btheta \mid \x} \right) \right].
\end{talign*}
\end{definition}
KL-based miscoverage of HPDR is weaker than the expected KL divergence between the full conditionals $P_{\btheta\mid\x}$ and $Q_{\btheta\mid\x}$ in \Cref{eq:NPE-KL-objective}, since it compares only the Bernoulli distributions induced by whether $\btheta$ lies in the HPDR $\HPDR$. This follows from the data-processing inequality \citep{polyanskiy2025information}, which states that KL divergence cannot increase under pushforwards. Moreover, $\kl_{\mathrm{cov}}^\alpha$ upper-bounds $\Delta_{\mathrm{cov}}^\alpha$ via Pinsker's inequality: $\Delta_{\mathrm{cov}}^\alpha(q; P) \leq \sqrt{\nicefrac{1}{2} \kl_{\mathrm{cov}}^\alpha (q; P)}$; see \Cref{lemma:kl-based-miscov-upper-bounds-abs-val}.
Small KL-based miscoverage of HPDR then implies small $\Delta_{\mathrm{cov}}^\alpha$, but not conversely: $\kl_{\mathrm{cov}}^\alpha(q;P)$ may remain positive even when $\Delta_{\mathrm{cov}}^\alpha(q;P) = 0$, and, unlike $\Delta_{\mathrm{cov}}^\alpha$, $\kl_{\mathrm{cov}}^\alpha(q;P)$ does not generally vanish for the prior. The latter would require the posterior to place the same mass as the prior on the HPDR for almost every $\x$.

The standard notion of CDF-based miscalibration quantifies deviation of
the distribution of $F^Q_{S\mid\x}(S_{\x}^P)$ from standard uniform 
\citep{talts2018validating, gneiting2007probabilistic}.
For the pushforward distributions $(F^Q_{S \mid \x})_{\#} P_{S \mid \x}$ and $(F^Q_{S \mid \x})_{\#} Q_{S \mid \x}$, where the latter is uniformly distributed under mild assumptions, this deviation is often measured via Kolmogorov--Smirnov distances or via Cram\'er--von Mises criteria \citep{rossi2019alternative, zhao2021diagnostics, falkiewicz2023calibrating}.
We introduce a KL-based alternative that captures deviation from uniformity and can be related to the population NPE risk.

\begin{definition}[KL-based miscalibration w.r.t. $S$] For $S:\Theta \times \X \to \mathbb{R}$, KL-based miscalibration is:
\begin{align*}
    \kl^S_{\mathrm{cal}}(q;P) := \mathbb{E}_{\x \sim P_{\x}}\left [\KL \left((F^Q_{S \mid \x})_{\#} P_{S \mid \x}  \,\middle\|\, (F^Q_{S \mid \x})_{\#} Q_{S \mid \x} \right) \right].
\end{align*}
\end{definition}
This notion is motivated by KL-based miscalibration discrepancies previously used in probabilistic regression \citep{utpala2020quantile,dheur2023large}.
As with $\kl_{\mathrm{cov}}^\alpha$, $\kl^S_{\mathrm{cal}}$ is weaker than the full KL divergence between $P_{\btheta\mid\x}$ and $Q_{\btheta\mid\x}$ by the data-processing inequality. The two nevertheless capture different notions of mismatch: for fixed \(\alpha\), \(\kl^\alpha_{\mathrm{cov}}\) compares only the mass assigned to the HPDR, so $\kl^\alpha_{\mathrm{cov}}(q;P)=0$ does not imply $P_{\btheta\mid\x}=Q_{\btheta\mid\x}$. By contrast, $\kl_{\mathrm{cal}}^S=0$ implies equality of the induced distributions, and under mild conditions on $F^Q_{S\mid\x}$, equality of $P_{S\mid\x}$ and $Q_{S\mid\x}$. 
If $S$ is a sufficient statistic for the model, this further implies $P_{\btheta\mid\x}=Q_{\btheta\mid\x}$. For \(S(\btheta ;\x)=q(\btheta\mid\x)\), under which coverage of HPDR is a special case of SBC, \(\kl^S_{\mathrm{cal}}\) upper-bounds \(\kl^\alpha_{\mathrm{cov}}\); see \Cref{lem:miscal-upper-bound-miscov}.

\section{Distributionally robust optimisation for NPE}\label{sec:method}

NPE would ideally target the population risk \(\mathcal{L}(q_\phi;P)\), but \(P\) is accessible only through the empirical measure \(P_n\), which can be a poor approximation when \(n\) is not sufficiently large. 
In such settings, this leads to substantial performance degradation \citep{wang2025learning}, and gives rise to the classical \emph{generalisation gap} \citep{vapnik1991principles, shalev2014understanding}.
In particular, minimisation over \(P_n\) yields an overly optimistic estimate of the achievable population risk \citep{kuhn2019wasserstein}: 
\begin{talign*}
\mathbb{E}_{\z_{1:n} \overset{\text{i.i.d.}}{\sim} P}[\inf_{\phi}\mathcal{L}(q_\phi; P_n)] \leq  \inf_{\phi}  \mathbb{E}_{\z_{1:n} \overset{\text{i.i.d.}}{\sim} P}[\mathcal{L}(q_\phi; P_n)] =
    \inf_{\phi}\mathcal{L}(q_\phi; P).
\end{talign*}
In NPE, this gap can translate into poor conditional density estimation.
We therefore use distributionally robust optimisation (DRO) to minimise an upper bound on the population risk, thus accounting for distributional uncertainty \citep{kuhn2019wasserstein, rahimian2022frameworks, wang2025learning}.

\paragraph{Motivating DRO-NPE.}\label{sec:WDRO-motivation}

We formalise DRO-NPE as a distributionally robust objective based on the \emph{Wasserstein} distance, defined as 
$ W^p_p(P,Q) := \inf_{\pi \in \Pi(P,Q)} \int_{\mathcal{Z} \times \mathcal{Z}} \|\z - \z^\prime\|^p \pi(d\z, d\z^\prime) $,
where $\Pi(P,Q)$ denotes the set of all joint probability distributions of $\z$ and $\z^\prime$ with marginals $P$ and $Q$ respectively, $p$ is the order of the Wasserstein distance, and $\| \cdot \|$ is the Euclidean norm.
DRO-NPE replaces the empirical NPE risk by its worst-case value over an \emph{ambiguity set} $\mathcal{A}$ of distributions \citep{blanchet2024distributionally, kuhn2025distributionally}:
\begin{talign*}
    \inf_{\phi}\,\overbrace{\NPEloss(q_\phi; P_n)}^{\text{empirical risk}}
\;\xrightarrow{\text{risk aversion}}\;
\inf_{\phi}\,\overbrace{\sup_{\tilde P \in \mathcal{A}}\NPEloss(q_\phi; \tilde P)}^{\text{DRO-NPE risk}}.
\end{talign*}
Here, we take $\mathcal{A}$ to be the $p$-Wasserstein ball of radius $\varepsilon \geq 0$ centred at $P_n$:
\begin{equation*}
\mathcal{A}_p(P_n; \varepsilon) := \left\{\tilde P \in \mathcal{P}_p : W_p(\tilde P, P_n) \leq \varepsilon\right\},
\quad
\mathcal{P}_p := \left\{\tilde P : \mathbb{E}_{\z \sim \tilde P}\big[\|\z - \z'\|^p\big] < \infty \;\; \forall \z' \in \mathcal{Z}\right\},
\end{equation*}
where $\mathcal{P}_p$ is the set of probability distributions with finite \(p\)-th moment.
The radius $\varepsilon \geq 0$ controls the size of the ambiguity set \citep[see][]{rahimian2022frameworks, kuhn2025distributionally}: $\varepsilon=0$ recovers the empirical risk, while larger values yield a more conservative objective by guarding against perturbations of $P_n$ within the Wasserstein ball \citep{wang2025learning}.
Although many distances and divergences have been considered for defining ambiguity sets in DRO, Wasserstein balls are particularly natural here.
For example, the commonly used $f$-divergence-based ambiguity sets \citep{hu2013kullback, bayraksan2015data, husain2023distributionally} primarily re-weight the support of the empirical distribution. 
By contrast, Wasserstein balls are based on mass transport \citep{kuhn2019wasserstein}, and can contain distributions supported away from the observed simulations, including continuous distributions. 

We now show that the population NPE risk upper bounds both $\kl_{\mathrm{cov}}^\alpha$ and $\kl^S_{\mathrm{cal}}$ by the data-processing inequality for the KL divergence (see \Cref{lemma:DPI-for-KL-div}). Thus, if $P$ were available, the NPE objective would provide a principled criterion for controlling these forms of posterior uncertainty. 
In finite-simulation settings, however, $P$ is unavailable, so we instead construct a high-probability upper bound on the population risk.
To do so, we rely on Wasserstein concentration results \citep{fournier2015rate, kuhn2019wasserstein}, stating that for a suitable radius $\varepsilon$ decreasing with $n$, the true distribution $P$ belongs to the ambiguity set $\mathcal A_p(P_n; \varepsilon)$ with high probability. This directly implies that the DRO-NPE objective upper-bounds the population NPE risk, and therefore also $\kl_{\mathrm{cov}}^\alpha$ and $\kl^S_{\mathrm{cal}}$; see Appendix~\ref{appendix:theorem-W-DRO-upper-bound-chain} for the proof.

\begin{theorem}
\label{theorem:W-DRO-upper-bound-chain}
For any $S: \Theta \times \mathcal{X} \to \mathbb{R}$ and $q_\phi$:
\begin{talign*}
    \mathrm{kl}_{\mathrm{cal}}^S(q_\phi; P) \overset{+K(P)}{\leq} \mathcal{L}_{\mathrm{NPE}}(q_\phi; P) \quad \text{and} \quad \sup_{\alpha \in (0,1)} \mathrm{kl}_{\mathrm{cov}}^\alpha(q_\phi; P) \overset{+K(P)}{\leq} \mathcal{L}_{\mathrm{NPE}}(q_\phi; P).
\end{talign*}
Moreover, let $p \in [1,\infty)$ and suppose there exists $\beta > 2p$ such that $M_{\beta} := \mathbb{E}_{\bm{z} \sim P}[\|\bm{z}\|^{\beta}] < \infty$. Fix $\delta \in (0,1)$. Then there exist positive constants $C_1, C_2$, with $C_2 \geq 1$, depending only on $p, d_{\mathcal{Z}}, \beta, M_{\beta}$ such that, writing $\rho := (\tfrac{1}{C_1}\log\tfrac{2C_2}{\delta})^{1/2}$, for any $q_\phi$, whenever $\varepsilon \geq \varepsilon(n,\delta) := \max \{\varepsilon_a(n,\delta), \varepsilon_b(n,\delta)\}$, where
\begin{equation*}
\varepsilon_a(n,\delta) :=
\begin{cases}
\rho^{1/p}\, n^{-1/(2p)}, & \text{if } d_{\mathcal{Z}} < 2p, \\[4pt]
\left( 2\rho\, n^{-1/2} \log\!\left(2 + \frac{n}{\rho^2}\right) \right)^{1/p}, & \text{if } d_{\mathcal{Z}} = 2p, \\[4pt]
\rho^{2/d_{\mathcal{Z}}}\, n^{-1/d_{\mathcal{Z}}}, & \text{if } d_{\mathcal{Z}} > 2p,
\end{cases}
\quad
\varepsilon_b(n,\delta) := \left(\frac{2C_2}{\delta}\right)^{\frac{1}{2p}} n^{-\frac{1}{2p}},
\end{equation*}
with probability at least $1-\delta$:
\begin{talign*}
    \mathcal{L}_{\mathrm{NPE}}(q_\phi; P) \leq \sup_{\tilde P \in \mathcal{A}_p(P_n;\varepsilon)} \mathcal{L}_{\mathrm{NPE}}(q_\phi; \tilde P). 
\end{talign*} 
\end{theorem}

While \Cref{theorem:W-DRO-upper-bound-chain} bounds $\mathrm{kl}_{\mathrm{cal}}^S$ and $\mathrm{kl}_{\mathrm{cov}}^\alpha$ separately, they are further connected for a particular choice of summary: under mild conditions, if $S(\btheta; \x) = f(q_\phi(\btheta \mid \x))$ for some strictly increasing $f$, then (see \Cref{lem:miscal-upper-bound-miscov}):
\begin{talign*}
\sup_{\alpha \in (0,1)} \mathrm{kl}_{\mathrm{cov}}^\alpha(q_\phi; P) \leq \mathrm{kl}_{\mathrm{cal}}^S(q_\phi; P),  
\end{talign*}
so that controlling calibration also controls coverage. Furthermore, the bound on $\mathrm{kl}_{\mathrm{cal}}^S$ could be extended to vector-valued summaries $S: \Theta \times \mathcal{X} \to \mathbb{R}^\nu$; we leave this investigation for future work.

The moment condition on $P$ is mild: for $p=2$, the case most relevant in our setting (see discussion below), it reduces to  $P$ having at least a finite fourth moment, which is satisfied by many light-tailed distributions.
The dependence of $\varepsilon(n, \delta)$ on $d_{\Z}$
reflects the curse of dimensionality inherent to Wasserstein concentration, with rates that are essentially optimal \citep{kuhn2019wasserstein, weed2019sharp, fournier2015rate}. In practice, however, SBI is often performed on handcrafted or learned summary statistics rather than raw observations \citep{Radev2022, deistler2025simulation}, which can substantially reduce the effective dimension.
Alternative DRO-based upper-bounding strategies may yield faster rates in some settings \citep[see, e.g.,][]{an2021generalization, gao2023finite} but require stronger assumptions, particularly on the loss, whereas our approach accommodates the general losses used in SBI. Next, we discuss the implementation of DRO-NPE.

\paragraph{A tractable objective for DRO-NPE.}\label{sec:DRO-NPE-method}
DRO-NPE admits two closely related formulations: a strong dual form (\Cref{eq:wasserstein-strong-dual}) and a regularised form (\Cref{proposition:W-DRO-upper-bound}).
Both provide upper bounds on the NPE population risk as in \Cref{theorem:W-DRO-upper-bound-chain}; the strong dual formulation is exact but challenging to obtain numerically, while the regularised form is computationally cheaper but approximate. 

We begin with the strong dual view, which replaces the generally intractable primal Wasserstein-DRO objective with an optimisation problem over auxiliary variables. 
Let the class of functions satisfying \textit{$\rho-$growth} for $\rho \geq 1$ be
\begin{talign*}
    \mathcal G_{\rho}(\Z) := \left\{ f:\Z \to \mathbb R \;:\; \exists\, C>0 \; \mathrm{s.t.} \; |f(\z)| \leq C(1+\|\z\|^\rho) \; \forall \z \in \Z \right\}.
\end{talign*}
Under the assumptions that for fixed $q_\phi$, $\z \mapsto \ellNPE(q_\phi;\z)$ is upper-semicontinuous and belongs to $\mathcal G_p(\Z)$, the \emph{strong dual formulation} is exact and finite \citep[][Theorem 1]{gao2023distributionally}:
\begin{align}
\label{eq:wasserstein-strong-dual}
        \sup_{\tilde P \in \mathcal{A}_p(P_n; \varepsilon)
        } \NPEloss(q_\phi; \tilde P) = \inf_{\lambda \geq 0}\! \left \{ \lambda \varepsilon^p + \frac{1}{n}\sum_{i=1}^n \sup_{\z' \in \Z}\! \Big ( \ellNPE(q_\phi; \z') - \lambda \|\z' - \z_i\|^p \Big) \right \}.
    \end{align}
Although the dual formulation retains the upper bound guarantee we seek and is more tractable than the primal Wasserstein-DRO objective, the inner supremum remains computationally demanding.
\citep{sinha2018certifying, volpi2018generalizing} approximate it by gradient ascent on adversarially perturbed inputs, which requires an inner optimisation loop for every sample at each training step and relies on the smoothness of the loss for its guarantees.
To avoid this, we instead focus on the case $p=2$ and derive in \Cref{proposition:W-DRO-upper-bound} a tractable finite-sample upper bound on the dual objective based on a second-order Taylor argument and growth control of the loss.
This yields a regularised objective well-suited for gradient-based optimisation.
In practice, we omit the $\mathcal{O}(\varepsilon^2)$ remainder, which we expect to be negligible in most cases.
The proof is in Appendix~\ref{appendix:proof-upper-bound}.

\begin{proposition}
Fix $q_\phi \in \Q$. 
Suppose that $\z \mapsto \ellNPE(q_\phi;\z)$ admits a twice continuously differentiable extension to $\mathbb R^{d_\Z}$, and belongs to
$\mathcal G_2(\Z)$.
Then, for every $ \varepsilon \geq 0$,
\begin{talign*}
\sup_{\tilde P \in \mathcal{A}_2(P_n; \varepsilon)
}
\NPEloss(q_\phi; \tilde P)
\leq
\underbrace{
\NPEloss(q_\phi; P_n)
+
\varepsilon
\left(
\frac{1}{n}\sum_{i=1}^n \|\nabla_{\z}\ellNPE(q_\phi;\z_i)\|^2
\right)^{\frac{1}{2}}
}_{\mathrm{the\; tractable \; DRO-NPE \; objective}} 
+ \, C_\phi^n \varepsilon^2.
\end{talign*}
Further, if \( \|\nabla_{\z} \ellNPE(q_\phi;\cdot)\| \in \mathcal G_\beta(\Z)\), \(
\|\nabla_{\z}^2 \ellNPE(q_\phi;\cdot)\| \in \mathcal G_\alpha(\Z)
\),
and $P$ is light-tailed, i.e. there exist $\eta>0$, $a>0$, and $M<\infty$
such that \( \mathbb{E}_{\z\sim P}\big[\exp(\eta\|\z\|^a)\big]\le M \),
then, for $m:=\max\{2,2\beta,\alpha\}$ and any sequence
$\varepsilon_n=o((\log n)^{-m/2a})$, the second-order term \(C_\phi^n \varepsilon_n^2\) is negligible in probability. 
\label{proposition:W-DRO-upper-bound}
\end{proposition}

\Cref{proposition:W-DRO-upper-bound} assumes twice continuous differentiability and growth control to upper-bound the 2-Wasserstein DRO objective by the tractable DRO-NPE objective with remainder of order $\mathcal{O}(\varepsilon^2)$.
Under further light-tail and derivative-growth conditions, this remainder vanishes in probability for any \(\varepsilon_n\) decaying slightly faster than logarithmically, a mild requirement satisfied in particular by polynomial rates, such as $\varepsilon(n,\delta)$ in \Cref{theorem:W-DRO-upper-bound-chain} with $p=2$.
This holds for fixed dimension: $d_{\mathcal{Z}}$, $d_\Theta$ and the architecture affect the constant $C^n_\phi$, but not the required rate.
Importantly, as shown in \Cref{lem:flows-assumptions} below, these assumptions hold for Masked Autoregressive Flows (MAFs) \citep{NIPS2017_6828} and coupling flows, the default density estimators in the \textit{sbi} \citep{tejero2020sbi} and \textit{BayesFlow} \citep{kuhmichel2026bayesflow} libraries; see Appendix~\ref{appendix:verify-cond-norm-flows} for a proof.
By contrast, existing upper bounds for \(p=2\) Wasserstein DRO typically require global smoothness assumptions on the loss, such as Lipschitz gradients \citep[e.g.,][]{gao2023finite}, and are therefore not directly suited to NPE with normalising flows.

\begin{lemma}\label{lem:flows-assumptions}
    Suppose $q_\phi$ is a MAF or coupling flow (formally defined in Appendix~\ref{appendix:verify-cond-norm-flows}), then $\ellNPE$ satisfies the assumptions of \Cref{proposition:W-DRO-upper-bound}.
\end{lemma}
Following \Cref{proposition:W-DRO-upper-bound}, in practice, we minimise the tractable DRO-NPE objective:
\begin{align*}
\mathcal{L}_{\text{NPE}}(q_\phi;P_n)+ \varepsilon \Omega(\phi; P_n), \quad \text{where} \quad  \Omega(\phi; P_n) := \left(
\frac{1}{n}\sum_{i=1}^n \|\nabla_{\z}\ellNPE(q_\phi;\z_i)\|^2
\right)^{\frac{1}{2}}.
\end{align*}
As $\varepsilon \to 0$, this reduces to standard NPE.
As $\varepsilon$ increases, the regularisation term dominates, and the optimisation problem approaches
$\inf_\phi \Omega(\phi; P_n) \geq 0$.
If this infimum is attained at zero, then any minimiser $\phi_{\text{min}}$ satisfies
\(
\nabla_{\z} \log q_{\phi_{\text{min}}}(\btheta_i\mid \x_i)=0
\; \forall i=1,\dots,n.
\)
Provided $q_{\phi_{\text{min}}}(\btheta_i\mid \x_i)>0$, this implies \( \nabla_{\z} q_{\phi_{\text{min}}}(\btheta_i\mid \x_i)=0 \quad \forall i=1,\dots,n\).
Thus, for large \(\varepsilon\), the objective favours conditional densities that vary less sharply in $\z$ at the simulated pairs.
This is consistent with improving expected coverage of HPDR, since overconfident densities tend to allocate excess mass around training pairs.

We further note that as $n$ grows, for fixed $\phi$ and under the conditions of \Cref{proposition:W-DRO-upper-bound}, $\Omega(\phi; P_n)$ converges almost surely to its finite population counterpart $\Omega(\phi; P) := (\mathbb{E}_{\bm{z} \sim P}[\|\nabla_{\bm{z}} \ell_{\mathrm{NPE}}(q_\phi; \bm{z})\|^2])^{1/2}$, so that for $\varepsilon_n \to 0$ the tractable objective recovers $\mathcal{L}_{\mathrm{NPE}}(q_\phi; P)$. This argument is pointwise in $\phi$; consistency of the resulting estimator with a data-dependent $\varepsilon_n$ would require uniform convergence over $\phi$, which we leave for future work.

\paragraph{Selecting $\varepsilon$.}
We follow common DRO practice and select $\varepsilon$ via a train-validation split \citep{mohajerin2018data}, using KL-based miscalibration $\kl_{\mathrm{cal}}^S$ with $S(\btheta; \x) = q_\phi(\btheta \mid \x)$ as the metric.
Under mild conditions, this metric upper-bounds miscoverage (\Cref{lem:miscal-upper-bound-miscov}), so reducing it can improve coverage. It also penalises both over- and underconfidence, whereas selection via absolute miscoverage favours larger $\varepsilon$; see our experiments in \Cref{sec:analysis}.
Specifically, for each candidate $\varepsilon$, we train DRO-NPE on the training set, estimate $\kl_{\mathrm{cal}}^S$ on the validation set via density-ratio estimation, and use Bayesian optimisation to select the next candidate; see Appendix~\ref{appendix:hyperparam-selection} for details.

\Cref{theorem:W-DRO-upper-bound-chain} also provides a radius $\varepsilon(n,\delta)$ under which its upper bound holds with high probability, but this radius depends on constants that are not available in closed form.
Using the explicit concentration bounds of \citet{fournier2023convergence}, one can obtain an analogue of \Cref{theorem:W-DRO-upper-bound-chain} with a computable radius, which guarantees that the bound holds; however, this radius is far too conservative in practice.
A promising alternative is to estimate $W_p(P_n, P)$ directly by bootstrapping the simulations, which satisfies the theory up to bootstrap error and is much cheaper than validation, since it requires no retraining.
However, unlike the approach with $\kl_{\mathrm{cal}}^S$, it remains unclear how the resulting $\varepsilon$ affects the coverage and calibration of the trained posterior, and we leave its study to future work.
Both approaches are discussed further in Appendix~\ref{appendix:alt-approaches-epsilon}.

\paragraph{Further statistical guarantees for DRO-NPE.} We derive statistical guarantees on the excess risk, miscalibration and miscoverage of both the exact and tractable formulations of DRO-NPE and NPE relative to the oracle estimator $\phi^\star := \arg\min_{\phi} \NPEloss(q_\phi;P)$ (see Appendix~\ref{appendix:excess-miscov-miscal}). To the best of our knowledge, comparable guarantees are not available for existing methods. In particular, under a stronger Lipschitz gradient assumption that could be satisfied in the NPE setting for bounded \(\Theta\), we derive a DRO-NPE excess risk bound that depends on $\Omega(\phi^\star; P_n)$ along with other $n-$dependent constants. In contrast, the equivalent excess risk bound for NPE additionally depends on the regulariser $\Omega(\hat{\phi}_n; P_n)$, which relies on the empirical risk minimiser. As a result, DRO-NPE can yield sharper bounds compared to NPE, when $q_{\hat{\phi}_n}$ exhibits high functional variation relative to $q_{\phi^\star}$.

\section{Experiments}\label{sec:experiments}

We evaluate DRO-NPE on a range of SBI tasks in \Cref{sec:benchmarking}, comparing it with existing NPE methods that encourage conservative posteriors. In \Cref{sec:analysis}, we study the effect of $\varepsilon$, compare alternative hyperparameter-selection metrics, and quantify the training overhead of DRO-NPE. Finally, we apply DRO-NPE to a real-world cosmology problem in \Cref{sec:experiments_cosmology}. The
code to reproduce our experiments is available at \url{https://github.com/yugahikida/dro-npe}.

\subsection{Benchmarking DRO-NPE}
\label{sec:benchmarking}

\begin{figure}
\centering
\begin{subfigure}{0.69\linewidth}
    \centering
    \includegraphics[width=\linewidth]{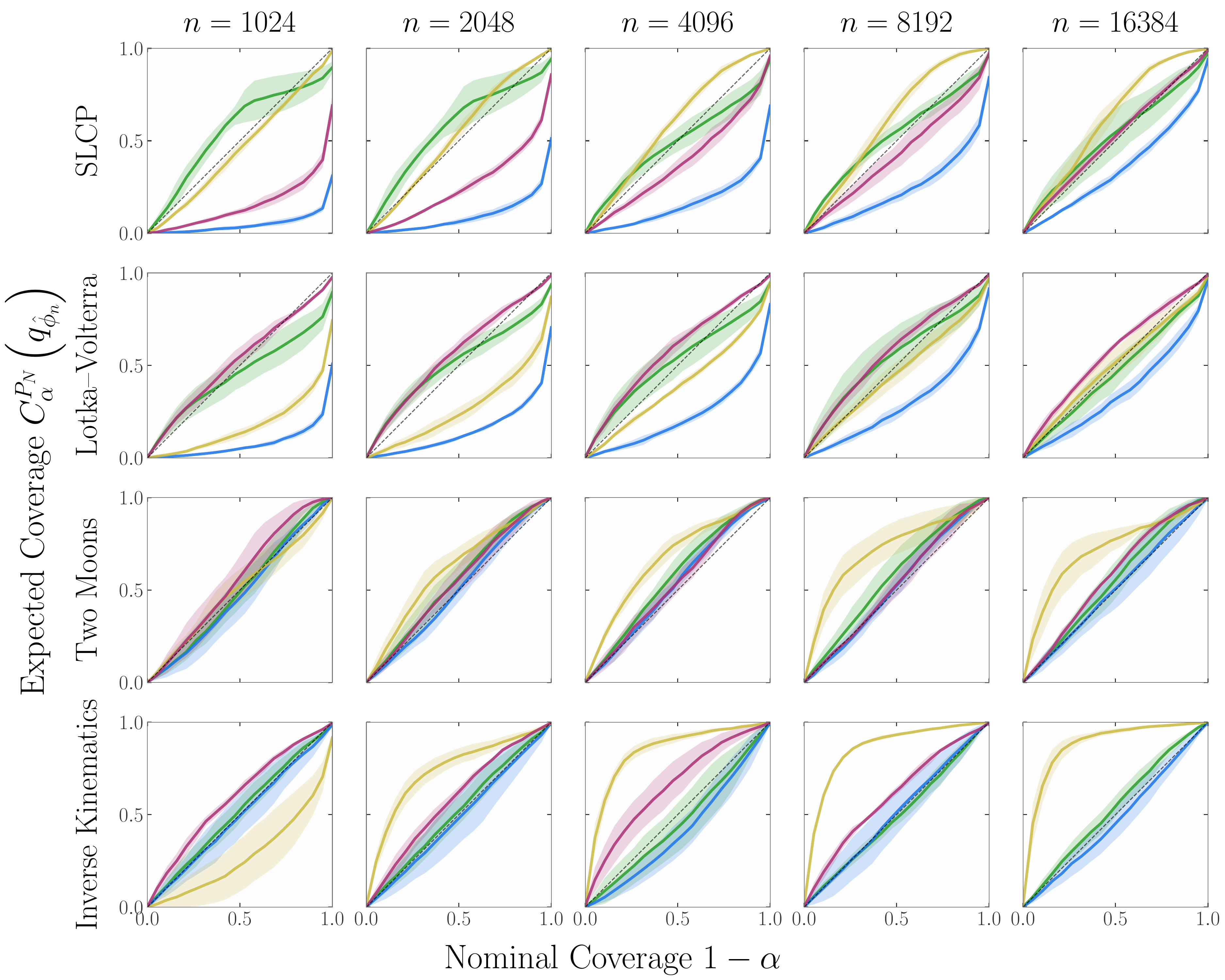}
    \caption{Coverage curves}
\end{subfigure}
\hfill
\begin{subfigure}{0.13\linewidth}
    \centering
    \raisebox{-2mm}{\includegraphics[width=\linewidth]{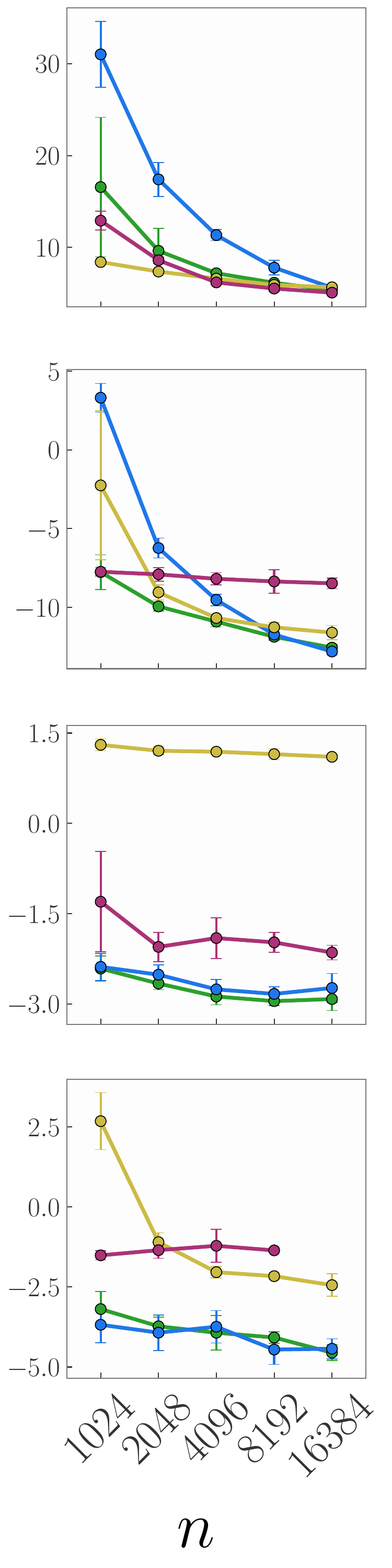}}
    \caption{NLPD}
\end{subfigure}
\hfill
\begin{subfigure}{0.13\linewidth}
    \centering
    \raisebox{-2mm}{\includegraphics[width=\linewidth]{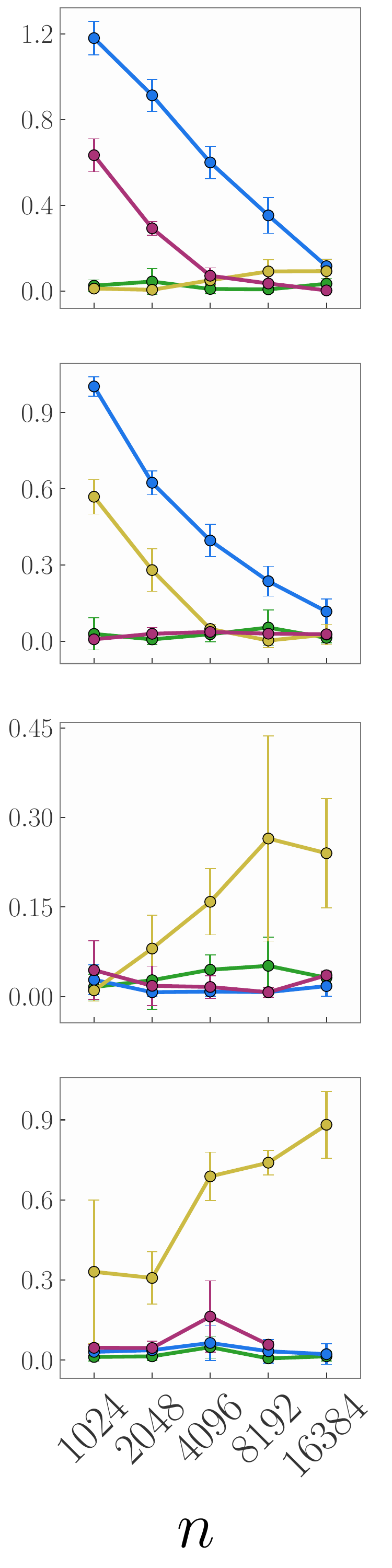}}
    \caption{$\kl_\mathrm{cal}^q$}
\end{subfigure}
\caption{Benchmarking DRO-NPE (\legendbox{drocolor}), CR-NPE (\legendbox{calnpecolor}), Bal-NPE (\legendbox{balancecolor}) and standard NPE (\legendbox{npecolor}) across four simulators and five simulation budgets. Means and standard deviations are shown over five random seeds. (a) Expected coverage curves of HPDR; the diagonal denotes perfect coverage, with curves above indicating conservativeness and curves below overconfidence. Coverage is computed at $18$ nominal levels using $N = 500$ test pairs
and $1000$ posterior samples per test point. 
(b) Negative log predictive density (NLPD) on held-out test data.
(c) Estimated KL-based miscalibration
$\kl_{\mathrm{cal}}^q$ on held-out test data; $\varepsilon$ is selected using a separate validation set.
Some CR-NPE results are missing due to numerical instabilities during training.}
\label{fig:benchmarking-main}
\end{figure}

We benchmark DRO-NPE against standard NPE, coverage-regularised NPE \citep[CR-NPE;][]{falkiewicz2023calibrating}, and balanced NPE \citep[Bal-NPE;][]{delaunoy2023balancing} on the following common SBI tasks from \citet{Lueckmann2021, kuhmichel2026bayesflow}: SLCP ($d_\X = 8$, $d_\Theta = 5$), Lotka--Volterra ($d_\X = 20$, $d_\Theta = 4$), Two Moons ($d_\X = 2$, $d_\Theta = 2$), and Inverse Kinematics ($d_\X = 2$, $d_\Theta = 4$). 
To isolate the effects of the training objectives, we use the same normalising-flow architecture and training protocol across methods, with the default hyperparameters of BayesFlow and Keras and without early stopping or additional regularisation. Consequently, our results may differ from those in the baseline papers. We also investigate combining each method with reduced model capacity and early stopping in ablation studies; see Appendix~\ref{app:additional-results}.

Evaluation is based on three metrics: (i) coverage plots, which show empirical versus nominal coverage, (ii) negative log predictive density (NLPD), which is the average negative log-density assigned by the learned posterior to the true parameter on held-out test points, and (iii) our $\kl_\mathrm{cal}^q$ metric. We compute the expected coverage using SBC with the summary $S(\btheta, \x) = \log q_\phi(\btheta \mid \x)$.
Following \citet{falkiewicz2023calibrating, delaunoy2023balancing}, we report results averaged over five random seeds; see \Cref{fig:benchmarking-main}.
Details of the baselines, tasks, and implementation are in Appendix~\ref{appendix:implementation-details}.

NPE yields overconfident posteriors for SLCP and Lotka--Volterra. Although this effect is more pronounced in low-simulation regimes, it persists for larger $n$. For Two Moons and Inverse Kinematics, the averaged NPE coverage curve is closer to the diagonal, though in some cases it shows substantial variability across runs.
Bal-NPE yields conservative posteriors for SLCP and Two Moons overall, but is overconfident for Lotka--Volterra and Inverse Kinematics in low-simulation regimes. A possible explanation is that the approximation $q_\phi(\btheta \mid \x) \approx p(\btheta \mid \x)$ used to derive the regularisation term in Bal-NPE is inaccurate early in training, which may hinder optimisation. Interestingly, Bal-NPE yields overly conservative posteriors for Inverse Kinematics and Two Moons when the simulation budget is large.
CR-NPE is conservative across almost all experiments, except for SLCP.  
We also observe numerical instabilities at larger sample sizes, leading to missing results in some cases (such as Inverse Kinematics with $n=16384$); see Appendix~\ref{app:arch-hyperparam} for details.

DRO-NPE yields conservative posteriors for SLCP and near-perfect coverage for Two Moons and Inverse Kinematics. In the latter two cases, standard NPE already attains near-perfect coverage, but DRO-NPE does so with lower variance across runs. The only notable failure is Lotka--Volterra at $n = 1024$, where the posterior remains slightly overconfident. This may be mitigated by selecting a larger $\varepsilon$.
In terms of NLPD, DRO-NPE achieves the lowest values overall. This is consistent with the fact that the baselines primarily encourage conservativeness, often at the expense of predictive performance. DRO-NPE also often improves on standard NPE, which we attribute to reduced generalisation error induced by distributional robustness.
Finally, DRO-NPE consistently yields low $\kl_\mathrm{cal}^q$ across simulation budgets and tasks, as expected because $\varepsilon$ is selected to minimise this criterion. 
\begin{figure}
\centering
\begin{subfigure}{0.24\linewidth}
    \centering
    \includegraphics[width=\linewidth]{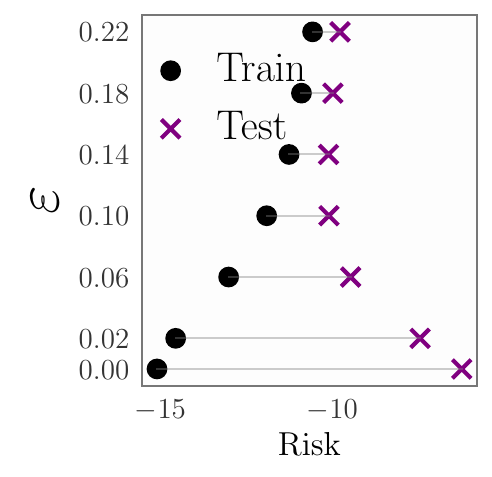}
    \caption{}
\end{subfigure}
\hfill
\begin{subfigure}{0.24\linewidth}
    \centering
    \includegraphics[width=\linewidth]{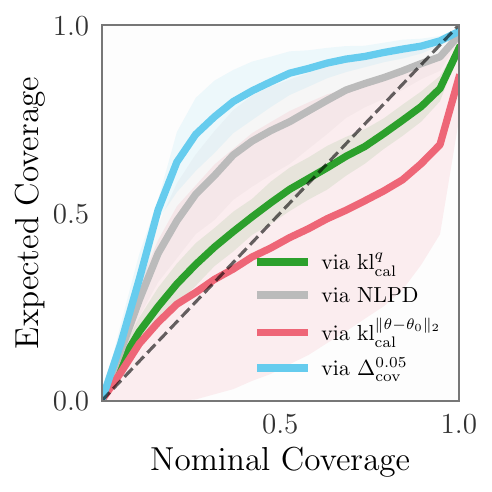}
    \caption{}
\end{subfigure}
\hfill
\begin{subfigure}{0.24\linewidth}
    \centering
    \includegraphics[width=\linewidth]{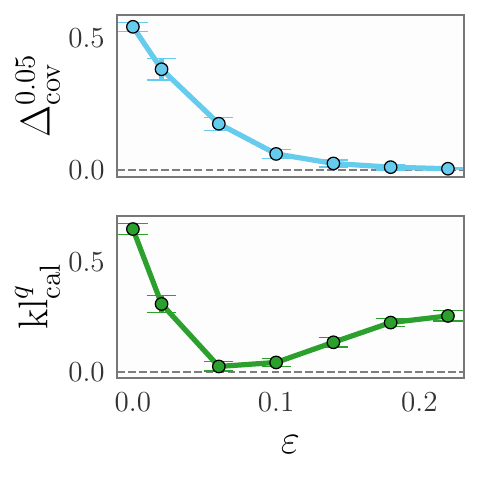}
    \caption{}
\end{subfigure}
\begin{subfigure}{0.24\linewidth}
    \centering
    \includegraphics[width=\linewidth]{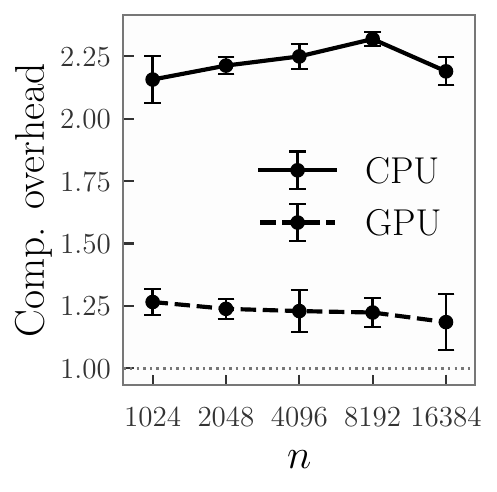}
    \caption{}
\end{subfigure}
\caption{ Analysis of DRO-NPE on Lotka--Volterra. Means and standard deviations are shown over five random seeds. (a) training and test risk across $\varepsilon$. 
(b) Coverage curves for $\varepsilon$ selected by $\kl_\mathrm{cal}^q$ (\legendbox{drocolor}), NLPD (\legendbox{dronlpdcolor}), $\kl_\mathrm{cal}^{\| \btheta- \btheta_0\|_2}$ with $\btheta_0$ drawn from the prior (\legendbox{drokldistcolor}), and absolute miscoverage at $\alpha = 0.05$ (\legendbox{drocoverror}). (c) Absolute miscoverage at $\alpha = 0.05$ (\legendbox{drocoverror}) and $\kl_\mathrm{cal}^q$ (\legendbox{drocolor}) across $\varepsilon$.  (d) Relative wall-clock runtime of DRO-NPE versus standard NPE on GPU and CPU.}
\label{fig:empirical-anal}
\end{figure}

\subsection{Analysing DRO-NPE}
\label{sec:analysis}

We now study the behaviour of DRO-NPE with respect to the key design choices involved in its implementation. All experiments are conducted using the Lotka–Volterra simulator with $n = 2048$.

\paragraph{Generalisation gap versus $\varepsilon$.} 
We first examine how $\varepsilon$ affects the generalisation gap by comparing empirical risk on training and test data. As shown in \Cref{fig:empirical-anal}(a), small values of $\varepsilon$ lead to a large generalisation gap and poor test performance, reflecting the optimistic bias of empirical risk minimisation in this setting.
As $\varepsilon$ increases, the gap narrows and test performance improves, attaining its minimum at $\varepsilon = 0.14$. Further increasing $\varepsilon$ continues to reduce the gap but worsens test performance, as the posterior estimator becomes excessively conservative.

\paragraph{Choice of hyperparameter-selection metric.}
In \Cref{sec:benchmarking}, we select $\varepsilon$ using the KL-based miscalibration metric $\kl_\mathrm{cal}^q$, corresponding to $S(\btheta ; \x)=q(\btheta\mid\x)$. 
We compare this choice with three alternatives: NLPD, absolute miscoverage \(\Delta_\mathrm{cov}^{0.05}\), and \(\kl_\mathrm{cal}^{\|\btheta-\btheta_0\|_2}\), computed with \(S(\btheta ; \x)=\|\btheta-\btheta_0\|_2\) for a fixed reference point \(\btheta_0\).
NLPD is a natural validation metric because it equals the NPE loss and upper-bounds \(\kl_\mathrm{cal}^q\), while absolute miscoverage directly measures deviation from nominal coverage. 
As shown in \Cref{fig:empirical-anal}(b), selecting \(\varepsilon\) via \(\kl_\mathrm{cal}^q\) yields a coverage curve closest to the diagonal.
Selection via NLPD and \(\Delta_\mathrm{cov}^{0.05}\) produces more conservative posteriors, whereas \(\kl_\mathrm{cal}^{\|\btheta-\btheta_0\|_2}\) leads to undercoverage and higher variability across runs. Results for SLCP are in Appendix~\ref{app:additional-results}.

To understand the behaviour of absolute miscoverage, \Cref{fig:empirical-anal}(c) plots \(\Delta_\mathrm{cov}^{0.05}\) as a function of \(\varepsilon\), together with \(\kl_\mathrm{cal}^q\). 
We observe that $\Delta_\mathrm{cov}^{0.05}$ decreases nearly monotonically with $\varepsilon$, as larger values produce wider, more conservative posteriors and hence better coverage. This explains why selecting \(\varepsilon\) via absolute miscoverage can favour overly conservative posterior approximations.
By contrast, \(\kl_\mathrm{cal}^q\) is minimised at an intermediate value of \(\varepsilon\), making it more informative for selecting \(\varepsilon\) in DRO-NPE.

\paragraph{Computational cost.}
DRO-NPE is more expensive than standard NPE, but only by a moderate constant factor. The main overhead comes from higher-order automatic differentiation through \(\nabla_{\z}\log q_\phi(\btheta\mid\x)\), so the computational cost of this regularisation term grows with the joint parameter--data dimension \(d_{\Z}\). 
As shown in \Cref{fig:empirical-anal}(d), DRO-NPE takes roughly \(2\times\) the wall-clock time of standard NPE on CPU, but only about \(1.25\times\) on GPU, as the additional operations are highly parallelisable. This relative overhead varies little with \(n\) over the range considered, although it may depend on the hardware and computational setup.
These timings assume fixed $\varepsilon$ and exclude hyperparameter selection.
For comparison, on Lotka--Volterra with $n=1024$, DRO-NPE training with fixed $\varepsilon$ takes $93$ seconds on GPU, compared to $1521$ seconds for the full procedure with 10 Bayesian optimisation iterations and final retraining.
Further details are provided in Appendix~\ref{app:cost_comparison}.

 \subsection{Application to cosmology}\label{sec:experiments_cosmology}

\begin{wrapfigure}{r}{0.5\textwidth}
    \centering
    \vspace{-1.7em}

    \begin{minipage}{0.54\linewidth}
        \centering
        \includegraphics[width=\linewidth]{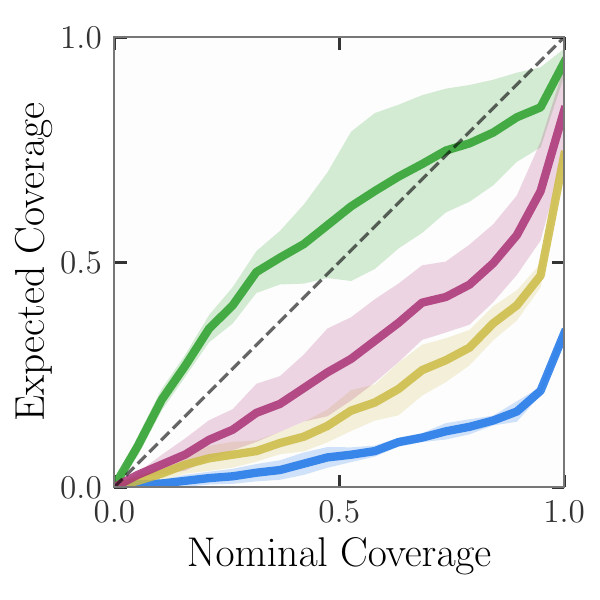}
    \end{minipage}
    \hfill
    \begin{minipage}{0.44\linewidth}
    \vspace{-0.8em}
        \centering
        \includegraphics[width=\linewidth]{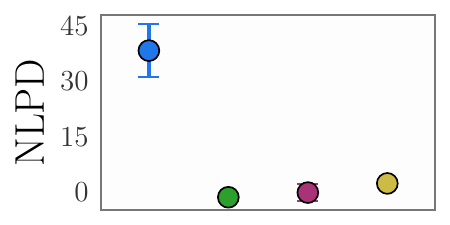}

        \vspace{0.5em}

        \includegraphics[width=\linewidth]{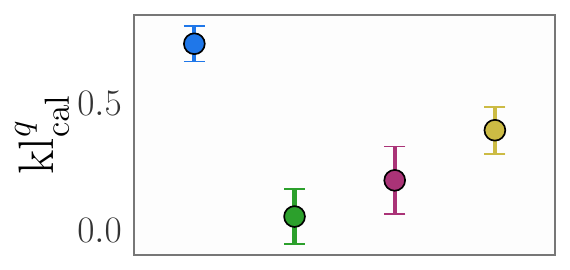}
    \end{minipage}
    \caption{Cosmology results: expected coverage, NLPD, and \(\kl_{\mathrm{cal}}^q\) for DRO-NPE (\legendbox{drocolor}), CR-NPE (\legendbox{calnpecolor}), Bal-NPE (\legendbox{balancecolor}), and standard NPE (\legendbox{npecolor}).}
    \label{fig:cosmo}
\end{wrapfigure}

Finally, we evaluate DRO-NPE on the CAMELS suite \citep{CAMELS_presentation, CAMELS_DR1}, a state-of-the-art cosmology dataset for machine learning on simulated universes. While the original data space is $256 \times 256$ dimensional, we use their standard $d_\X = 39$ dimensional summary as data, similar to \citet{Hikida2025}. One of the goals of CAMELS is to enable accurate constraints on $d_\Theta = 2$ cosmological parameters: the matter density $\Omega_m$ and the amplitude of matter fluctuations $\sigma_8$. 
We use \(1{,}000\) high-fidelity baryonic hydrodynamic simulations available to us, reserving \(900\) for training and \(100\) for testing. 
Conservative inference is particularly important in this setting as simulations are scarce, and the observables depend not only on \(\Omega_m\) and \(\sigma_8\), but also on nuisance astrophysical processes that can produce similar simulated data \citep{CAMELS_presentation}. Overconfident posteriors can therefore lead to overstated cosmological constraints and misleading scientific conclusions.

\Cref{fig:cosmo} reports expected coverage, NLPD, and $\kl_{\mathrm{cal}}^q$ across five independent training runs for each method on the same fixed train--test split. Standard NPE (\legendbox{npecolor}) is severely overconfident and has much worse NLPD than all conservative methods, while Bal-NPE (\legendbox{balancecolor}) improves NLPD but remains strongly undercovered. CR-NPE (\legendbox{calnpecolor}) improves upon Bal-NPE across all metrics, yet its coverage remains below the diagonal for all nominal levels. DRO-NPE (\legendbox{drocolor}) performs best overall, with the closest coverage to the diagonal, the lowest NLPD, and the lowest $\kl_{\mathrm{cal}}^q$. These results suggest that DRO-NPE can provide conservative yet efficient posterior inference in challenging low-simulation settings.

\section{Conclusion}\label{sec:conclusion}

We proposed DRO-NPE, a conservative formulation of NPE, to address finite-simulation failure modes including overconfidence, poor calibration, and large generalisation gaps. Our approach yields an objective that upper bounds the population NPE loss and KL-based metrics for uncertainty-quantification error, admits a tractable regularised objective, and enjoys statistical guarantees on excess risk, miscoverage, and miscalibration. 
Importantly, it remains compatible with standard normalising-flow architectures commonly used in SBI. 
However, DRO-NPE inherits limitations from DRO: it is more expensive than standard NPE, and there remains a theory--practice gap in selecting $\varepsilon$, since computing a theoretical radius guaranteeing the derived upper bounds is typically infeasible.

Although we focused on NPE, DRO-NPE could be extended to other conditional density estimation problems, score estimation for diffusion models, variational inference, amortised inference tasks, and other losses via analogous DRO formulations.
It could also be combined with post-hoc calibration methods \citep{Masserano2023_Waldo, patel2023variational, cabezas2025cp4sbi} and ensembles \citep{Hermans2022}.
More broadly, DRO offers one route to upper-bounding the population NPE objective; future work could investigate whether alternative frameworks, such as PAC-Bayes \citep{alquier2024user, haddouche2021pac}, yield analogous guarantees and practical training objectives. Future work could also investigate the robustness of DRO-NPE to certain out-of-distribution shifts.

\subsection*{Acknowledgments}
The authors thank Niall Jeffrey for help with using the CAMELS dataset.
WL was supported by UCL’s Centre for Doctoral Training in Data-Intensive Science and by the project ‘Bayesian robustness in filtering algorithms' funded by The Alan Turing Institute.
WL also acknowledges funding from the Aalto University Science Institute (AScI) Visiting Doctoral Researcher programme.
YH and AB were supported by the Research Council of Finland grant no. 362534.
CD and FXB were supported by the EPSRC grant [EP/Y022300/1]. CD was
additionally supported by the project ‘Reliable insights from scientific simulations’ funded from the EPSRC grant [UKRI3030].
The authors acknowledge the computational resources provided by the Aalto Science-IT project.

\bibliography{bibliography}
\bibliographystyle{apalike}


\newpage
\appendix

\begin{center}
    \LARGE \textbf{Appendix}
\end{center}
\vspace{1em}
We start by presenting the theoretical results on calibration, expected coverage of HPDR, miscalibration, and miscoverage in Appendix~\ref{appendix:theory-calib-cov}.
Then, in Appendix~\ref{app:theory-wasserstein-DRO}, we present theoretical results relating to the DRO-NPE method.
Finally, we present implementation details and additional experiments in Appendix~\ref{appendix:implementation-details} and Appendix~\ref{app:additional-results}, respectively.

\section{Theory: calibration and expected coverage of HPDR }
\label{appendix:theory-calib-cov}

In this section, we first show that the expected coverage of HPDR is a special case of calibration. Second, we present the data-processing inequality (DPI) and the chain rule for the Kullback-Leibler (KL) divergence, which lets us compare \(P\) and \(Q\) after applying pushforwards. Third, we prove that KL-based miscoverage upper-bounds its absolute value notion. Finally, we establish that KL-based miscalibration upper-bounds miscoverage.

We first prove that the expected coverage of HPDR is a special case of calibration, a result that was briefly shown in \citet{delaunoy2022towards}. 
This will require a statement more mathematical than \Cref{def:cal}, presented below. 
We say that $Q_{\btheta \mid \x}$ (or equivalently $q$) is calibrated w.r.t. to $S$ if:
\begin{talign}
     \mathbb{E}_{\x \sim P_{\x}} \left [P_{S \mid \x}\left(F^Q_{S \mid \x}(S_{\x}^P) \leq u \right) \right] = u \qquad \forall u \in [0,1]. \label{eq:calib-cdf}
 \end{talign}
 This is equivalent to stating that $F^Q_{S\mid\x}(S_{\x}^P)$ is distributed as $\mathcal U[0,1]$ under $\x\sim P_{\x}$. 

\begin{lemma}[Expected coverage of HPDR as a special case of calibration]
\label{lemma:cov-special-case-cal}
    Fix $q$. Let $S(\btheta; \x):= q(\btheta \mid \x)$, and assume that the CDF of $Q_{S \mid \x}$ is continuous and that $F^Q_{S\mid \x}$ is strictly increasing. Then, for every $\alpha \in (0,1)$,
    \[
    C_\alpha^P(q)
    =
    1 - \mathbb{E}_{\x \sim P_{\x}}\left[ P_{S \mid \x}\left ( F^Q_{S\mid \x}(S_{\x}^P) < \alpha \right) \right].
    \]
\end{lemma}
\begin{proof}
 Suppose  $S(\btheta;\x) = q(\btheta \mid \x)$. Then, notice that
 \begin{align*}
 F^Q_{S\mid \x}(\tau^q_{\alpha}(\x)) &= Q_{S \mid \x}( (-\infty, \tau_\alpha^q(\x)]) \\
 &= Q_{\btheta \mid \x}(\{S(\btheta ; \x) \leq \tau_{\alpha}^q(\x)  \})
 \\
 &= Q_{\btheta \mid \x}(\{q(\btheta \mid \x) \leq \tau_{\alpha}^q(\x)  \})
 = \alpha,
 \end{align*}
 where the first equality is by definition of the cumulative distribution function, the second by that of the pushforward, the third by that of $S$, and the last by that of $\tau_{\alpha}^q(\x)$ together with continuity of the CDF of $Q_{S \mid \x}$.
 Therefore, denoting $S_{\x}^P = S(\btheta;\x)$ whenever $\btheta \sim P_{\btheta \mid \x}$, we have that for any $\x \in \X$,
\begin{align*}
P_{\btheta\mid\x} \left(\HPDR \right)
&= P_{\btheta\mid\x} \left(q(\btheta\mid\x)\ge \tau^q_\alpha(\x)\right)\\[0.4em]
&= P_{S\mid\x}\left(S_{\x}^P \ge \tau^q_\alpha(\x)\right)\\[0.4em]
&= P_{S\mid\x}\left(F^Q_{S\mid \x}(S_{\x}^P)\ge F^Q_{S\mid\x}(\tau^q_\alpha(\x))\right)\\[0.4em]
&= P_{S\mid\x}\left(F^Q_{S\mid \x}(S_{\x}^P)\ge \alpha\right).
\end{align*}
The first equality is the definition of the HPDR at level $\alpha$. 
The second follows from the choice $S(\btheta;\x)=q(\btheta\mid \x)$.
The third uses that $F_{S\mid\x}^Q$ is continuous and strictly increasing.
The last equality uses that $F_{S\mid\x}^Q(\tau_\alpha^q(\x))=\alpha$, established above.
Finally, we get that:
\begin{align*}
    C_\alpha^P(q) &= \mathbb{E}_{\x \sim P_{\x}}\left [ P_{\btheta \mid \x}(\HPDR) \right] \\[0.4em]
    &= \mathbb{E}_{\x \sim P_{\x}}\left[ P_{S \mid \x}\left ( F^Q_{S\mid \x}(S_{\x}^P) \geq \alpha \right) \right] \\[0.4em]
    &= 1 - \mathbb{E}_{\x \sim P_{\x}}\left[ P_{S \mid \x}\left ( F^Q_{S\mid \x}(S_{\x}^P) < \alpha \right) \right].
\end{align*}
Since this identity holds for every \(\alpha\in(0,1)\), it shows that expected HPDR coverage is exactly the special case of \eqref{eq:calib-cdf} obtained by taking \(S(\btheta;\x)=q(\btheta\mid\x)\) and \(u=\alpha\).
\end{proof}

We now define the expected Kullback-Leibler (KL) divergence, which will be used frequently in this paper. Suppose $P_{\btheta \mid \x}$ is absolutely continuous w.r.t. $Q_{\btheta \mid \x}$, and for simplicity, that $P_{\btheta \mid \x}$ and $Q_{\btheta \mid \x}$ admit densities $p(\btheta \mid \x), q(\btheta \mid \x)$. Then, the KL divergence is:
\begin{align*}
    \mathbb{E}_{\x \sim P_{\x}}\left [\KL \Big(P_{\btheta \mid \x} \;\|\; Q_{\btheta \mid \x} \Big)\right] := \mathbb{E}_{\x \sim P_{\x}}\left[ \mathbb{E}_{\btheta \sim P_{\btheta \mid \x}}\left[\log\frac{p(\btheta \mid \x)}{q(\btheta \mid \x)}\right] \right].
\end{align*}
We can then state the data-processing inequality and the chain rule for the KL divergence.
\begin{lemma}[Theorem 2.15 \& Corollary 2.18 from \citet{polyanskiy2025information}. Chain rule and data-processing inequality for KL divergence]
\label{lemma:DPI-for-KL-div}
Suppose that for some measurable space $\Y$, $f: \Theta \to \Y $ is a measurable function. Then, for any fixed $\x \in \X$ and any two distributions $P_{\btheta \mid \x}, Q_{\btheta \mid \x}$, such that $P_{\btheta \mid \x} \ll Q_{\btheta \mid \x}$, i.e. $P_{\btheta \mid \x}$ is absolutely continuous w.r.t. $Q_{\btheta \mid \x}$,
\begin{align}
    \KL (P_{\btheta \mid \x} \| Q_{\btheta \mid \x}) = \KL(f_{\#} P_{\btheta \mid \x}  \| f_{\#}Q_{\btheta \mid \x}) + \mathbb{E}_{t \sim f_{\#} P_{\btheta \mid \x}}\left[\KL(P_{\btheta \mid  f(\btheta)=t, \x} \| Q_{\btheta \mid  f(\btheta)=t, \x})\right].
\end{align}
In particular,
\[
\KL (f_{\#} P_{\btheta \mid \x} \| f_{\#} Q_{\btheta \mid \x}) \leq \KL (P_{\btheta \mid \x} \| Q_{\btheta \mid \x}),
\]
and equality holds if $f$ is one-to-one, or when $f$ is a sufficient statistic for testing $P_{\btheta \mid \x}$ against $Q_{\btheta \mid \x}$.
\label{lemma:chain-rule-and-DPI-for-KL}
\end{lemma}
\vspace{1em}

We now show that KL-based miscoverage controls the usual absolute coverage error. 
In particular, by Pinsker's inequality, small KL-based miscoverage implies small absolute miscoverage. 
When this upper bound is reasonably tight, reducing the former also reduces the latter.
The proof uses the following elementary identity for Bernoulli distributions. 
Let $P_1$ and $P_2$ be Bernoulli distributions with probability parameters $p_1$ and $p_2$, respectively. Then
\[
\mathrm{TV}(P_1, P_2)
:= \frac{1}{2}\sum_{s \in \{0,1\}} |P_1(s) - P_2(s)|
= |p_1 - p_2|.
\]

\begin{lemma}[KL-based miscoverage upper-bounds its absolute value notion] 
\label{lemma:kl-based-miscov-upper-bounds-abs-val} 
Define $\kl_{\mathrm{cov}}^\alpha$ as in \Cref{def:kl-miscoverage}, and $\Delta_{\mathrm{cov}}^\alpha(q; P) := \left | C_\alpha^P(q) - (1 - \alpha) \right |$. Then, for any $\alpha \in (0,1)$ and $q(\btheta \mid \x)$,
\[
\Delta_{\mathrm{cov}}^\alpha(q; P) \leq \sqrt{\frac{1}{2}\kl_{\mathrm{cov}}^\alpha(q; P)}.
\]
\end{lemma}
\begin{proof}
    By definition of $C_\alpha^P$, we have that
    \begin{align*}
        \Delta_{\mathrm{cov}}^\alpha(q; P) =
         | \mathbb{E}_{\x \sim P_{\x}}[P_{\btheta \mid \x}(\HPDR) - (1-\alpha)]| 
        \leq \mathbb{E}_{\x \sim P_{\x}}[|P_{\btheta \mid \x}(\HPDR) - (1-\alpha)|].
    \end{align*}
    Then, for each fixed $\x$,
    \[
    |P_{\btheta \mid \x}(\HPDR) - (1-\alpha)| = \mathrm{TV}\left((T_{\HPDR})_{\#}P_{\btheta \mid \x}, \, (T_{\HPDR})_{\#}Q_{\btheta \mid \x}\right),
    \]
    since $(T_{\HPDR})_{\#}P_{\btheta \mid \x}$ and $(T_{\HPDR})_{\#}Q_{\btheta \mid \x}$ are Bernoulli laws with parameters $P_{\btheta \mid \x}(\HPDR)$ and $Q_{\btheta \mid \x}(\HPDR)=1-\alpha$. Therefore,  applying Pinsker's inequality pointwise in $\x$ yields:
    \begin{align*}
    \Delta_{\mathrm{cov}}^\alpha(q; P) &\leq \mathbb{E}_{\x \sim P_{\x}}\left[\mathrm{TV}\left((T_{\HPDR})_{\#}P_{\btheta \mid \x}, \, (T_{\HPDR})_{\#}Q_{\btheta \mid \x}\right)\right]  \\
    &\leq \mathbb{E}_{\x \sim P_{\x}}\left[\sqrt{\frac{1}{2}\KL \left((T_{\HPDR})_{\#}P_{\btheta \mid \x}\, \| \, (T_{\HPDR})_{\#}Q_{\btheta \mid \x} \right)} \right].
    \end{align*}
    Finally, Jensen's inequality gives:
    \begin{align*}
     &\mathbb{E}_{\x \sim P_{\x}}\left[\sqrt{\frac{1}{2}\KL \left((T_{\HPDR})_{\#}P_{\btheta \mid \x}\, \| \, (T_{\HPDR})_{\#}Q_{\btheta \mid \x} \right)} \right] \\
     &\leq \sqrt{\frac{1}{2} \mathbb{E}_{\x \sim P_{\x}}\left[\KL \left((T_{\HPDR})_{\#}P_{\btheta \mid \x}\, \| \, (T_{\HPDR})_{\#}Q_{\btheta \mid \x} \right)\right]}  = \sqrt{\frac{1}{2}\kl_{\mathrm{cov}}^\alpha(q; P)}.
    \end{align*}
\end{proof}
We end this section by showing that miscalibration is an upper-bound to miscoverage. This result will be used to prove \Cref{theorem:W-DRO-upper-bound-chain} in the next section.

\begin{lemma}[Miscalibration upper-bounds miscoverage] 
\label{lem:miscal-upper-bound-miscov}
Define 
\[
\kl_{\mathrm{NPE}}(q; P) := \mathbb{E}_{\x \sim P_{\x}}\left[\KL(P_{\btheta \mid \x}  \; \| \; Q_{\btheta \mid \x} )\right] = \NPEloss(q; P) + K(P).
\]
Then, for any distribution $Q$ with density $q$, 
\[
\sup_{\alpha \in (0,1)} \kl_{\mathrm{cov}}^\alpha(q; P) \leq  \kl_{\mathrm{NPE}}(q;P) \qquad \text{and} \qquad \kl_{\mathrm{cal}}^S(q;P) \leq  \kl_{\mathrm{NPE}}(q;P).
\]
Further, suppose that $S(\btheta;\x)=f(q(\btheta\mid\x))$ for some strictly increasing $f:[0, \infty) \to\mathbb{R}$, and $F^Q_{S\mid\x}$ is also strictly increasing. Then
\[
\sup_{\alpha \in (0,1)} \kl_{\mathrm{cov}}^\alpha(q; P) \le \kl_{\mathrm{cal}}^S(q;P) \leq \kl_{\mathrm{NPE}}(q;P).
\]
\end{lemma}

\begin{proof}
Showing that $\kl_{\mathrm{cov}}^\alpha(q; P) \leq \kl_{\mathrm{NPE}}(q; P)$ and  $\kl_{\mathrm{cal}}^S(q; P) \leq \kl_{\mathrm{NPE}}(q; P)$ is a direct application of the data-processing inequality (DPI) from \Cref{lemma:chain-rule-and-DPI-for-KL}. 
In particular, since $\kl_{\mathrm{cov}}^\alpha(q; P) \leq \kl_{\mathrm{NPE}}(q; P)$ holds for any $\alpha \in (0,1)$, we can further conclude that $\sup_{\alpha \in (0,1)} \kl_{\mathrm{cov}}^\alpha(q; P) \leq \kl_{\mathrm{NPE}}(q; P)$.
We then focus on showing $\kl_{\mathrm{cov}}^\alpha(q; P) \leq \kl_{\mathrm{cal}}^S(q; P)$.

First, whenever $S(\btheta; \x)=f(q(\btheta \mid \x))$ and $f$ is strictly increasing, then 
\begin{align*}
(T_{\HPDR})_{\#} P_{\btheta \mid \x}(1) &= P_{\btheta \mid \x}(\HPDR) \\
&=  P_{\btheta \mid \x}(\{q(\btheta \mid \x) \geq \tau_{\alpha}^q(\x) \}) \\
&=P_{\btheta \mid \x}(\{ f(q(\btheta \mid \x)) \geq f(\tau_{\alpha}^q(\x)) \}) \quad \text{since }f \text{ is strictly increasing,} \\
&= P_{S \mid \x} (\{ s \geq f(\tau_{\alpha}^q(\x))\}) \quad \text{using that } S(\btheta; \x) = f(q(\btheta \mid \x)), \\
&= (T_{A_\alpha^q(\x)} )_{\#}P_{S \mid \x}(1),
\end{align*}
where $A_\alpha^q(\x) := \{ s \geq f(\tau_\alpha^q(\x)) \}$, and similarly,
\[
(T_{\HPDR})_{\#} P_{\btheta \mid \x}(0) = P_{S \mid \x} (\{ s < f(\tau_{\alpha}^q(\x))\}) = (T_{A_\alpha^q(\x)} )_{\#}P_{S \mid \x}(0)
\]
Therefore, the laws $(T_{\HPDR})_{\#} P_{\btheta \mid \x}$ and $(T_{A_\alpha^q(\x)} )_{\#}P_{S \mid \x}$ coincide. This also holds for $(T_{\HPDR})_{\#} Q_{\btheta \mid \x}$ and $(T_{A_{\alpha}^{q}(\x)})_{\#} Q_{S \mid \x}$. 
Hence, we can rewrite miscoverage as follows:
\begin{align*}
    \kl_{\mathrm{cov}}^\alpha(q; P) =\mathbb{E}_{\x \sim P_{\x}}\left [\KL \left ( (T_{A_\alpha^q(\x)} )_{\#}P_{S \mid \x} \;\Big \| \; (T_{A_\alpha^q(\x)} )_{\#}Q_{S \mid \x} \right) \right].
\end{align*}
But,
\begin{align*}
    T_{A_{\alpha}^q(\x)}(s) &= \mathbf{1}\{ s \geq f(\tau^q_\alpha(\x)) \} \\
    &= \mathbf{1}\{F^Q_{S \mid \x}(s) \geq F^Q_{S \mid \x}(f(\tau^q_{\alpha}(\x)))\} \quad \text{since } F_{S \mid \x}^Q \text{ is strictly increasing}, \\
    &= T_{B^q_{\alpha}(\x)}(F^Q_{S \mid \x}(s))
\end{align*}
for $B_\alpha^q(\x) := [F^Q_{S \mid \x}(f(\tau^q_\alpha(\x))), 1]$. Therefore,
\[
(T_{A_\alpha^q(\x)} )_{\#}P_{S \mid \x} = (T_{B_{\alpha}^q(\x)})_{\#} \Big( (F^Q_{S \mid \x})_{\#} P_{S \mid \x} \Big ),
\]
and similarly for $Q_{S \mid \x}$. 
Thus, 
\begin{align*}
    \kl_{\mathrm{cov}}^\alpha(q; P) &= \mathbb{E}_{\x \sim P_{\x}}\left [ \KL\left( (T_{B_{\alpha}^q(\x)})_{\#} \Big( (F^Q_{S \mid \x})_{\#} P_{S \mid \x} \Big )  \Big\| (T_{B_{\alpha}^q(\x)})_{\#} \Big( (F^Q_{S \mid \x})_{\#} Q_{S \mid \x} \Big )  \right) \right] \\[0.5em]
    & \leq \mathbb{E}_{\x \sim P_{\x}}\left [ \KL\left(  (F^Q_{S \mid \x})_{\#} P_{S \mid \x}  \Big\|  (F^Q_{S \mid \x})_{\#} Q_{S \mid \x}   \right) \right] \quad \text{by the DPI,} \\[0.5em]
    &= \kl_{\mathrm{cal}}^S(q; P).
\end{align*}
Since this line of reasoning holds for any $\alpha \in (0,1)$, we further conclude that
\[
\sup_{\alpha \in (0,1)}  \kl_{\mathrm{cov}}^\alpha(q; P) \leq  \kl_{\mathrm{cal}}^S(q; P).
\]
\end{proof}

\section{Theory: DRO-NPE method}
\label{app:theory-wasserstein-DRO}
In \Cref{appendix:theorem-W-DRO-upper-bound-chain}, we adapt a Wasserstein concentration result and use it to prove \Cref{theorem:W-DRO-upper-bound-chain}.
Then, in \Cref{appendix:proof-upper-bound}, we specialise to the Wasserstein distance with $p=2$, and derive a local quadratic upper bound under a growth condition and a continuous differentiability assumption, which in turn yields \Cref{proposition:W-DRO-upper-bound}.
In \Cref{appendix:excess-miscov-miscal}, we further prove an excess risk type of result for miscoverage and miscalibration under stricter conditions.
Finally, we verify these conditions for normalising flows \citep{papamakarios2021normalizing, kobyzev2020normalizing}.

\subsection{Wasserstein concentration results and proof of \Cref{theorem:W-DRO-upper-bound-chain}} \label{appendix:theorem-W-DRO-upper-bound-chain}
In this subsection, we first adapt a result from \citet{fournier2015rate} to our setting.
Then, we use this result, as well as the data-processing inequality (\Cref{lemma:DPI-for-KL-div}), to prove \Cref{theorem:W-DRO-upper-bound-chain}.
\begin{lemma}[Adapting Theorem 2 from \citet{fournier2015rate}]
\phantomsection \label{lemma:wasserstein-concentration-ineq}
    Suppose that $p \geq 1$, and that there exists $\beta > 2p$ such that $M_\beta := \mathbb{E}_{\z \sim P}[\| \z \|^\beta] < \infty$. Then, $\forall \varsigma \in (0,\beta)$ there exist positive constants $C_1, C_2$, with $C_2 \geq 1$, depending only on $p, d_{\Z}, \beta, M_{\beta}, \varsigma$ such that for any $\delta \in (0,1)$,
    \[
    \mathbb{P}(W_p(P_n, P) \leq \varepsilon) \geq 1-\delta,
    \]
    whenever $\varepsilon \geq \varepsilon(n, \delta) := \max\{\varepsilon_a(n,\delta),\, \varepsilon_b(n,\delta)\}$, where, writing $\rho := \left(\frac{1}{C_1}\log\frac{2C_2}{\delta}\right)^{1/2}$,
    \[
    \varepsilon_a(n,\delta) :=
    \begin{cases}
    \rho^{1/p}\, n^{-1/(2p)}, & \text{ if }\; d_{\Z} < 2p, \\[8pt]
    \left(2\rho\, n^{-1/2}\log\!\left(2+\dfrac{n}{\rho^2}\right)\right)^{1/p}, & \text{ if }\; d_{\Z} = 2p, \\[8pt]
    \rho^{2/d_{\Z}}\, n^{-1/d_{\Z}}, & \text{ if }\; d_{\Z} > 2p,
    \end{cases}
    \qquad
    \varepsilon_b(n,\delta) := \left(\frac{2C_2}{\delta}\right)^{\frac{1}{\beta-\varsigma}} n^{\frac{1}{\beta-\varsigma}-\frac{1}{p}}.
    \]
\end{lemma}

\begin{proof}
From Theorem 2 of \citet{fournier2015rate}, whenever $P$ is a probability distribution on $\mathbb{R}^{d_{\Z}}$, if there exists $\beta > 2p$ such that \( M_\beta:=\mathbb{E}_{\z \sim P}[\| \z \|^\beta ] < \infty \), then for all $n\geq 1$, $\varepsilon>0$,
\[
\mathbb{P}(W_p^p(P_n, P) \geq \varepsilon) \leq a(n, \varepsilon) \mathbf{1}_{\varepsilon \leq 1} + b(n,\varepsilon),
\]
where 
\[
a(n, \varepsilon) = C_2\begin{cases}
    \exp(-C_1 n \varepsilon^2), & \text{ if } \; d_{\Z} < 2p, \\
    \exp\left(-C_1 n \left(\frac{\varepsilon}{\log \left(2 + \frac{1}{\varepsilon}\right)}\right)^2\right), & \text{ if }\; d_{\Z} = 2p, \\
    \exp\left(-C_1 n \varepsilon^{\frac{d_{\Z}}{p}}\right), & \text{ if }\; d_{\Z} > 2p,
\end{cases}
\]
and 
\[
    b(n, \varepsilon) = C_2 n (n \varepsilon)^{-\frac{\beta - \varsigma}{p}} \quad \forall \varsigma \in (0, \beta).
\]
The positive constants $C_1,C_2$ depend only on $p, d_{\Z}$ and on $\beta, M_\beta, \varsigma$, and we may take $C_2 \geq 1$ without loss of generality.
Therefore,
\[
\mathbb{P}(W_p(P_n, P) \leq \varepsilon) \geq 1 - \mathbb{P}(W_p^p(P_n, P) \geq \varepsilon^p) \geq 1 - a(n, \varepsilon^p) \mathbf{1}_{\varepsilon^p \leq 1} - b(n, \varepsilon^p).
\]

Fix $\delta \in (0,1)$. We wish to find conditions for $\varepsilon$ in terms of $\delta, n$ and other constants such that:
\[
\mathbb{P}(W_p(P_n, P) \leq \varepsilon)  \geq 1 - \delta.
\]
We do so by selecting $\varepsilon$ such that it satisfies both $a(n, \varepsilon^p)\mathbf{1}_{\varepsilon^p \leq 1} \leq \delta /2, \; b(n, \varepsilon^p) \leq \delta / 2$, so that:
\[
1 - a(n, \varepsilon^p) \mathbf{1}_{\varepsilon^p \leq 1} - b(n, \varepsilon^p) \geq 1 - \delta \implies \mathbb{P}(W_p(P_n, P) \leq \varepsilon)  \geq 1 - \delta.
\]
The bound on $b$ does not depend on the regime of $p$:
\[
b(n, \varepsilon^p)=C_2 n (n \varepsilon^p)^{-\frac{\beta - \varsigma}{p}} \leq \frac{\delta}{2} \iff \varepsilon \geq \varepsilon_b = \left(\frac{2C_2}{\delta} \right)^{\frac{1}{\beta-\varsigma}} n^{\frac{1}{\beta-\varsigma} - \frac{1}{p}}.
\]
For the bound on $a$, we distinguish three cases.

\underline{\textbf{Case 1: }$d_{\Z} > 2p$}
\[
a(n, \varepsilon^p) = C_2\exp(-C_1n \varepsilon^{d_{\Z}}) \leq \frac{\delta}{2} \iff \varepsilon \geq \varepsilon_a = \left(\frac{1}{C_1n} \log \frac{2C_2}{\delta} \right)^{\frac{1}{d_{\Z}}} = \rho^{2/d_{\Z}}\, n^{-1/d_{\Z}}.
\]

\underline{\textbf{Case 2: }$d_{\Z} < 2p$}
\[
a(n,\varepsilon^p) = C_2\exp(-C_1 n \varepsilon^{2p}) \leq \frac{\delta}{2} \iff \varepsilon \geq \varepsilon_a = \left(\frac{1}{C_1n} \log \frac{2C_2}{\delta} \right)^{\frac{1}{2p}} = \rho^{1/p}\, n^{-1/(2p)}.
\]

\underline{\textbf{Case 3: }$d_{\Z} = 2p$}
\[
a(n,\varepsilon^p) = C_2\exp\left(-C_1 n \left(\frac{\varepsilon^p}{\log(2+\varepsilon^{-p})}\right)^2\right) \leq \frac{\delta}{2}
\iff \frac{\varepsilon^p}{\log(2+\varepsilon^{-p})} \geq \rho_n,
\]
where $\rho_n := \rho\, n^{-1/2}$. Let $y_n := 2\rho_n\log(2+\rho_n^{-2}) = 2\rho\, n^{-1/2}\log\!\left(2+\frac{n}{\rho^2}\right)$. We show that $\varepsilon^p \geq y_n$ suffices.
If $\rho_n \geq 1$, then $y_n \geq 2\rho_n \log 2 > 1$, so $\varepsilon^p > 1$ and $\mathbf{1}_{\varepsilon^p \leq 1} = 0$.
If $\rho_n < 1$, since $f(y) := y/\log(2+1/y)$ is increasing on $(0,\infty)$, it suffices to show $f(y_n) \geq \rho_n$. As $2\log 2 > 1 > \rho_n$, we have $y_n \geq 2\rho_n\log 2 \geq \rho_n^2$, so $1/y_n \leq \rho_n^{-2}$ and
\[
\rho_n\log(2+1/y_n) \leq \rho_n\log(2+\rho_n^{-2}) = y_n/2 \leq y_n,
\]
i.e.\ $f(y_n) \geq \rho_n$. Hence $\varepsilon \geq \varepsilon_a = y_n^{1/p}$ suffices.

We then select, in all three cases,
\[
\varepsilon(n, \delta):= \max \left\{ \varepsilon_a(n,\delta), \; \varepsilon_b(n,\delta) \right\}.
\]
\end{proof}

We are now ready to prove the Theorem.
\begin{proof}[Proof of \Cref{theorem:W-DRO-upper-bound-chain}]
    By \Cref{lem:miscal-upper-bound-miscov}, we have that for any $S$, any $q_\phi$ and every $\alpha \in (0,1)$:
\[
\kl_{\mathrm{cov}}^\alpha(q_\phi; P) \leq \kl_{\mathrm{cal}}^S(q_\phi; P) \leq \NPEloss(q_\phi; P) + K(P).
\]
Taking the supremum over $\alpha \in (0,1)$ gives the first part of the theorem.
Further, the assumptions of \Cref{lemma:wasserstein-concentration-ineq} are satisfied. The lemma holds for any $\varsigma \in (0,\beta)$; for readability, we fix $\varsigma := \beta-2p \in (0,\beta)$, so that $\beta - \varsigma = 2p$ and we recover $\varepsilon_b(n,\delta)$, as well as constants $C_1, C_2$ depending only on $p, d_{\Z}, \beta, M_\beta$, as in \Cref{theorem:W-DRO-upper-bound-chain}. We obtain that whenever $\varepsilon\geq \varepsilon(n,\delta)$,
\[
\mathbb{P}(W_p (P, P_n) \leq \varepsilon) \geq 1 - \delta,
\]
which implies that with probability at least $1-\delta$, $P  \in\mathcal{A}_p(P_n; \varepsilon)$. This gives that with probability at least $1-\delta$, for any $q_\phi$,
\[
\NPEloss(q_\phi; P) \leq \sup_{\tilde P \in\mathcal{A}_p(P_n;  \varepsilon)} \NPEloss(q_\phi; \tilde{P}).
\]
Hence, putting the inequalities together, with probability at least $1-\delta$, for any $q_\phi$,
\[
\kl_{\mathrm{cov}}^\alpha(q_\phi; P) \leq \kl_{\mathrm{cal}}^S(q_\phi; P) \leq \NPEloss(q_\phi; P) + K(P) \leq \sup_{\tilde P \in\mathcal{A}_p(P_n; \varepsilon)} \NPEloss(q_\phi; \tilde{P}) + K(P). 
\]
\end{proof}

\subsection{Proof of \Cref{proposition:W-DRO-upper-bound}} \label{appendix:proof-upper-bound}

We first show that under the assumption of $2-$growth, we can prove a local quadratic upper bound, so long as the function class is also twice continuously differentiable on $\mathbb{R}^{d_{\Z}}$.
\begin{lemma}
\phantomsection\label{lem:quad-up}
    Fix $q_\phi \in Q$. Suppose that $\z \mapsto \ell(q_\phi; \z) \in \mathcal{G}_{2}(\Z)$, and that it is twice continuously differentiable  on $\mathbb{R}^{d_\Z}$. Then, for any fixed $\z_0 \in \Z$, there exists a $C_\phi(\z_0)<\infty$ which is constant in $\z$ such that
    \[
    \ell(q_\phi; \z) \leq \ell(q_\phi; \z_0) + \nabla_{\z} \ell(q_\phi; \z_0)^\top (\z - \z_0) + C_{\phi}(\z_0) \|\z - \z_0\|^2, \quad \forall \z \in \Z.
    \]
\end{lemma}
\begin{proof}
We proceed by considering two cases: when $\| \z - \z_0\| < 1$ and $\| \z - \z_0\| \geq 1$. We denote $B(\z_0, 1):= \{\z : \| \z - \z_0 \| < 1 \}$, and $\bar{B}(\z_0, 1)$ its closure.

\textbf{Case 1: $\| \z - \z_0\| < 1$}: Since $B(\z_0, 1)\subset \mathbb{R}^{d_{\Z}}$, and the mapping $\z \mapsto \ell(q_\phi; \z)$ is continuously differentiable on $\mathbb{R}^{d_{\Z}}$, which is open and convex, Taylor's theorem (Theorem 2.68 of \citet{folland2024advanced}) applies at $\z_0$ for every $\z \in B(\z_0, 1)$:
\[
\ell(q_\phi; \z) = \ell(q_\phi; \z_0) + \nabla_{\z} \ell(q_\phi; \z_0) ^\top (\z - \z_0) + R_1(\z_0, \z - \z_0), \quad \forall \z \in B(\z_0, 1)
\]
where $R_1(\z_0, \z- \z_0)$ is the Taylor expansion remainder term. 
Further, since $\bar{B}(\z_0, 1)$ is compact and the Hessian $\nabla_{\z}^2 \ell(q_\phi; \z)$ is continuous by assumption, the Hessian has finite norm: $\| \nabla_{\z}^2 \ell(q_\phi; \z) \| \leq M(\z_0) \; \forall \z \in \bar{B}(\z_0, 1)$ for some constant $M(\z_0)$ independent of $\z$.
Therefore, with Corollary 2.75 of \citet{folland2024advanced}, we obtain:
\[
|R_1(\z_0, \z - \z_0)| \leq \frac {M(\z_0)} 2 \| \z - \z_0 \|^2, \quad \forall \z \in B(\z_0, 1).
\]
We then conclude:
\[
\ell (q_\phi; \z) \leq \ell (q_\phi; \z_0) + \nabla_{\z} \ell (q_\phi; \z_0)^\top (\z - \z_0) + \frac {M(\z_0)}{2}  \| \z - \z_0 \|^2, \quad \forall \z \in B(\z_0, 1).
\]

\textbf{Case 2: $\| \z - \z_0\| \geq 1$}: Notice that 
\[
- \nabla_{\z} \ell(q_\phi; \z_0)^\top  (\z - \z_0) \leq \| \nabla_{\z} \ell(q_\phi; \z_0) \| \|  \z - \z_0 \| \leq \frac{1}{2} \| \nabla_{\z} \ell(q_\phi; \z_0)  \|^2 + \frac{1}{2} \| \z - \z_0 \|^2,
\]
where the first inequality is by Cauchy-Schwarz, and the second holds because 
whenever $a, b>0$, $ab \leq \frac12 a^2 + \frac12 b^2$. 
Further, using the growth condition of the map $\z \mapsto \ell(q_\phi; \z)$ for $\rho=2$, for some constant $\tilde C_\phi$,
\[
\ell(q_\phi; \z) \leq \tilde C_\phi(1 +   \| \z_0 + \z - \z_0\|^2)  \leq \tilde C_\phi(1 + 2 \| \z_0\|^2 + 2 \| \z - \z_0\|^2),
\]
where the second inequality is using that $\|a + b\|^2 \leq 2 \| a \|^2 + 2 \|b\|^2$. Together, these results imply:
\begin{align*}
 \ell(q_\phi; \z) - \nabla_{\z} \ell(q_\phi; \z_0)^\top  (\z - \z_0) &\leq \tilde C_\phi(1 + 2 \| \z_0\|^2 + 2 \| \z - \z_0\|^2) \\
 &+ \frac{1}{2} \| \nabla_{\z} \ell(q_\phi; \z_0)  \|^2 + \frac{1}{2} \| \z - \z_0 \|^2 + \ell(q_\phi; \z_0) - \ell(q_\phi; \z_0) \\
 &:= K_\phi(\z_0) + \ell(q_\phi; \z_0) +  \left(2 \tilde C_\phi + \frac12\right) \| \z - \z_0\|^2,
\end{align*}
for some constant $K_\phi(\z_0)$ independent of $\z$, defined as:
\[
K_\phi(\z_0):= \tilde C_\phi + 2 \tilde C_\phi \| \z_0 \|^2 + \frac12 \| \nabla_{\z} \ell(q_\phi; \z_0) \|^2 - \ell(q_\phi; \z_0).
\]
Finally, since $\|z - \z_0\| \geq 1$, we have that
\[
K_\phi(\z_0) + \left(2 \tilde C_\phi + \frac12\right) \| \z - \z_0\|^2 \leq  \left(|K_\phi(\z_0)| +2 \tilde C_\phi + \frac12\right) \| \z - \z_0\|^2 := A_\phi(\z_0) \| \z - \z_0\|^2,
\]
implying
\[
\ell(q_\phi; \z)  \leq \ell(q_\phi; \z_0) +  \nabla_{\z} \ell(q_\phi; \z_0)^\top  (\z - \z_0) + A_\phi(\z_0) \| \z - \z_0\|^2, \quad \forall \z \in \{ \z \in \Z : \|\z - \z_0\| \geq 1\}.
\]

Finally, taking $C_\phi(\z_0) := \max \left\{\frac{M(\z_0)}{2}, A_\phi (\z_0) \right\}$, we get that
\[
\ell(q_\phi; \z)  \leq \ell(q_\phi; \z_0) +  \nabla_{\z} \ell(q_\phi; \z_0)^\top  (\z - \z_0) + C_\phi(\z_0) \| \z - \z_0\|^2 \quad \forall \z \in \Z.
\]
\end{proof}

This quadratic upper bound can then be used to solve the inner supremum in the strong dual of the Wasserstein DRO objective and prove \Cref{proposition:W-DRO-upper-bound}. 
Note that the proof keeps the loss function more general, with $\ellNPE$ recovered by setting $\ell \equiv \ellNPE$.

\begin{proof}[Proof of \Cref{proposition:W-DRO-upper-bound}] We divide this proof into two parts. First, we derive the upper bound in terms of the gradient regulariser and the $C_\phi^n$ remainder term. Second, we show that the remainder term remains negligible under the assumptions considered in the Proposition.
We keep the proof at a general level of loss function $\ell$, and do not use $\ellNPE$ specifically.

First, recall that, as in \Cref{eq:wasserstein-strong-dual}, under the upper semi-continuity and $2-$growth assumptions on $\ell$, which hold in this setting, we have \citep{gao2023distributionally}:
\[
\sup_{\tilde P \in\mathcal{A}_p(P_n; \varepsilon)} \mathbb{E}_{\z \sim \tilde P}[\ell(q_\phi; \z)] = \inf_{\lambda \geq 0}\; \Bigg \{ \lambda \varepsilon^2 + \frac{1}{n}\sum_{i=1}^n \sup_{\z' \in \Z}\; \Big ( \ell(q_\phi; \z') - \lambda \|\z' - \z_i\|^2 \Big) \Bigg \}.
\]
Further, by \Cref{lem:quad-up}, for each $i$ and every $\z' \in \Z$,
    \[
    \ell(q_\phi; \z') \leq \ell(q_\phi; \z_i) + \nabla_{\z} \ell(q_\phi; \z_i)^\top (\z' - \z_i) + C_\phi(\z_i) \| \z' - \z_i\|^2.
    \]
    Hence, for every $\lambda \geq 0$,
    \begin{align*}
    \sup_{\z' \in \Z}\; \Big ( \ell(q_\phi; \z') - \lambda \|\z' - \z_i\|^2 \Big)  &\leq \ell(q_\phi; \z_i) + \sup_{\z' \in \Z}\; \Big \{ \nabla_{\z} \ell(q_\phi; \z_i)^\top (\z' - \z_i) \\[-1em]
    &\qquad\qquad\qquad\qquad\qquad\qquad+ (C_{\phi}(\z_i) - \lambda) \|\z' - \z_i\|^2 \Big \} \\
    & \leq \ell(q_\phi; \z_i) + \sup_{\z' \in \mathbb{R}^{d_{\Z}}}\; \Big \{ \nabla_{\z} \ell(q_\phi; \z_i)^\top (\z' - \z_i) \\[-1em]
    &\qquad\qquad\qquad\qquad\qquad\qquad+ (C_{\phi}(\z_i) - \lambda) \|\z' - \z_i\|^2 \Big \}.
    \end{align*}
    Now, restrict $\lambda$ such that $\lambda > C_\phi(\z_i), \forall i=1,\dots, n$. Then, since the objective inside the supremum is strictly concave and over $\mathbb{R}^{d_{\Z}}$, we get that the solution for every $i$ to the following is the unique maximiser $\z^\star_i$: 
    \begin{align*}
        0&=\nabla_{\z'}\Big ( \nabla_{\z} \ell(q_\phi; \z_i)^\top (\z' - \z_i) + (C_{\phi}(\z_i) - \lambda) \|\z' - \z_i\|^2 \Big) \\
        &= \nabla_{\z} \ell(q_\phi; \z_i) + 2(C_\phi(\z_i) - \lambda )( \z' - \z_i) \implies \z^\star_i = \z_i +  \frac{\nabla_{\z} \ell(q_\phi; \z_i)}{2(\lambda - C_\phi(\z_i))}.
    \end{align*}
    Fix $C^n_\phi:=\max_{1\leq i\leq n} C_\phi(\z_i)$. Then, substituting back, we get that for every $\lambda > C^n_\phi$,
    \begin{align*}
        &\lambda \varepsilon^2 + \frac{1}{n} \sum_{i=1}^n \sup_{\z' \in \Z}\; \Big ( \ell(q_\phi; \z') - \lambda \|\z' - \z_i\|^2 \Big) \\
        &\leq \frac{1}{n}\sum_{i=1}^n \ell(q_\phi; \z_i) + \lambda \varepsilon^2  +\frac{1}{4(\lambda - C^n_\phi)} \cdot \frac{1}{n}\sum_{i=1}^n \| \nabla_{\z} \ell(q_\phi; \z_i) \|^2  \\
        &= \frac{1}{n}\sum_{i=1}^n \ell(q_\phi; \z_i) + C^n_\phi \varepsilon^2 + t \varepsilon^2 + \frac{1}{4 t} \cdot \frac{1}{n}\sum_{i=1}^n \| \nabla_{\z} \ell(q_\phi; \z_i) \|^2 ,
    \end{align*}
    where in the last line we substitute $t = \lambda - C^n_\phi$. Now,
    \begin{align*}
        &\inf_{\lambda \geq 0}\; \Bigg \{ \lambda \varepsilon^2 + \frac{1}{n}\sum_{i=1}^n \sup_{\z' \in \Z}\; \Big ( \ell(q_\phi; \z') - \lambda \|\z' - \z_i\|^2 \Big) \Bigg \} \\
        &\leq \inf_{\lambda - C^n_\phi > 0}\; \Bigg \{ \lambda \varepsilon^2 + \frac{1}{n}\sum_{i=1}^n \sup_{\z' \in \Z}\; \Big ( \ell(q_\phi; \z') - \lambda \|\z' - \z_i\|^2 \Big) \Bigg \} \\
        &\leq \inf_{t > 0}  \Bigg \{ \frac{1}{n}\sum_{i=1}^n \ell(q_\phi; \z_i) + C^n_\phi \varepsilon^2 + t \varepsilon^2 + \frac{1}{4 t} \cdot \frac{1}{n}\sum_{i=1}^n \| \nabla_{\z} \ell(q_\phi; \z_i) \|^2 \Bigg \}\\
        &= \frac{1}{n}\sum_{i=1}^n \ell(q_\phi; \z_i) + \varepsilon \left (\frac{1}{n} \sum_{i=1}^n \| \nabla_{\z} \ell(q_\phi; \z_i)\|^2 \right)^{1/2} + C^n_\phi \varepsilon^2,
    \end{align*} 
    where in the first line we restrict the infimum to $\lambda > C^n_\phi$, in the second line we use the previously derived upper-bound and replace $t=\lambda - C^n_\phi$, and in the last line, we use that for some $A, \varepsilon$ constant in $t$,
    \[
    \inf_{t > 0} \left(t \varepsilon^2 + \frac{A}{4t}\right) = \varepsilon \sqrt{A}.
    \]

For the remainder term, recall that 
    \[
    C_\phi^n := \max_{1 \leq i \leq n} C_\phi(\z_i), \quad C_\phi(\z_i):= \max \left\{\frac{1}{2} \sup_{\z \in \bar{B}(\z_i, 1)} \| \nabla_{\z}^2 \ell(q_\phi; \z)\|, \; 2 \tilde{C}_\phi + \frac{1}{2} + |K_{\phi}(\z_i)| \right\},
    \]
    with
    \[
    K_\phi(\z_i)= \tilde C_\phi + 2 \tilde C_\phi \| \z_i \|^2 + \frac12 \| \nabla_{\z} \ell(q_\phi; \z_i) \|^2 - \ell(q_\phi; \z_i).
    \]
    Then, using the growth condition for $\ell, \nabla_{\z} \ell$, there exist $C_{0,\phi} < \infty$ and $C_{1,\phi} < \infty$ such that 
    \[
    K_\phi(\z_i) \leq \tilde C_\phi + 2 \tilde C_\phi \|\z_i\|^2 + \frac12 (C_{1,\phi} + C_{1, \phi}\|\z_i\|^\beta)^2 + C_{0, \phi}(1 + \|\z_i\|^2) \leq C ( 1 + \| \z_i \|^{\max \{2, 2\beta \}})
    \]
    for some $C < \infty$ independent of $\z_i$.
    Further, from the growth assumption for $\nabla^2 \ell$ there exists $C_{2,\phi} < \infty$ such that:
    \[
    \frac{1}{2} \sup_{\z \in \bar{B}(\z_i, 1)} \| \nabla^2 \ell (q_\phi; \z)\| \leq \frac{1}{2} \sup_{\z \in \bar{B}(\z_i, 1)} C_{2, \phi}(1 + \| \z \|^\alpha) \leq \frac{C_{2, \phi}} {2}(1 + ( 1 + \|\z_i \|)^\alpha) \leq C' (1 + \|\z_i\|^\alpha)
    \]
    for some constant $C'$ independent of $\z_i$.
    Hence, for any $\z_i$, there exists some $A_\phi < \infty$, such that 
    \[
    C_\phi(\z_i) \leq A_\phi (1 + \| \z_i \|^m), \quad m:=\max \{2, 2\beta , \alpha \}. 
    \]
    Therefore, we obtain:
    \[
    C_\phi^n =\max_{1 \leq i \leq n} C_\phi(\z_i) \leq A_\phi (1 + \max_{1 \leq i \leq n}\|\z_i\|^m) := \bar C_\phi^n. 
    \]
    Now, we turn to the light-tail assumption. 
    Recall that \( \mathbb{E}_{\z \sim P}[\exp(\eta \| \z \|^a)] \leq M \).
    By Markov's inequality, for $t > 0$,
    \[
    \mathbb{P}(\| \z \| > t) =\mathbb{P}(\exp(\eta \| \z \|^{a}) > \exp(\eta t^{a})) \leq \frac{\mathbb{E}_{\z \sim P}[\exp(\eta \| \z \|^a)]}{\exp(\eta t^a)} \leq M \exp(- \eta t^a).
    \]
    Hence,
    \begin{align*}
    \mathbb{P}(\max_{1 \leq i \leq n}\| \z_i \| \leq t) &= \mathbb{P}(\|\z_1 \| \leq t, \dots, \|\z_n \| \leq t) \\
    &=1 - \mathbb{P}(\{\|\z_1\| > t\} \cup \dots \cup \{\|\z_n\| > t\}) \\
    &\geq 1 - \sum_{i=1}^n \mathbb{P}(\|\z_i\| > t) \geq 1 - nM \exp(-\eta t^a).
    \end{align*}
    Now, since $m > 0$, $\|\z \|^m$ is increasing, so we have that $\max_i \| \z_i \|^m = (\max_i \| \z_i \|)^m$. Hence, 
    \[
    \mathbb{P}\left(\max_{1 \leq i \leq n} \| \z_i \| \leq t \right) = \mathbb{P}\left((\max_{1 \leq i \leq n} \| \z_i \|)^m \leq t^m \right) = \mathbb{P}\left(\max_{1 \leq i \leq n} \|\z_i \|^m \leq t^m \right) = \mathbb{P}\left(\bar C_\phi^n  \leq A_\phi(1 + t^m)\right)
    \]
    so that
    \[\mathbb{P}\left(\bar C_\phi^n  \leq A_\phi(1 + t^m)\right) \geq 1 - n M \exp(-\eta t^a).
    \]
    For some fixed $\delta > 0$, we then have that with probability at least $1-\delta$
    \[
    C_\phi^n \leq \bar C_\phi^n  \leq A_\phi \left(1 + \left( \frac1\eta \log \frac{M n}{\delta}\right)^{m/a} \right).
    \]
    Selecting $\varepsilon_n = o((\log n)^{-m / 2a})$ then yields the desired result:
    \[
    C_\phi^n \varepsilon_n^2 \to 0 \quad \text{in probability under }P.
    \]
\end{proof}

\subsection{Excess risk, miscoverage, and miscalibration} \label{appendix:excess-miscov-miscal}
In this section, we investigate the excess risk, miscoverage and miscalibration rates of the density estimator obtained by DRO-NPE. To obtain such theoretical results, we require a stricter assumption on the pointwise loss function $\ell$: Lipschitz gradient, as stated in \Cref{assump:lip-gradient}.

\begin{assumption}[Lipschitz gradient. \citet{gao2023finite}, Assumption 2]
For every $q_\phi \in \Q$, the map $\z \mapsto \ell(q_\phi;\z)$ is differentiable on $\Z$, and there exists $h>0$ such that \(
\|\nabla_{\z}\ell(q_\phi;\z_1)-\nabla_{\z}\ell(q_\phi;\z_2)\|
\le h\|\z_1-\z_2\|,
\; \forall \z_1,\z_2\in\Z,\ \forall q_\phi\in\Q. \)
\label{assump:lip-gradient}
\end{assumption}

The first result, \Cref{theorem:DRO-surrogate-vs-oracle}, shows the excess miscoverage and miscalibration error rates of the DRO-NPE estimator when the objective is in its exact (or equivalently, strong dual) form relative to the oracle minimiser, which corresponds to the density estimator that minimises the population risk.
The proof of \Cref{theorem:DRO-surrogate-vs-oracle} relies on the excess DRO risk (or optimality gap), a quantity commonly studied in DRO \citep[see e.g.][]{zeng2022generalization, gao2023finite}, and uses a hybrid of Wasserstein concentration results from \citet{fournier2015rate} and generalisation bounds from \citet{gao2023finite}.

\begin{theorem}[DRO-NPE miscoverage and miscalibration rate]
\phantomsection \label{theorem:DRO-surrogate-vs-oracle}
Define the oracle NPE and DRO minimisers as follows:
\[
\phi^\star := \arg \min_{\phi} \NPEloss(q_\phi; P), \quad \phi^{\mathrm{DRO}}_n := \arg \min_{\phi} \sup_{\tilde P \in\mathcal{A}_2(P_n; \varepsilon)} \NPEloss(q_\phi; \tilde P).
\]
Also let
\begin{align*}
   G_{\mathrm{cov}}^\alpha(q_\phi; P) :&= \kl_{\mathrm{NPE}}(q_\phi; P) - \kl_{\mathrm{cov}}^\alpha(q_\phi; P) \\
   &= \mathbb{E}_{\x \sim P_{\x}}\left[ \mathbb{E}_{t \sim (T_{\HPDR})_{\#} P_{\btheta \mid \x}}\left[\KL(P_{\btheta \mid  T_{\HPDR}(\btheta)=t, \x} \| Q_{\btheta \mid  T_{\HPDR}(\btheta)=t, \x})\right] \right] \geq 0, \\
  G_{\mathrm{cal}}^S(q_\phi; P) &:= \kl_{\mathrm{NPE}}(q_\phi; P) - \kl_{\mathrm{cal}}^S(q_\phi; P) \\
  &= \mathbb{E}_{\x \sim P_{\x}}\left[
\mathbb{E}_{u \sim (F^Q_{S \mid \x})_{\#} P_{S \mid \x}}\left[
\KL\!\left(
P_{\btheta \mid F^Q_{S \mid \x}(S(\btheta;\x))=u, \x}
\,\middle\|\,
Q_{\btheta \mid F^Q_{S \mid \x}(S(\btheta;\x))=u, \x}
\right)
\right]\right] \geq 0.
\end{align*}
%
Suppose the assumptions from \Cref{theorem:W-DRO-upper-bound-chain} hold and pick $\varepsilon \geq \varepsilon(n, \delta)$ accordingly, for some $\delta>0$.
Further, suppose \Cref{assump:lip-gradient} holds for $\ellNPE(q_\phi; \z) = - \log q_\phi(\btheta \mid \x)$ .
Then,  with probability at least $1-\delta$, for any $\alpha \in (0,1)$,
\begin{equation}
    \kl_{\mathrm{cov}}^\alpha(q_{\phi^{\mathrm{DRO}}_n}; P) - \kl_{\mathrm{cov}}^\alpha(q_{\phi^\star}; P) \leq 2 \varepsilon\; \mathbb{E}_{\z \sim P}[\| \nabla \ellNPE (q_{\phi^\star}; \z)\|^2]^{\frac{1}{2} }+ 4 h \varepsilon^2  + G_{\mathrm{cov}}^\alpha(q_{\phi^\star}; P)
\end{equation}
with $h$ as in \Cref{assump:lip-gradient}, and also:
\begin{equation}
    \kl_{\mathrm{cal}}^S(q_{\phi^{\mathrm{DRO}}_n}; P) - \kl_{\mathrm{cal}}^S(q_{\phi^\star}; P) \leq 2 \varepsilon\; \mathbb{E}_{\z \sim P}[\| \nabla \ellNPE (q_{\phi^\star}; \z)\|^2]^{\frac{1}{2} }+ 4 h \varepsilon^2  + G_{\mathrm{cal}}^S(q_{\phi^\star}; P)
\end{equation}
\end{theorem}

\begin{proof}
We pick $\varepsilon \geq \varepsilon(n, \delta)$, as stated in \Cref{theorem:W-DRO-upper-bound-chain}.
Then, with probability at least $1-\delta$,
\begin{align*}
    \NPEloss(q_{\phi^{\mathrm{DRO}}_n}; P) \leq \sup_{\tilde P \in\mathcal{A}_2(P_n; \varepsilon)} \NPEloss(q_{\phi^{\mathrm{DRO}}_n}; \tilde P) \leq \sup_{\tilde P \in\mathcal{A}_2(P_n; \varepsilon)} \NPEloss(q_{\phi^\star}; \tilde P),
\end{align*}
where the first inequality holds because the concentration result ensures that $P \in\mathcal{A}_2(P_n; \varepsilon)$, and the second holds by the definition of the DRO minimiser $\phi^{\mathrm{DRO}}_n$.
We further conclude that with probability $1-\delta$, 
\[
\sup_{\tilde P \in\mathcal{A}_2(P_n; \varepsilon)} \NPEloss(q_{\phi^\star}; \tilde P) \leq \sup_{\tilde P \in\mathcal{A}_2(P_n; 2\varepsilon)} \NPEloss(q_{\phi^\star}; \tilde P),
\]
since $W_2(\tilde P, P) \leq W_2(\tilde P, P_n) + W_2( P_n, P) \leq 2 \varepsilon$ by the triangle inequality, implying $\mathcal{A}_2(P_n; \varepsilon) \subseteq \mathcal{A}_2(P_n; 2\varepsilon)$.

Now, we may use Lemma 2 of \cite{gao2023finite}, which states that whenever $\z \mapsto \ell(q_\phi; \z)$ has an $h-$Lipschitz gradient, $\forall \varepsilon\geq 0$, for any distribution $Q$ on $\Z$,
\[
\left|\sup_{\tilde P \in\mathcal{A}_2(P_n; \varepsilon)} \mathbb{E}_{\z \sim \tilde P}[\ell(q_\phi; \z)] - \mathbb{E}_{\z \sim Q}[\ell(q_\phi; \z)] -  \varepsilon \,\mathbb{E}_{\z \sim Q}\left[\| \nabla_{\z} \ell(q_\phi; \z)\|^2 \right]^{\frac{1}{2}} \right| \leq \varepsilon^2 h.
\]
We use this result with $\ellNPE(q_\phi; \z):= -\log q_\phi(\btheta \mid \x)$.
Then, with probability at least $1-\delta$,
\[
\sup_{\tilde P \in\mathcal{A}_2(P_n;  2\varepsilon)} \NPEloss(q_{\phi^\star}; \tilde P) - \NPEloss(q_{\phi^\star}; P) \leq 2\varepsilon \,\mathbb{E}_{\z \sim P}\left[\| \nabla_{\z} \ellNPE(q_{\phi^\star}; \z)\|^2 \right]^{\frac{1}{2}} + 4\varepsilon^2 h.
\]
We then obtain the \emph{DRO excess risk} as follows:
\begin{align}
  \NPEloss(q_{\phi^{\mathrm{DRO}}_n}; P) -  \NPEloss(q_{\phi^\star}; P) \leq 2\varepsilon \,\mathbb{E}_{\z \sim P}\left[\| \nabla_{\z} \ellNPE(q_{\phi^\star}; \z)\|^2 \right]^{\frac{1}{2}} + 4\varepsilon^2 h.
  \label{eq:dro-excess-risk}
\end{align}

Now, recall that
\[
\kl_{\mathrm{NPE}}(q_\phi; P) :=  \mathbb{E}_{\x \sim P_{\x}}\left [\KL \Big(P_{\btheta \mid \x} \;\|\; Q_{\btheta \mid \x}\Big)\right]  \equiv \NPEloss(q_\phi; P) + K(P).
\]
By \Cref{lem:miscal-upper-bound-miscov}, for any $\alpha \in (0,1)$ and $q_\phi$, $G_{\mathrm{cov}}^\alpha(q_\phi; P) \geq 0$. Similarly, for any choice of $S$ and $q_\phi$, $G_{\mathrm{cal}}^S(q_\phi; P) \geq 0$.
Therefore,
\begin{align*}
&\kl_{\mathrm{cov}}^\alpha(q_{\phi^{\mathrm{DRO}}_n}; P) - \kl_{\mathrm{cov}}^\alpha(q_{\phi^\star}; P) \\
&\leq \kl_{\mathrm{NPE}}(q_{\phi^{\mathrm{DRO}}_n}; P) - \kl_{\mathrm{cov}}^\alpha(q_{\phi^\star}; P) \\[0.5em]
&= \NPEloss(q_{\phi^{\mathrm{DRO}}_n}; P) - \NPEloss(q_{\phi^{\star}}; P) +\; \kl_{\mathrm{NPE}}(q_{\phi^\star}; P)  - \kl_{\mathrm{cov}}^\alpha(q_{\phi^\star}; P) \\[0.5em]
    &\leq 2\varepsilon \,\mathbb{E}_{\z \sim P}\left[\| \nabla_{\z} \ellNPE(q_{\phi^\star}; \z)\|^2 \right]^{\frac{1}{2}} + 4\varepsilon^2 h + G_{\mathrm{cov}}^\alpha(q_{\phi^\star}; P),
\end{align*}
and similarly for calibration,
\[
\kl_{\mathrm{cal}}^S(q_{\phi^{\mathrm{DRO}}_n}; P) - \kl_{\mathrm{cal}}^S(q_{\phi^\star}; P) \leq 2\varepsilon \,\mathbb{E}_{\z \sim P}\left[\| \nabla_{\z} \ellNPE(q_{\phi^\star}; \z)\|^2 \right]^{\frac{1}{2}} + 4\varepsilon^2 h + G_{\mathrm{cal}}^S(q_{\phi^\star}; P).
\]
\end{proof}

Next, in \Cref{lemma:NPE-vs-DRO-excess-risk}, we derive an excess risk bound under additional assumptions for the NPE empirical risk minimiser as well as for the regularised, first-order DRO-NPE objective, so that they can be compared.
The proof pattern for \Cref{lemma:NPE-vs-DRO-excess-risk} is almost the same for NPE and DRO-NPE, both using the generalisation bound in Corollary 3 of \citet{gao2023finite} and sharing the same remainder terms and sample-dependent DRO radius.
Note that although all of the below results are in terms of $\ellNPE, \NPEloss$, similar proof patterns can be used for more general risks and pointwise losses $\mathcal L, \ell$.

Before proving \Cref{lemma:NPE-vs-DRO-excess-risk}, we first define the following NPE and DRO-NPE minimisers, and establish the additional \Cref{as:cor-2-gao} needed for \Cref{lemma:NPE-vs-DRO-excess-risk}.
\begin{align*}
    &\hat{\phi}_n:= \arg \min_{\phi} \NPEloss(q_\phi; P_n) \\
    &\phi_{n}^{\mathrm{DRO-NPE}}:= \arg \min_{\phi}\; \NPEloss(q_\phi; P_n) + \varepsilon
\left(
\frac{1}{n}\sum_{i=1}^n \|\nabla_{\z}\ellNPE(q_\phi;\z_i)\|^2
\right)^{\frac{1}{2}} .
\end{align*}
We now state the additional assumptions needed for \Cref{lemma:NPE-vs-DRO-excess-risk}:
\begin{assumption}[Assumptions of Corollary 3 \cite{gao2023finite}] \label{as:cor-2-gao}Suppose that:
\begin{enumerate}
    \item[(i)]$P$ satisfies the transportation-information inequality $T_2(\tau)$ i.e. there exists $\tau > 0$ such that $$W_2(P,Q) \leq \sqrt{\tau \KL(Q||P)} \quad \forall Q \in \mathcal{P}_2.$$ 

    \item[(ii)] 
    There exists measurable functions $\kappa: \mathcal{Z} \rightarrow \mathbb{R}_{+}$ and $ \kappa_2: \Z \to \mathbb{R}_+$ and constants $\kappa_M, \kappa_L, \kappa_{2,M}, \kappa_{2, L} \geq 0$ satisfying $\kappa(\z) \leq \kappa_M + \kappa_L \|\z\|^2$ and $\kappa_2(\z) \leq \kappa_{2,M} + \kappa_{2,L} \|\z\|^2$ for all $\z \in \mathcal{Z}$, such that: 
    \begin{align*}
        |\ell(q_{\tilde{\phi}};\z) - \ell(q_{\phi};\z)| \leq \kappa(\z) \| \tilde{\phi} - \phi \| \quad \mathrm{for \,all}\;  \phi,\tilde{\phi} \in \Phi,\quad P-a.e. \quad \z \in \mathcal{Z} \\
        |\nabla_{\z} \ell(q_{\tilde{\phi}};\z) - \nabla_{\z} \ell(q_{\phi};\z)| \leq \kappa_2(\z) \| \tilde{\phi} - \phi \| \quad \mathrm{for \,all}\;  \phi,\tilde{\phi} \in \Phi,\quad P-a.e. \quad \z \in \mathcal{Z}
    \end{align*}
    \item[(iii)]  Let $N(\varepsilon; \Phi, \|\cdot\|)$ denote the covering number of set $\Phi$ equipped with norm $\|\cdot\|$, defined as the smallest cardinality of an $\varepsilon$-cover of $\Phi$, where $\Phi_{\varepsilon}$ is an $\varepsilon$-cover of $\Phi$ if for each $\phi \in \Phi$, there exists $\tilde{\phi} \in \Phi_\varepsilon$ such that $\|\tilde{\phi} - \phi \| \leq \varepsilon$. We assume that the parameter space admits a finite covering number at scale $\frac{1}{n}$, i.e. $\log N(\frac{1}{n}; \Phi, \|\cdot\|) < \infty$.
    \item[(iv)] Assume $\sigma = \sup_{\phi \in \Phi} \frac{ \left( \mathbb{E}_{\z \sim P} \left[ \|\nabla_{\z}\,\ell(q_\phi;\z)\|_2^4 \right] \right)^{\frac{1}{2}} }{ \mathbb{E}_{\z \sim P} \left[ \|\nabla_{\z}\,\ell(q_\phi;\z)\|_2^2 \right]} < \infty$.
\end{enumerate}
\end{assumption}
As noted by \citet{gao2023finite}, Assumption~\ref{as:cor-2-gao}(i) is implied by a log-Sobolev inequality.
This condition holds, for example, for Gaussian and strongly log-concave distributions, and more generally for many light-tailed distributions with Gaussian-type concentration. 
Assumption~\ref{as:cor-2-gao}(ii) is a Lipschitz continuity assumption in $\phi$ for the loss and its gradient, with a Lipschitz constant allowed to grow at most quadratically in $z$. 
Thus, changes in the parameter $\phi$ induce controlled changes in the loss and its gradient, up to a quadratic envelope in the data variable.
For example, by the mean value theorem, (ii) holds whenever $\Phi$ is compact and convex and $\|\nabla_\phi \ell(q_\phi; z)\|$ and $\|\nabla_\phi \nabla_z \ell(q_\phi; z)\|$ grow at most quadratically in $z$ uniformly over $\Phi$; it can fail when $\Phi$ is unbounded, since the loss may then become arbitrarily sensitive to $\phi$.
Finally, Assumption~\ref{as:cor-2-gao}(iii) controls the complexity of the parameter space, and hence, together with Assumption~\ref{as:cor-2-gao}(ii), the induced loss class $\{ \ell(q_\phi; \cdot) : \phi \in \Phi \}$. 
Since $\Phi \subseteq \mathbb{R}^{d_{\Phi}}$, the assumption is satisfied, for example, whenever $\Phi$ is bounded. 
Intuitively, (iii) requires that finitely many points suffice to place every $\phi \in \Phi$ within distance $1/n$ of one of them; for bounded $\Phi \subset \mathbb{R}^{d_\Phi}$, the logarithm of the number of points needed grows only as $d_\Phi \log n$.
Finally, Assumption~\ref{as:cor-2-gao}(iv) is a uniform moment condition on the input-gradient norm, and is satisfied when \(\|\nabla_{\z}\ell(q_\phi;\z)\|_2\) has uniformly controlled fourth moment relative to its second moment under \(P\), over \(\phi\in\Phi\).

Having now stated the assumptions, before presenting \Cref{lemma:NPE-vs-DRO-excess-risk}, we define a few quantities that will appear in the Lemma.
Let \(t>0\) and \(n>8\sigma^2t\). Define the sample-dependent DRO radius \(\varepsilon_n\), the moment-based concentration remainder term \(\rho_n\), and the covering-number, function class complexity penalty term $\varrho_n$ as:
\begin{align*}
\varepsilon_n&:=\sqrt{\frac{\tau\!\left(t+\log\!\bigl(1+\mathcal{N}(1/n;\Phi,\|\cdot\|)\bigr)\right)}{n}}\left(1+\sigma\sqrt{\frac{2\!\left(t+\log\!\bigl(1+\mathcal{N}(1/n;\Phi,\|\cdot\|)\bigr)\right)}{n}}\right),\\
\rho_n&:=\frac{2\mathbb{E}_{\z\sim P}[\kappa_2]+\sqrt{\mathbb{V}_{\z\sim P}[\kappa_2]}+\varepsilon_n\sqrt{\mathbb{E}_{\z\sim P}[\kappa_2^2]+\sqrt{\mathbb{V}_{\z\sim P}[\kappa_2^2]}}}{n}\\
\varrho_n&:= \frac{\tau ( t + \log(1 + \mathcal{N}(\frac{1}{n}; \Phi, \| \cdot \|)))}{n}
\end{align*}
\begin{lemma}[NPE vs DRO-NPE excess risk]
\phantomsection \label{lemma:NPE-vs-DRO-excess-risk}
    Assume that \Cref{assump:lip-gradient} and \Cref{as:cor-2-gao} hold with $\ellNPE$ and $\varepsilon_n, \rho_n, \varrho_n$ are defined as in the above. Then, with probability at least $1 - 2/n - 2e^{-t}$,
    \begin{align*}
    &\NPEloss(q_{\hat{\phi}_n}; P) - \NPEloss(q_{\phi^\star}; P) \\
      & \leq   \varepsilon_n \left( \frac{1}{n}\sum_{i=1}^n\| \nabla_{\z} \ellNPE (q_{\phi^{\star}}; \z_i)\|^2\right)^{\frac{1}{2}} + \left( \frac{1}{n}\sum_{i=1}^n\| \nabla_{\z} \ellNPE (q_{\hat{\phi}_n}; \z_i)\|^2\right)^{\frac{1}{2}} + 2\rho_n  + 2h\varrho_n.
    \end{align*}
    Moreover, under the same assumptions:
    \begin{align*}
        \NPEloss(q_{\phi_{n}^{\mathrm{DRO-NPE}}}; P) - \NPEloss(q_{\phi^{\star}}; P) \leq  2 \varepsilon_n \left( \frac{1}{n}\sum_{i=1}^n\| \nabla_{\z} \ellNPE (q_{\phi^{\star}}; \z_i)\|^2\right)^{\frac{1}{2}} + 2\rho_n  + 2h\varrho_n
    \end{align*}
\end{lemma}
\begin{proof}
It follows from Corollary 3 of \citet{gao2023finite} that with probability at least $1-\frac{2}{n} - 2e^{-t}$:
\begin{align} \label{eq:cor-2-gao}
\mathbb{E}_{\z \sim P}[\ellNPE(q_\phi; \z)] \leq \frac{1}{n} \sum_{i=1}^n \ellNPE(q_\phi; \z_i) + \varepsilon_n \left(\frac{1}{n} \sum_{i=1}^n\| \nabla_{\z} \ellNPE (q_{\phi}; \z_i)\|^2\right)^{\frac{1}{2}} + \rho_n +  h\varrho_n, \, \forall \phi,
\end{align}
Note that the same result may be applied to the function $-\ellNPE(q_\phi; \z)$; this will be used below.
Importantly, since the bound holds uniformly over $\phi$, we may use it to handle $\hat{\phi}_n$.

Hence, applying the bound for $\ellNPE$ and $\hat{\phi}_n$,
\begin{align*}
    &\NPEloss(q_{\hat{\phi}_n}; P) - \NPEloss(q_{\phi^\star}; P) \\
    &\leq  \NPEloss(q_{\hat{\phi}_n}; P_n) + \varepsilon_n \left( \frac{1}{n}\sum_{i=1}^n\| \nabla_{\z} \ellNPE (q_{\hat{\phi}_n}; \z_i)\|^2\right)^{\frac{1}{2}} + \rho_n  + h\varrho_n - \NPEloss(q_{\phi^\star}; P) \\
    & \leq  \NPEloss(q_{\phi^{\star}}; P_n) + \varepsilon_n \left( \frac{1}{n}\sum_{i=1}^n\| \nabla_{\z} \ellNPE (q_{\hat{\phi}_n}; \z_i)\|^2\right)^{\frac{1}{2}} + \rho_n  + h\varrho_n - \NPEloss(q_{\phi^\star}; P) \\
    & \leq   \varepsilon_n \left( \left( \frac{1}{n}\sum_{i=1}^n\| \nabla_{\z} \ellNPE (q_{\phi^\star}; \z_i)\|^2\right)^{\frac{1}{2}} + \left( \frac{1}{n}\sum_{i=1}^n\| \nabla_{\z} \ellNPE (q_{\hat{\phi}_n}; \z_i)\|^2\right)^{\frac{1}{2}} \right) + 2\rho_n  + 2h\varrho_n ,
\end{align*}
where in the first inequality we use the corollary's bound for $\phi=\hat{\phi}_n$; in the second inequality we use the fact that $\hat{\phi}_n$ minimises $\NPEloss(q_\phi; P_n)$; and in the third inequality we use \Cref{eq:cor-2-gao},  applied to the pointwise $-\ellNPE(q_{\phi^\star}; \z)$.
Similarly,  
\begin{align*}
    &\NPEloss(q_{\phi_{n}^{\mathrm{DRO-NPE}}}; P) - \NPEloss(q_{\phi^{\star}}; P) \\
    &\leq \NPEloss(q_{\phi_n^{\mathrm{DRO-NPE}}}; P_n) + \varepsilon_n \left( \frac{1}{n}\sum_{i=1}^n\| \nabla_{\z} \ellNPE (q_{\phi^{\mathrm{DRO-NPE}}}; \z_i)\|^2\right)^{\frac{1}{2}} + \rho_n  + h\varrho_n - \NPEloss(q_{\phi^\star}; P) \\
    & \leq  \NPEloss(q_{\phi^{\star}}; P_n) + \varepsilon_n \left( \frac{1}{n}\sum_{i=1}^n\| \nabla_{\z} \ellNPE (q_{\phi^\star}; \z_i)\|^2\right)^{\frac{1}{2}} + \rho_n  + h\varrho_n - \NPEloss(q_{\phi^\star}; P) \\
    & \leq  2 \varepsilon_n \mathbb{E}_{\z \sim P_n}[\| \nabla_{\z} \ellNPE (q_{\phi^{\star}}; \z)\|^2]^{\frac{1}{2}} + 2\rho_n  + 2h\varrho_n
\end{align*}
where we follow the same steps as before, except for the second inequality, where we use that $\phi^{\mathrm{DRO-NPE}}_n$ minimises $\NPEloss(q_{\phi}; P_n) + \varepsilon_n \left(
\frac{1}{n}\sum_{i=1}^n \|\nabla_{\z}\ellNPE(q_\phi;\z_i)\|^2
\right)^{\frac{1}{2}}$.

\end{proof}

\Cref{lemma:NPE-vs-DRO-excess-risk} derives an excess risk upper bound for NPE and for the regularised objective of DRO-NPE. The two bounds are derived in similar ways and differ only in the regulariser term $\Omega(\phi; P_n)$, which can be thought of as an empirical measure of functional variation via the loss gradient.
Indeed, the key distinction between the DRO and NPE bounds derived is that the DRO-NPE bound only depends on the oracle variation $\Omega(\phi^\star; P_n)$.
In contrast, the NPE bound depends on both the oracle variation and the variation of the learned empirical minimiser, $\Omega(\hat{\phi}_n; P_n)$.
This means that the NPE bound can be affected by the roughness of the empirical minimiser.
Thus, the DRO guarantee is sharper whenever the empirical minimiser has larger variation than the oracle.

\subsection{Verifying conditions for normalising flows (proof of \Cref{lem:flows-assumptions})
} 
\label{appendix:verify-cond-norm-flows}
In this section, we formally introduce MAFs and coupling flows and prove \Cref{lem:flows-assumptions}, showing that all assumptions of \Cref{proposition:W-DRO-upper-bound} are satisfied when $q_\phi$ is chosen to be one of these commonly used normalising flows. 
\paragraph{Masked autoregressive flows.} We first write the masked autoregressive flows (MAFs; \citet{NIPS2017_6828}) for $q_\phi(\btheta\mid\x)$, implemented in the \textit{sbi} library \citep{tejero2020sbi}, defined through a single autoregressive transform $u=u(\btheta,\x)\in\R^{d_\Theta}$ with coordinates
\begin{equation*}
u_i(\btheta,\x)
:= \frac{\theta_i-\mu_i(\btheta_{<i},\x)}{\sigma_i(\btheta_{<i},\x)},
\qquad \btheta_{<i}:=(\theta_1,\dots,\theta_{i-1}).
\label{eq:maf-layer}
\end{equation*}
We set
\begin{align*}
\mu_i(\btheta_{<i},\x)
&=v_i^\top \g_i^{(L)}(\btheta_{<i},\x)+c_i,
\\[0.5em]
\g_i^{(\ell)}(\btheta_{<i},\x)
&=\tanh\!\big(B_i^{(\ell)}\g_i^{(\ell-1)}(\btheta_{<i},\x)+d_i^{(\ell)}\big),
\quad \ell=1,\dots,L,
\\[0.5em]
\g_i^{(0)}(\btheta_{<i},\x)
&=[\btheta_{<i};\x],
\\[2em]
\sigma_i(\btheta_{<i},\x)
&=\text{softplus}\!\big(s_i(\btheta_{<i},\x)\big),
\qquad \text{softplus}(t):=\log(1+e^t),
\\[0.5em]
s_i(\btheta_{<i},\x)
&=w_i^\top \h_i^{(L)}(\btheta_{<i},\x)+b_i,
\\[0.5em]
\h_i^{(\ell)}(\btheta_{<i},\x)
&=\tanh\!\big(A_i^{(\ell)}\h_i^{(\ell-1)}(\btheta_{<i},\x)+a_i^{(\ell)}\big),
\qquad \ell=1,\dots,L.
\\[0.5em]
\h_i^{(0)}(\btheta_{<i},\x)
&=[\btheta_{<i};\x],
\end{align*}
for some depth $L$, weight matrices $\{B_i^{(\ell)}\}_{\ell=1}^L$ and $\{A_i^{(\ell)}\}_{\ell=1}^L$, biases $\{d_i^{(\ell)}\}_{\ell=1}^L$ and $\{a_i^{(\ell)}\}_{\ell=1}^L$, and final weights $(v_i,c_i)$ and $(w_i,b_i)$; all of these parameters are part of $\phi$. Here $[\btheta_{<i};\x]$ denotes vector concatenation. 
Finally, we let the base density be standard Gaussian, i.e.
\[
p_0(u)=(2\pi)^{-d_\Theta/2}\exp\!\left(-\frac12\|u\|^2\right).
\]
Hence,
\begin{align*}
\log q_\phi(\btheta\mid\x)
&= -\frac12\|u(\btheta,\x)\|^2 - \sum_{i=1}^{d_\Theta}\log \sigma_i(\btheta_{<i},\x) - \frac{d_\Theta}{2}\log 2\pi \\[0.5em]
&=-\frac12 \sum_{i=1}^{d_{\Theta}} \left (\frac{\theta_i - \mu_i}{\sigma_i} \right)^2 - \sum_{i=1}^{d_\Theta}\log \sigma_i - \frac{d_\Theta}{2}\log 2\pi.
\label{eq:maf-logdens}
\end{align*}
For simplicity, in what follows, we will use the notation \( \mu_i:=\mu_i(\btheta_{<i},\x)\), \( \sigma_i:=\sigma_i(\btheta_{<i},\x)\), \(\; \g_i^{(\ell)}:=\g_i^{(\ell)}(\btheta_{<i},\x)\), \(\h_i^{(\ell)}:=\h_i^{(\ell)}(\btheta_{<i},\x)\), and \(s_i:=s_i(\btheta_{<i},\x)  \).

\begin{lemma} \label{lem:MAF}
Let \(q_\phi(\btheta \mid \x)\) be a masked autoregressive flow of the form described above, with tanh hidden layers, softplus scale function, and standard Gaussian base density. 
Let \( \ellNPE(q_\phi;\z) := -\log q_\phi(\btheta \mid \x)\), \(\z=(\btheta,\x)\in\mathbb{R}^{d_\Z}\).
Then the map \(\z \mapsto \ell_{\mathrm{NPE}}(q_\phi;\z)\) is twice continuously differentiable on \(\mathbb{R}^{d_\Z}\), and there exists a constant \(C>0\), depending on \(\phi\), such that for all \(\z \in \mathbb{R}^{d_\Z}\),
\[
|\ell_{\mathrm{NPE}}(q_\phi;\z)| \le C(1+\|\z\|^2).
\]
Moreover, there exist constants \(C_1,C_2>0\), depending on \(\phi\), such that for all \(\z\in\mathbb{R}^{d_\Z}\)
\[
\|\nabla_{\z}\ell_{\mathrm{NPE}}(q_\phi;\z)\|
\le C_1(1+\|\z\|^2),
\qquad
\|\nabla_{\z}^2\ell_{\mathrm{NPE}}(q_\phi;\z)\|
\le C_2(1+\|\z\|^2).
\]
\end{lemma}
\begin{proof}
    We first show that the loss $\ellNPE$ has at most a quadratic growth rate.
     Recall that 
    \begin{align}
        \log q_\phi(\btheta\mid\x) = -\frac12 \sum_{i=1}^{d_{\Theta}} \left (\frac{\theta_i - \mu_i}{\sigma_i} \right)^2 - \sum_{i=1}^{d_\Theta}\log \sigma_i - \frac{d_\Theta}{2}\log 2\pi.
    \end{align}
Since the hidden activations are tanh, all hidden-layer outputs are uniformly bounded. Hence, for each $i$,
\begin{enumerate}
    \item there exists $C_{\mu, i} <\infty$ such that $|\mu_i(\btheta_{<i}, \x)| \leq C_{\mu, i}$ for all $\z$;
    \item $s_i(\btheta_{<i}, \x)$ is uniformly bounded, and therefore $0 < \sigma_{i, \text{min}} \leq \sigma_i(\btheta_{<i}, \x) \leq \sigma_{i, \text{max}} < \infty$  for all $\z$.
\end{enumerate}
Therefore, 
\begin{align*}
    \left ( \frac{\theta_i - \mu_i}{\sigma_i} \right)^2 \leq \left ( \frac{\theta_i - \mu_i}{\sigma_{i, \text{min}}} \right)^2 \leq  \frac{2\theta_i^2 + 2\mu_i^2}{\sigma_{i, \text{min}}^2}  \leq \frac{2\theta_i^2 + 2C_{\mu,i}^2}{\sigma_{i, \text{min}}^2}.
\end{align*}
Substituting this bound and using $|\log \sigma_i |\leq \max\{|\log \sigma_{i, \text{min}}|, |\log \sigma_{i, \text{max}}|\} := C_{\log\sigma}$, we get that
\begin{align*}
    \left|\log q_\phi(\btheta\mid\x)\right| \leq \frac{1}{2} \sum_{i=1}^{d_\Theta} \left( \frac{2\theta_i^2 + 2C_{\mu,i}^2}{\sigma_{i, \text{min}}^2} +C_{\log\sigma} \right) + \frac{d_\Theta}{2}\log 2\pi.
\end{align*}
Hence, there exists constants $A_\phi, B_\phi >0$ such that
\begin{align*}
    \left|\log q_\phi(\btheta\mid\x) \right|\leq A_\phi \|\btheta\|^2 + B_\phi.
\end{align*}
Since $\|\btheta\| \leq \|(\btheta, \x)\| = \|\z\|$, we obtain 
\begin{align*}
    \left|\ellNPE(q_\phi;\z)\right| \leq A_\phi \|\z\|^2 + B_\phi,
\end{align*}
showing that the MAF with log-loss has a finite $2-$growth rate. 

Next, we show that the map $\z \to \ellNPE$ is twice continuously differentiable on $\mathbb{R}^{d_\Z}$. Indeed, since each layer of the networks defining $\mu_i$ and $s_i$ consists of affine transformations followed by $\text{tanh}$, it follows that $\mu_i, s_i \in C^\infty(\mathbb{R}^{d_\Z})$. Since $\sigma_i = \text{softplus}(s_i)$ and $\text{softplus}$ is $C^\infty$ with strictly positive output, we have $\sigma_i \in C^\infty(\mathbb{R}^{d_\Z})$ and $\sigma_i(\z) > 0$ for all $\z \in \mathbb{R}^{d_\Z}$. Therefore, the map $\z \to \ellNPE$ is a composition of $C^\infty$ functions and is thus $C^\infty$ on $\mathbb{R}^{d_\Z}$.

We now verify the condition on the growth of gradients. Note that from the proof of Proposition 3 (see Equation 38) in \cite{bharti2026amortisedprovablyrobustsimulationbasedinference}, it follows by the same arguments after swapping $\btheta$ and $\x$ (since this proposition refers to the NLE case in which the MAF has an identical form after swapping $\btheta$ and $\x$ everywhere)  we have that there exists $C_0 < \infty$ such that:
\begin{align*}
    \left\|\nabla_{\btheta} \log q_{\phi}(\btheta \mid \x) \right\| \leq C_0 (1 + \|\btheta\|^2). 
\end{align*}
We now verify that a similar argument holds for $\|\nabla_{\x} \log q_{\phi}(\btheta \mid \x)\|$. In particular, for each $j =1, \dots, d_{\X}$:
\begin{align*}
    \frac{\partial}{\partial \x_j}\log q_\phi(\btheta\mid\x) &= -\frac12 \sum_{i=1}^{d_{\Theta}}  \frac{\partial}{\partial \x_j}\left(\frac{\theta_i - \mu_i}{\sigma_i} \right)^2 - \sum_{i=1}^{d_\Theta} \frac{\partial}{\partial \x_j}\log \sigma_i \\
    &= \sum_{i=1}^{d_\Theta} \frac{\theta_i - \mu_i}{\sigma_i^2} \frac{\partial \mu_i}{\partial \x_j} + \sum_{i=1}^{d_\Theta} \frac{(\theta_i - \mu_i)^2}{\sigma_i^3} \frac{\partial\sigma_i}{\partial \x_j} - \sum_{i=1}^{d_\Theta} \frac{1}{\sigma_i} \frac{\partial \sigma_i}{\partial \x_j}.
\end{align*}
Following Equation (32) of \cite{bharti2026amortisedprovablyrobustsimulationbasedinference} we have that since outputs of $\mu_i$ and $s_i$ are affine functions of $\h_i^{L}$, there exist $C_1, C_2 < \infty$ such that:
\begin{align*}
    \left| \mu_i(\btheta_{<i},\x)\right| \leq C_1, \quad \left\| \frac{\partial \mu_i} {\partial \x_j}\right\| \leq C_2.
\end{align*}
We now turn to the derivative of $\sigma_i$, recalling that
\begin{align*}
&\sigma_i(\btheta_{<i},\x)
=
\softplus \big(s_i(\btheta_{<i},\x)\big),
\qquad
s_i(\btheta_{<i},\x)
=
w_i^\top h_i^{(L)}(\btheta_{<i},\x)+b_i,\\
&\softplus'(t)=\frac{1}{1+\exp(-t)},
\qquad \qquad \; \,
\softplus''(t)=\frac{\exp(-t)}{(1+\exp(-t))^2},
\end{align*}
where, for each $\ell\in\{1,\dots,L\}$, \( h_i^{(\ell)} = \tanh \big(A_i^{(\ell)} h_i^{(\ell-1)}+a_i^{(\ell)}\big)\).
Hence,
\begin{align*}
\frac{\partial \sigma_i}{\partial \x_j}
&=
\frac{1}{1+\exp(-s_i)}
\frac{\partial s_i}{\partial \x_j}
=
\frac{1}{1+\exp(-s_i)}
\,w_i^\top
\frac{\partial h_i^{(L)}}{\partial \x_j}.
\end{align*}
Moreover, for each hidden layer,
\begin{align*}
\frac{\partial h_i^{(\ell)}}{\partial \x_j}
&=
\sech^2\!\big(A_i^{(\ell)} h_i^{(\ell-1)}+a_i^{(\ell)}\big)
\odot
\left(
A_i^{(\ell)}
\frac{\partial h_i^{(\ell-1)}}{\partial \x_j}
\right).
\end{align*}
where $\odot$ denotes element-wise multiplication of vectors.
Since i) $\sech^2$ is uniformly bounded, ii) the network has finitely many layers with fixed finite weight matrices, and iii) \( \partial h_i^{(0)}/\partial \x_j =e_j\),  it follows that $\partial h_i^{(\ell)}/\partial \x_j$ is uniformly bounded on $\X$ for every $\ell,i,j$. 
Therefore, both $\partial s_i/\partial \x_j$ and $\partial \sigma_i/\partial \x_j$ are uniformly bounded on $\mathbb{R}^{d_\X}$ and there exists $C_3 < \infty$ such that:
\begin{align*}
    \left\| \frac{\partial \sigma_i}{\partial \x_j}\right\| \leq C_3.
\end{align*}
Therefore, 
\begin{align*}
    \left | \frac{\partial}{\partial \x_j} \log q_\phi (\btheta \mid \x) \right| &\leq \frac{1}{\sigma_{i,\text{min}}^3} \sum_{i=1}^{d_\Theta} (C_2 |\theta_i - \mu_i| + |\theta_i - \mu_i|^2 C_3 + C_3 )
\end{align*}
and since $|\mu_i| \leq C_1$ it follows that there exist $C_4, C_5 < \infty$ such that: $|\theta_i - \mu_i| \leq C_4 (1 + |\theta_i|)$ and $|\theta_i - \mu_i|^2 \leq C_5 (1+ |\theta_i|^2)$. Thus, there exists $C_6 <\infty$ such that 
\begin{align*}
    \left | \frac{\partial}{\partial \x_j} \log q_\phi (\btheta \mid \x) \right| \leq \frac{1}{\sigma_{i,\text{min}}^3} \sum_{i=1}^{d_\Theta} (C_2 C_4 (1+|\theta_i|  + C_3C_5(1+|\theta_i|^2)  + C_3 ) 
    \leq C_6 (1+ \|\btheta\|^2). 
\end{align*}
Therefore, we have:
\begin{align*}
    \left\|\nabla_{\btheta} \log q_{\phi}(\btheta \mid \x) \right\| \leq C_0 (1 + \|\btheta\|^2), \quad \left\|\nabla_{\x} \log q_{\phi}(\btheta \mid \x) \right\| \leq C_6\sqrt{d_\X} (1+\|\btheta\|^2)
\end{align*}
and hence:
\begin{align*}
    \|\nabla_{\z} \ellNPE(q_\phi;\z)\| = \|\nabla_{\z}( -\log q_\phi(\btheta\mid\x))\| &\leq \left\|\nabla_{\btheta} \log q_{\phi}(\btheta \mid \x) \right\| + \left\|\nabla_{\x} \log q_{\phi}(\btheta \mid \x) \right\| \\
    &\leq C_0 (1 + \|\btheta\|^2) + C_6 \sqrt{d_\X}(1+\|\btheta\|^2) \\
    &:= C_7 (1+\|\btheta\|^2) \\
    &\leq C_7 (1 + \|\z\|^2).
\end{align*}
Lastly, we verify the growth condition for the norm of the Hessian. 

Recall that
\begin{align*}
-\log q_\phi(\btheta\mid\x)
=
\frac12 \sum_{i=1}^{d_\Theta} u_i^2
+
\sum_{i=1}^{d_\Theta}\log \sigma_i
+
\frac{d_\Theta}{2}\log 2\pi,
\end{align*}
where \( u_i=\frac{\theta_i-\mu_i}{\sigma_i}\).  Therefore, for any $j,k\in\{1,\dots,d_{\Z}\}$,
\begin{align}
\partial_{z_j}(-\log q_\phi(\btheta\mid\x))
&=
\sum_{i=1}^{d_\Theta}
\left(
u_i\,\partial_{z_j}u_i
+
\frac{1}{\sigma_i}\partial_{z_j}\sigma_i
\right),
\\
\partial_{z_k}\partial_{z_j}(-\log q_\phi(\btheta\mid\x))
&=
\sum_{i=1}^{d_\Theta}
\left(
u_i\,\partial_{z_k}\partial_{z_j}u_i
+
(\partial_{z_k}u_i)(\partial_{z_j}u_i)
-
\frac{1}{\sigma_i^2}(\partial_{z_k}\sigma_i)(\partial_{z_j}\sigma_i)
+
\frac{1}{\sigma_i}\partial_{z_k}\partial_{z_j}\sigma_i
\right).\label{eq:second-der-maf}
\end{align}
By points 1. and 2. above we have, \(
|u_i|
=
\left|\frac{\theta_i-\mu_i}{\sigma_i}\right|
\le \frac{|\theta_i| + C_{\mu,i}}{\sigma_{i,\text{min}}} \leq C_u(1 + |\theta_i|)
\) for $C_u := \max_i(\max\{1/\sigma_{i,\text{min}},C_{\mu,i}/\sigma_{i,\text{min}}\}) <\infty$. Moreover, since we have shown that both $|\partial_{z_j} \mu_i| \leq C_{\mu'}$ and $|\partial_{z_j} \sigma_i|\leq C_{\sigma'}$ where $C_{\mu'}$ and $C_{\sigma'}$ are constants independent of $\z$ and
\begin{align*}
\partial_{z_j}u_i
&=
\frac{\partial_{z_j}(\theta_i-\mu_i)}{\sigma_i}
-
\frac{\theta_i-\mu_i}{\sigma_i^2}\partial_{z_j}\sigma_i,
\end{align*}
it follows that $|\partial_{z_j} u_i| \leq C_{u'} (1 + |\theta_i|) \leq C_{u'}(1+|\theta_i|)$ for some $C_{u'} < \infty$ depending on $\sigma_{i,\text{min}}, C_{\mu'}$ and $C_{\sigma'}$.

We now consider the second derivative of $\sigma_i$:
\begin{align}
\frac{\partial^2 \sigma_i}{\partial z_k \partial z_j}
&=
\frac{\exp(-s_i)}{(1+\exp(-s_i))^2}
\frac{\partial s_i}{\partial z_k}
\frac{\partial s_i}{\partial z_j}
+
\frac{1}{1+\exp(-s_i)}
\frac{\partial^2 s_i}{\partial z_k \partial z_j}.
\label{eq:sigma-second-deriv}
\end{align}
Thus it remains to bound $\partial^2 s_i/(\partial z_k \partial z_j)$.
Since
\begin{align*}
\frac{\partial s_i}{\partial z_j}
=
w_i^\top \frac{\partial h_i^{(L)}}{\partial z_j},
\qquad
\frac{\partial^2 s_i}{\partial z_k \partial z_j}
=
w_i^\top \frac{\partial^2 h_i^{(L)}}{\partial z_k \partial z_j},
\end{align*}
it suffices to bound the second derivatives of the hidden states. Differentiating once more,
\begin{align*}
\frac{\partial^2 h_i^{(\ell)}}{\partial z_k \partial z_j}
&=
\frac{\partial}{\partial z_k}
\left[
\sech^2\!\big(A_i^{(\ell)} h_i^{(\ell-1)}+a_i^{(\ell)}\big)
\odot
\left(
A_i^{(\ell)}
\frac{\partial h_i^{(\ell-1)}}{\partial z_j}
\right)
\right]
\\
&=
\left[
-2\sech^2\!\big(A_i^{(\ell)} h_i^{(\ell-1)}+a_i^{(\ell)}\big)
\tanh\!\big(A_i^{(\ell)} h_i^{(\ell-1)}+a_i^{(\ell)}\big)
\right]
\odot
\left(
A_i^{(\ell)}
\frac{\partial h_i^{(\ell-1)}}{\partial z_k}
\right)
\odot
\left(
A_i^{(\ell)}
\frac{\partial h_i^{(\ell-1)}}{\partial z_j}
\right)
\\
&\qquad
+
\sech^2\!\big(A_i^{(\ell)} h_i^{(\ell-1)}+a_i^{(\ell)}\big)
\odot
\left(
A_i^{(\ell)}
\frac{\partial^2 h_i^{(\ell-1)}}{\partial z_k \partial z_j}
\right).
\end{align*}
Now, we have that: i)$\tanh$ and $\sech^2$ are uniformly bounded, ii) the first derivatives $\partial h_i^{(\ell-1)}/\partial z_j$ have already been shown to be bounded, iii) the flow has finitely many layers with fixed finite weight matrices, and iv)
\begin{align*}
\frac{\partial^2 h_i^{(0)}}{\partial z_k\partial z_j}=0.
\end{align*}
Hence, or each layer,
\begin{align*}
\left|
\frac{\partial^2 h_i^{(\ell)}}{\partial z_k \partial z_j}
\right|
<\infty
\qquad
\forall\, \ell,i,j,k,\; \z\in \mathbb{R}^{d_\Z}.
\end{align*}
Consequently, $\partial^2 s_i/(\partial z_k\partial z_j)$ is uniformly bounded on $\mathbb{R}^{d_{\Z}}$.
Returning to \eqref{eq:sigma-second-deriv}, and using that
\begin{align*}
0<
\frac{\exp(-s_i)}{(1+\exp(-s_i))^2}
\le \frac14,
\qquad
0<
\frac{1}{1+\exp(-s_i)}
<1,
\end{align*}
we conclude that
\begin{align*}
\left|
\frac{\partial^2 \sigma_i}{\partial z_k \partial z_j}
\right|
\le C_{\sigma''}
\end{align*}
for some finite constant $C_{\sigma''}$, uniformly over $\z\in \mathbb{R}^{d_\Z}$.
We now treat $\mu_i$. Recall that
\begin{align*}
\mu_i(\btheta_{<i},\x)
=
v_i^\top g_i^{(L)}(\btheta_{<i},\x)+c_i,
\end{align*}
where the hidden layers $g_i^{(\ell)}$ are again defined through finitely many linear maps composed with $\tanh$. Therefore, using the same argument as above, but without the final softplus layer, we get that
\begin{align*}
\left|
\frac{\partial^2 \mu_i}{\partial z_k \partial z_j}
\right|
\le C_{\mu''}
\end{align*}
for some finite constant $C_{\mu''}$, uniformly over $\z \in \mathbb{R}^{d_\Z}$.
Using the bounds on the first and second derivatives of $\mu_i$ and $\sigma_i$ we first have that
\begin{align*}
\partial_{z_j}u_i
&=
\frac{\partial_{z_j}(\theta_i-\mu_i)}{\sigma_i}
-
\frac{\theta_i-\mu_i}{\sigma_i^2}\partial_{z_j}\sigma_i,
\end{align*}
and hence $|\partial_{z_j}u_i|\le  C_{u'}(1+\|\btheta\|)$ for some positive constant $C_{u'} < \infty$. Differentiating once more,
\begin{align*}
\partial_{z_k}\partial_{z_j}u_i
&=
\frac{\partial_{z_k}\partial_{z_j}(\theta_i-\mu_i)}{\sigma_i}
-
\frac{\partial_{z_j}(\theta_i-\mu_i)}{\sigma_i^2}\partial_{z_k}\sigma_i
-
\frac{\partial_{z_k}(\theta_i-\mu_i)}{\sigma_i^2}\partial_{z_j}\sigma_i
\\
&\qquad
-
\frac{\theta_i-\mu_i}{\sigma_i^2}\partial_{z_k}\partial_{z_j}\sigma_i
+
2\frac{\theta_i-\mu_i}{\sigma_i^3}
(\partial_{z_k}\sigma_i)(\partial_{z_j}\sigma_i),
\end{align*}
so that $|\partial_{z_k}\partial_{z_j}u_i|\le C_{u''} (1+\|\btheta\|)$ for some positive constant $C_{u''} < \infty$ .
Returning to the expression in \Cref{eq:second-der-maf},
we have that for each of the four terms:
\begin{align*}
    \left| u_i \partial_{z_k} \partial_{z_j} u_i\right| &\leq C_u(1+\|\btheta\|) C_{u''}(1+\|\btheta\|) \leq C_{8}(1+\|\btheta\|^2) \\
    \left| (\partial_{z_k}u_i) (\partial_{z_j}u_i) \right| &\leq C_{u'}^2 (1+\|\btheta\|)^2 \leq C_{9} (1+\|\btheta\|^2)\\
    \left|\frac{1}{\sigma_i^2}(\partial_{z_k} \sigma_i)(\partial_{z_j} \sigma_i)\right| &\leq \sigma_{i,\text{min}}^{-2} C_{\sigma'}^2 := C_{10}\\
    \left|\frac{1}{\sigma_i} \partial_{z_k} \partial_{z_j} \sigma_i\right| &\leq 
    \sigma_{i,\text{min}}^{-1} C_{\sigma''}
    :=C_{11}
\end{align*}
for some finite constant $C_{8},C_{9}, C_{10}, C_{11} < \infty$.
Therefore, 
\begin{align*}
    \left|\partial_{z_k}\partial_{z_j}(-\log q_\phi(\btheta\mid\x)) \right| &\leq d_{\Theta}\left( 
    C_{8}(1+\|\btheta\|^2) +C_{9} (1+\|\btheta\|^2)+ C_{10} + C_{11}
    \right) \\
    &\leq C_{12}(1 + \|\btheta\|^2)\\& \leq C_{12}(1 + \|\z\|^2)
\end{align*}
for some positive constant $C_{12} < \infty$. To conclude note that, for $\|\cdot\|_F$ denoting the Frobenius norm, it follows that:
\begin{align*}
    \left\| \nabla^2_{\z} \ellNPE(q_\phi;\z)\right\| &\leq  \left\| \nabla^2_{\z} \ellNPE(q_\phi;\z)\right\|_F\\
    &= \sqrt{\sum_{j=1}^{d_\Z} \sum_{k=1}^{d_\Z} \left|\partial_{z_k}\partial_{z_j}(-\log q_\phi(\btheta\mid\x)) \right|^2 }\\
    &\leq \sqrt{d_\Z^2 C_{12}^2 (1+\|\z\|^2)^2} \\
    &= d_{\mathcal{Z}} C_{12}(1+\|\z\|^2).
\end{align*}
\end{proof}

\paragraph{Coupling flows.}
We next describe coupling flows \citep{Kingma2018_coupling, Radev2022} implemented in BayesFlow package \citep{kuhmichel2026bayesflow} for $q_\phi(\btheta\mid\x)$. 
First, the Actnorm layer, which is an invertible adaptation of batch normalisation for stable training, is applied to $\btheta$. It is followed by a permutation, multiplying by the permutation matrix $P \in \{0, 1\}^{d_\Theta \times d_\Theta}$:
$\tilde \btheta = P (\alpha \odot \btheta + \beta)$ where $\alpha, \beta \in \mathbb{R}^{d_\Theta}$ are trainable parameters, and $\odot$ denotes element-wise multiplication.

We then split the dimension into two parts, followed by affine transformation:
$\tilde \btheta = [\tilde \btheta_\mathcal{A}; \tilde \btheta_\mathcal{B}]$ where $\tilde \btheta_\mathcal{A} = \tilde \btheta_{1: \lfloor d_\Theta/2 \rfloor}$ is the first $\lfloor d_\Theta/2 \rfloor$ elements and 
$\tilde \btheta_\mathcal{B} =  \tilde \btheta_{\lfloor d_\Theta/2 \rfloor + 1: d_\Theta}$ is the last $d_\Theta - \lfloor d_\Theta/2 \rfloor$ elements of $\tilde \btheta$.
An affine coupling transformation $u=u(\tilde \btheta,\x)\in\mathbb{R}^{d_\Theta}$ is defined by 
\begin{equation*}
u(\tilde \btheta_\mathcal{A}, \x) :=  \sigma(\tilde \btheta_\mathcal{B},\x) \odot \tilde \btheta_\mathcal{A} + \mu(\tilde \btheta_\mathcal{B},\x).
\qquad
u(\tilde\btheta_\mathcal{B}, \x) := \tilde\btheta_\mathcal{B}.
\label{eq:coupling-layer}
\end{equation*}
Note that the forward coupling transformation is defined differently than for MAFs.

The parameters $\mu \in \mathbb{R}^{\lfloor d_\Theta/2 \rfloor}$ and $\sigma \in \mathbb{R}^{\lfloor d_\Theta/2 \rfloor}$ are defined in a similar way as in \textit{MAF} in sbi package. We use the tanh activation function instead of the default function implemented in BayesFlow (mish):
\begin{align*}
\sigma(\tilde \btheta_\mathcal{B},\x)
&=\text{softplus}\!\big( \sinh^{-1}(\tilde\sigma(\tilde \btheta_\mathcal{B},\x)) + \mathbf{1}\log (e - 1)\big),
\qquad 
\\[0.5em]
[\mu(\tilde \btheta_\mathcal{B}, \x) ;
\tilde \sigma (\tilde \btheta_\mathcal{B}, \x)]
&=W^\top \h^{(L)}( \tilde \btheta_\mathcal{B},\x)+b,
\\[0.5em]
\h^{(\ell)}(\tilde \btheta_\mathcal{B},\x)
&=\tanh\!\big(A^{(\ell)}\h^{(\ell-1)}(\tilde \btheta_\mathcal{B},\x)+a^{(\ell)}\big),
\qquad \ell=1,\dots,L.
\\[0.5em]
\h^{(0)}(\tilde \btheta_\mathcal{B},\x)
&=[\tilde \btheta_\mathcal{B};\x],
\end{align*} 

where $\text{softplus}, \text{sinh}^{-1}, \text{tanh}$ are applied element-wise.
The log-determinant of the Actnorm and the affine transformation is defined by
\[
\log\left|\det \frac{\partial \tilde\btheta}{\partial \btheta}\right| = \sum_{i=1}^{d_\Theta} \log | \alpha_i|
,\qquad
\log\left|\det \frac{\partial u}{\partial \tilde \btheta}\right|
=
\sum_{i=1}^{\lfloor d_\Theta/2 \rfloor}
\log \sigma_{i}(\tilde \btheta_\mathcal{B},\x)
\]

Finally, taking the base density to be standard Gaussian, given a single coupling block, we have
\begin{align*}
\log q_\phi(\btheta\mid\x)
&=
-\frac12
\left[
\sum_{i=1}^{\lfloor d_\Theta/2 \rfloor}
\left(
\sigma_i(\tilde \btheta_\mathcal{B},\x)\tilde \btheta_{\mathcal{A},i} + \mu_i(\tilde \btheta_\mathcal{B},\x)
\right)^2
+ \| \tilde \btheta_{\mathcal{B}}\|_2^2
\right]
\\&+ \sum_{i=1}^{d_\Theta} \log | \alpha_i|
+\sum_{i=1}^{\lfloor d_\Theta/2 \rfloor}
\log \sigma_{i}(\tilde \btheta_\mathcal{B},\x)
-\frac{d_\Theta}{2}\log 2\pi .
\end{align*}
In practice, the sequence of the Actnorm layer, the permutation, and the mapping from $\tilde \btheta$ to $u$ forms a block, which is repeated multiple times to increase the flow's expressiveness.

\begin{lemma}[Single BayesFlow coupling block]
\phantomsection \label{lem:single-coupling-growth}
Let $\z=(\btheta,\x)\in\mathbb{R}^{d_\Theta+d_\X}$, and consider the BayesFlow coupling block defined above by
\[
\tilde \btheta=P(\alpha\odot\btheta+\beta),\qquad
u_\mathcal{A}=\sigma(\tilde \btheta_\mathcal{B},\x)\odot \tilde \btheta_\mathcal{A}+\mu(\tilde \btheta_\mathcal{B},\x),\qquad
u_\mathcal{B}=\tilde \btheta_\mathcal{B},
\]
where $\mu$ and $\sigma$ are defined through the network architecture above. Then $\z\mapsto \ellNPE(q_\phi;\z)$ belongs to $\mathcal G_2(\mathcal Z)$ and is of class $C^2$. Moreover,
\[
\|\nabla_{\z}\ellNPE(q_\phi;\cdot)\|\in\mathcal G_2(\mathcal Z),
\qquad
\|\nabla_{\z}^2\ellNPE(q_\phi;\cdot)\|\in\mathcal G_2(\mathcal Z).
\]
\end{lemma}

\begin{proof}
By construction, the ActNorm parameters satisfy $\alpha_i\neq 0$ for all $i$. Moreover, since the hidden activations are $\tanh$ and the output layers are affine, the maps $\mu$ and $\tilde\sigma$ are bounded and belong to $C^2$, with uniformly bounded first and second derivatives. Since
\[
\sigma=\softplus\!\big(\sinh^{-1}(\tilde\sigma)+\mathbf 1\log(e-1)\big),
\]
the same holds for $\sigma$. As $\tilde\sigma$ is bounded, there exist constants
\[
0<\sigma_{\min}\le \sigma_i(\tilde\btheta_\mathcal{B},\x)\le \sigma_{\max}<\infty
\]
for all $i$ and $(\tilde\btheta_\mathcal{B},\x)$. Hence there exist constants
\[
C_\mu,C_{\mu'},C_{\mu''},C_{\sigma'},C_{\sigma''}<\infty
\]
such that for all $i,j,k$,
\[
|\mu_i|\le C_\mu,\qquad |\partial_{z_j}\mu_i|\le C_{\mu'},\qquad |\partial_{z_k}\partial_{z_j}\mu_i|\le C_{\mu''},
\]
\[
|\partial_{z_j}\sigma_i|\le C_{\sigma'},\qquad |\partial_{z_k}\partial_{z_j}\sigma_i|\le C_{\sigma''}.
\]

Since $P$ is a permutation matrix, it preserves the Euclidean norm. Hence
\[
\|\tilde\btheta\|_2=\|P(\alpha\odot\btheta+\beta)\|_2=\|\alpha\odot\btheta+\beta\|_2
\le \|\text{diag}(\alpha)\|\,\|\btheta\|_2+\|\beta\|_2
\le C_0(1+\|\btheta\|_2)
\]
for some $C_0<\infty$. Also, $\tilde\btheta$ is affine in $\btheta$, so all first derivatives $\partial_{z_j}\tilde\btheta_i$ are bounded and all second derivatives $\partial_{z_k}\partial_{z_j}\tilde\btheta_i$ vanish.

We first control the loss. Since $u_\mathcal{B}=\tilde\btheta_\mathcal{B}$, we have
\[
\|u_\mathcal{B}\|_2\le \|\tilde\btheta\|_2\le C_0(1+\|\btheta\|_2).
\]
For each $i\le \lfloor d_\Theta/2\rfloor$,
\[
|u_{\mathcal A,i}|=|\sigma_i(\tilde\btheta_\mathcal{B},\x)\tilde\btheta_{\mathcal A,i}+\mu_i(\tilde\btheta_\mathcal{B},\x)|
\le \sigma_{\max}|\tilde\btheta_{\mathcal A,i}|+C_\mu
\le C_1(1+\|\btheta\|_2)
\]
for some $C_1<\infty$. Therefore,
\[
\|u\|_2^2\le C_2(1+\|\btheta\|_2^2)
\]
for some $C_2<\infty$. Since $\sum_{i=1}^{d_\Theta}\log|\alpha_i|$ is constant and each $\log\sigma_i(\tilde\btheta_\mathcal{B},\x)$ is uniformly bounded, it follows from the expression for $\log q_\phi(\btheta\mid\x)$ that
\[
|\ellNPE(q_\phi;\z)|=|-\log q_\phi(\btheta\mid\x)|\le C_3(1+\|\btheta\|_2^2)\le C_3(1+\|\z\|_2^2)
\]
for some $C_3<\infty$. Hence $\ellNPE(q_\phi;\cdot)\in\mathcal G_2(\mathcal Z)$.

We next control the gradient. For $u_\mathcal{B}=\tilde\btheta_\mathcal{B}$, the first derivatives are bounded and the second derivatives vanish. For $u_\mathcal{A}$, for each $i,j$,
\[
\partial_{z_j}u_{\mathcal A,i}
=(\partial_{z_j}\sigma_i)\tilde\btheta_{\mathcal A,i}
+\sigma_i\,\partial_{z_j}\tilde\btheta_{\mathcal A,i}
+\partial_{z_j}\mu_i.
\]
Hence
\[
|\partial_{z_j}u_{\mathcal A,i}|
\le C_{\sigma'}C_0(1+\|\btheta\|_2)+\sigma_{\max}\sup_{i,j}|\partial_{z_j}\tilde\btheta_{\mathcal A,i}|+C_{\mu'}
\le C_4(1+\|\btheta\|_2)
\]
for some $C_4<\infty$. Therefore,
\[
\left|\partial_{z_j}\frac12\|u\|_2^2\right|
=\left|\sum_i u_i\partial_{z_j}u_i\right|
\le C_5(1+\|\btheta\|_2^2)
\]
for some $C_5<\infty$. Moreover, since $\sigma_i\ge \sigma_{\min}$,
\[
\left|\partial_{z_j}\log\sigma_i\right|
=\left|\frac{\partial_{z_j}\sigma_i}{\sigma_i}\right|
\le \sigma_{\min}^{-1}C_{\sigma'}.
\]
Thus
\[
\|\nabla_{\z}\ellNPE(q_\phi;\z)\|\le C_6(1+\|\btheta\|_2^2)\le C_6(1+\|\z\|_2^2)
\]
for some $C_6<\infty$, and hence $\|\nabla_{\z}\ellNPE(q_\phi;\cdot)\|\in\mathcal G_2(\mathcal Z)$.

Finally, for the Hessian, for each $i,j,k$,
\begin{align*}
\partial_{z_k}\partial_{z_j}u_{\mathcal A,i}
&=(\partial_{z_k}\partial_{z_j}\sigma_i)\tilde\btheta_{\mathcal A,i}
+(\partial_{z_j}\sigma_i)(\partial_{z_k}\tilde\btheta_{\mathcal A,i})
+(\partial_{z_k}\sigma_i)(\partial_{z_j}\tilde\btheta_{\mathcal A,i})\\
&\qquad+\sigma_i\,\partial_{z_k}\partial_{z_j}\tilde\btheta_{\mathcal A,i}
+\partial_{z_k}\partial_{z_j}\mu_i.
\end{align*}
Since $\partial_{z_k}\partial_{z_j}\tilde\btheta_{\mathcal A,i}=0$, this yields
\[
|\partial_{z_k}\partial_{z_j}u_{\mathcal A,i}|\le C_7(1+\|\btheta\|_2)
\]
for some $C_7<\infty$. Consequently,
\[
\left|\partial_{z_k}\partial_{z_j}\frac12\|u\|_2^2\right|
=\left|\sum_i\Big[(\partial_{z_k}u_i)(\partial_{z_j}u_i)+u_i\,\partial_{z_k}\partial_{z_j}u_i\Big]\right|
\le C_8(1+\|\btheta\|_2^2)
\]
for some $C_8<\infty$. Also,
\[
\partial_{z_k}\partial_{z_j}\log\sigma_i
=\frac{\partial_{z_k}\partial_{z_j}\sigma_i}{\sigma_i}
-\frac{(\partial_{z_k}\sigma_i)(\partial_{z_j}\sigma_i)}{\sigma_i^2},
\]
so
\[
|\partial_{z_k}\partial_{z_j}\log\sigma_i|
\le \sigma_{\min}^{-1}C_{\sigma''}+\sigma_{\min}^{-2}C_{\sigma'}^2.
\]
Thus
\[
\|\nabla_{\z}^2\ellNPE(q_\phi;\z)\|\le C_9(1+\|\btheta\|_2^2)\le C_9(1+\|\z\|_2^2)
\]
for some $C_9<\infty$, and hence $\|\nabla_{\z}^2\ellNPE(q_\phi;\cdot)\|\in\mathcal G_2(\mathcal Z)$. Since all components in the construction are $C^2$, we also have $\ellNPE(q_\phi;\cdot)\in C^2(\mathbb R^{d_\Theta+d_\X})$.
\end{proof}

\begin{lemma}[Multiple BayesFlow coupling blocks]
\phantomsection \label{lem:multi-coupling-growth}
Fix $M\ge 1$, the number of blocks, and let
\[
\z^{(0)}=\z=(\btheta,\x).
\]
For $\ell=1,\dots,M$, write $\z^{(\ell)}=(\btheta^{(\ell)},\x)$, where $\btheta^{(\ell)}$ is obtained from $\btheta^{(\ell-1)}$ as follows:
\[
\tilde\btheta^{(\ell)}
= P^{(\ell)}\bigl(\alpha^{(\ell)}\odot \btheta^{(\ell-1)}+\beta^{(\ell)}\bigr),
\]
split
\[
\tilde\btheta^{(\ell)}
=
\bigl[\tilde\btheta^{(\ell)}_{\mathcal A};\tilde\btheta^{(\ell)}_{\mathcal B}\bigr],
\]
and define
\[
\btheta^{(\ell)}_{\mathcal A}
=
\sigma^{(\ell)}\bigl(\tilde\btheta^{(\ell)}_{\mathcal B},\x\bigr)\, \odot
\tilde\btheta^{(\ell)}_{\mathcal A}
+
\mu^{(\ell)}\bigl(\tilde\btheta^{(\ell)}_{\mathcal B},\x\bigr),
\qquad
\btheta^{(\ell)}_{\mathcal B}
=
\tilde\btheta^{(\ell)}_{\mathcal B},
\]
where each $\mu^{(\ell)}$ and $\sigma^{(\ell)}$ is defined through the BayesFlow architecture above.

Then $\z\mapsto \ellNPE(q_\phi;\z)$ belongs to $\mathcal G_2(\mathcal Z)$ and is of class $C^2$. Moreover, with
\[
\beta_M:=M+1,\qquad \alpha_M:=2M,
\]
we have
\[
\|\nabla_{\z}\ellNPE(q_\phi;\cdot)\|\in\mathcal G_{\beta_M}(\mathcal Z),
\qquad
\|\nabla_{\z}^2\ellNPE(q_\phi;\cdot)\|\in\mathcal G_{\alpha_M}(\mathcal Z).
\]
\end{lemma}

\begin{proof}
For each $\ell=1,\dots,M$, let $F^{(\ell)}:\mathcal Z\to\mathcal Z$ denote the $\ell$-th block map, defined by
\[
\z^{(\ell)}=F^{(\ell)}(\z^{(\ell-1)}),\qquad \z^{(0)}=\z=(\btheta,\x).
\]
By Lemma~\ref{lem:single-coupling-growth}, each $F^{(\ell)}$ belongs to $C^2$, and there exist constants $a_\ell,b_\ell,c_\ell<\infty$ such that for all $\z\in\mathcal Z$,
\begin{align*}
\|F^{(\ell)}(\z)\|_2&\le a_\ell(1+\|\z\|_2),\\
\|\nabla_{\z}F^{(\ell)}(\z)\|&\le b_\ell(1+\|\z\|_2),\\
\|\nabla_{\z}^2F^{(\ell)}(\z)\|&\le c_\ell(1+\|\z\|_2).
\end{align*}
Since each $P^{(\ell)}$ is a permutation matrix, it preserves Euclidean norms, so it affects only constants and not growth orders.

We first bound the iterates. Since
\[
\|\z^{(\ell)}\|_2=\|F^{(\ell)}(\z^{(\ell-1)})\|_2\le a_\ell(1+\|\z^{(\ell-1)}\|_2),
\]
an induction on $\ell$ yields
\[
\|\z^{(M)}\|_2\le C_M(1+\|\z\|_2)
\]
for some $C_M<\infty$. In particular,
\[
\|\btheta^{(M)}\|_2\le C_M(1+\|\z\|_2).
\]

Now the full log-density is given by
\[
\log q_\phi(\btheta\mid\x)
=
-\frac12\|\btheta^{(M)}\|_2^2
+\sum_{\ell=1}^M\log\left|\det\frac{\partial \btheta^{(\ell)}}{\partial \tilde\btheta^{(\ell)}}\right|
+\sum_{\ell=1}^M\log\left|\det\frac{\partial \tilde\btheta^{(\ell)}}{\partial \btheta^{(\ell-1)}}\right|
-\frac{d_\Theta}{2}\log(2\pi).
\]
The second sum is constant in $\z$, since each term equals $\sum_i\log|\alpha_i^{(\ell)}|$. The first sum equals
\[
\sum_{\ell=1}^M\sum_{i=1}^{\lfloor d_\Theta/2\rfloor}
\log \sigma_i^{(\ell)}\bigl(\tilde\btheta_{\mathcal B}^{(\ell)},\x\bigr),
\]
which is uniformly bounded because each $\sigma_i^{(\ell)}$ is bounded above and below away from zero by the single-block analysis. Hence
\[
|\ellNPE(q_\phi;\z)|\le C'_M(1+\|\btheta^{(M)}\|_2^2)\le C'_M(1+\|\z\|_2^2)
\]
for some $C'_M<\infty$. Therefore $\ellNPE(q_\phi;\cdot)\in\mathcal G_2(\mathcal Z)$.

We next control the gradient. By the chain rule,
\[
\nabla_{\z}\z^{(M)}(\z)
=
\nabla_{\z}F^{(M)}\!\bigl(\z^{(M-1)}(\z)\bigr)\cdots
\nabla_{\z}F^{(1)}\!\bigl(\z^{(0)}(\z)\bigr).
\]
Hence
\[
\|\nabla_{\z}\z^{(M)}(\z)\|
\le \prod_{\ell=1}^M b_\ell\bigl(1+\|\z^{(\ell-1)}(\z)\|_2\bigr)
\le C''_M(1+\|\z\|_2)^M
\]
for some $C''_M<\infty$. Since $\btheta^{(M)}$ is the first component of $\z^{(M)}$, this implies
\[
\|\nabla_{\z}\btheta^{(M)}(\z)\|\le C''_M(1+\|\z\|_2)^M.
\]
Therefore
\[
\nabla_{\z}\Big(\tfrac12\|\btheta^{(M)}(\z)\|_2^2\Big)
=
\nabla_{\z}\btheta^{(M)}(\z)^\top \btheta^{(M)}(\z),
\]
and so
\[
\left\|\nabla_{\z}\Big(\tfrac12\|\btheta^{(M)}(\z)\|_2^2\Big)\right\|
\le \|\nabla_{\z}\btheta^{(M)}(\z)\|\,\|\btheta^{(M)}(\z)\|_2
\le C'''_M(1+\|\z\|_2)^{M+1}
\]
for some $C'''_M<\infty$. The derivatives of the accumulated log-determinant terms are controlled similarly, since each block contributes derivatives of $\log\sigma_i^{(\ell)}$ composed with the pre-coupling variables $(\tilde\btheta_{\mathcal B}^{(\ell)},\x)$, and the derivatives of $\log\sigma_i^{(\ell)}$ are bounded by the single-block analysis. Hence
\[
\|\nabla_{\z}\ellNPE(q_\phi;\z)\|\le C'''_M(1+\|\z\|_2)^{M+1},
\]
so $\|\nabla_{\z}\ellNPE(q_\phi;\cdot)\|\in\mathcal G_{M+1}(\mathcal Z)$.

Finally, we control the Hessian. Differentiating
\[
\nabla_{\z}\Big(\tfrac12\|\btheta^{(M)}(\z)\|_2^2\Big)
=
\nabla_{\z}\btheta^{(M)}(\z)^\top \btheta^{(M)}(\z)
\]
gives
\[
\nabla_{\z}^2\Big(\tfrac12\|\btheta^{(M)}(\z)\|_2^2\Big)
=
\bigl(\nabla_{\z}\btheta^{(M)}(\z)\bigr)^\top \nabla_{\z}\btheta^{(M)}(\z)
+\sum_i \theta_i^{(M)}(\z)\,\nabla_{\z}^2\theta_i^{(M)}(\z).
\]
It therefore remains to bound $\nabla_{\z}^2\z^{(M)}(\z)$. Since $\z^{(0)}(\z)=\z$, we have
\[
\nabla_{\z}\z^{(0)}(\z)=I,\qquad \nabla_{\z}^2\z^{(0)}(\z)=0.
\]
The second-order chain rule gives the recursion
\[
\nabla_{\z}^2\z^{(\ell)}(\z)
=
\nabla_{\z}^2F^{(\ell)}\!\bigl(\z^{(\ell-1)}(\z)\bigr)
\bigl[\nabla_{\z}\z^{(\ell-1)}(\z),\nabla_{\z}\z^{(\ell-1)}(\z)\bigr]
+
\nabla_{\z}F^{(\ell)}\!\bigl(\z^{(\ell-1)}(\z)\bigr)\,
\nabla_{\z}^2\z^{(\ell-1)}(\z),
\]
Hence
\[
\|\nabla_{\z}^2\z^{(\ell)}(\z)\|
\le c_\ell(1+\|\z^{(\ell-1)}(\z)\|_2)\|\nabla_{\z}\z^{(\ell-1)}(\z)\|^2
+b_\ell(1+\|\z^{(\ell-1)}(\z)\|_2)\|\nabla_{\z}^2\z^{(\ell-1)}(\z)\|.
\]
Using the bounds
\[
\|\z^{(\ell-1)}(\z)\|_2\le C_M(1+\|\z\|_2),\qquad
\|\nabla_{\z}\z^{(\ell-1)}(\z)\|\le C''_M(1+\|\z\|_2)^{\ell-1},
\]
an induction on $\ell$ yields
\[
\|\nabla_{\z}^2\z^{(M)}(\z)\|\le C''''_M(1+\|\z\|_2)^{2M-1}
\]
for some $C''''_M<\infty$. In particular,
\[
\|\nabla_{\z}^2\btheta^{(M)}(\z)\|\le C''''_M(1+\|\z\|_2)^{2M-1}.
\]
Therefore
\[
\left\|\nabla_{\z}^2\Big(\tfrac12\|\btheta^{(M)}(\z)\|_2^2\Big)\right\|
\le \|\nabla_{\z}\btheta^{(M)}(\z)\|^2+\|\btheta^{(M)}(\z)\|_2\,\|\nabla_{\z}^2\btheta^{(M)}(\z)\|
\le C'''''_M(1+\|\z\|_2)^{2M}
\]
for some $C'''''_M<\infty$. The second derivatives of the accumulated log-determinant terms are controlled in the same way and are of no larger polynomial order. Hence
\[
\|\nabla_{\z}^2\ellNPE(q_\phi;\z)\|\le C'''''_M(1+\|\z\|_2)^{2M},
\]
so $\|\nabla_{\z}^2\ellNPE(q_\phi;\cdot)\|\in\mathcal G_{2M}(\mathcal Z)$. Since each $F^{(\ell)}$ is $C^2$, the composition is also $C^2$, and therefore $\ellNPE(q_\phi;\cdot)\in C^2(\mathcal Z)$.
\end{proof}

\Cref{lem:flows-assumptions} then follows directly from Lemmas \ref{lem:MAF} and \ref{lem:multi-coupling-growth}.

\section{Experimental details}\label{appendix:implementation-details}

This section is organised as follows. Appendix~\ref{appendix:hyperparam-selection} outlines our hyperparameter-selection procedure using the KL-based miscalibration metric, Appendix~\ref{appendix:baseline-details} describes the baseline methods, Appendix~\ref{app:arch-hyperparam} contains the implementation details, Appendix~\ref{app:simulator} contains the details of the simulators used in \Cref{sec:experiments}, and  Appendix~\ref{app:cost_comparison} discusses the computational cost.

\subsection{Hyperparameter selection method for DRO-NPE}
\label{appendix:hyperparam-selection}

\paragraph{Estimating KL-based miscalibration.}
We show that KL-based miscalibration can be estimated by
density-ratio estimation, but first introduce the notation needed
for this. Recall that $P_{\x}$ is the marginal distribution
obtained from the joint $P\equiv P_{\btheta,\x}$, and that
$F^Q_{S\mid\x}$ is the CDF of $Q_{S\mid\x}$, which is the
distribution of $S(\btheta;\x)$ whenever
$\btheta\sim Q_{\btheta\mid\x}$ for some fixed $\x$.
For each $\x$, we define
$\widetilde U\mid\x\sim(F^Q_{S\mid\x})_{\#}P_{S\mid\x}$.
Let $P_{\widetilde U,\x}$ denote the resulting joint distribution
of $(\widetilde U,\x)$. We also define an independent reference
pair $(U,\x)$, where $U\sim\mathcal U[0,1]$ is independent of
$\x\sim P_{\x}$, so that
$P_{U,\x}=P_U\otimes P_{\x}$.
We write the respective densities as
$p_{U,\x}(u,\x)=p_U(u)p_{\x}(\x)=p_{\x}(\x)$
since $p_U(u)=1$ on $[0,1]$, and
$p_{\widetilde U,\x}(u,\x)
=p_{\widetilde U\mid\x}(u\mid\x)p_{\x}(\x)$.

Under the continuity conditions for the uniform reference,
the KL-based miscalibration is
\[
\kl_{\mathrm{cal}}^S(q_\phi;P)
=
\mathbb E_{\x\sim P_{\x}}
\left[
-h\left((F^Q_{S\mid\x})_{\#}P_{S\mid\x}\right)
\right]
=
\int p_{\x}(\x)
\int p_{\widetilde U\mid\x}(u\mid\x)
\log p_{\widetilde U\mid\x}(u\mid\x)
\,du\,d\x,
\]
where $h(\mu)=-\int p(u)\log p(u)\,du$ denotes the differential
entropy. Since $p_U(u)=1$ on $[0,1]$, we can equivalently write
\begin{align*}
&\int p_{\x}(\x)
\int p_{\widetilde U\mid\x}(u\mid\x)
\log p_{\widetilde U\mid\x}(u\mid\x)\,du\,d\x
\\
&\qquad=
\int\!\!\int
p_{\widetilde U\mid\x}(u\mid\x)p_{\x}(\x)
\log
\frac{
p_{\widetilde U\mid\x}(u\mid\x)p_{\x}(\x)
}{
p_U(u)p_{\x}(\x)
}
\,du\,d\x
\\
&\qquad=
\KL\left(P_{\widetilde U,\x}\,\middle\|\,P_{U,\x}\right).
\end{align*}
Therefore, $\kl_{\mathrm{cal}}^S(q_\phi;P)$ can be estimated by
density-ratio estimation between samples from the dependent joint
$P_{\widetilde U,\x}$ and the independent reference
$P_U\otimes P_{\x}$.

To this end, we formulate a binary classification problem:
samples from class $y=1$ are drawn from
$(\widetilde U,\x)\sim P_{\widetilde U,\x}$, while samples from
class $y=0$ are drawn from
$(U,\x)\sim P_{U,\x}=P_U\otimes P_{\x}$.
With equal class priors, the Bayes-optimal classifier
$D^\star:[0,1]\times\X\to[0,1]$ is given by
\[
D^\star(u,\x)
=
\frac{p_{\widetilde U,\x}(u,\x)}
{p_{\widetilde U,\x}(u,\x)+p_{U,\x}(u,\x)}.
\]
Rearranging gives
\[
\frac{
p_{\widetilde U\mid\x}(u\mid\x)p_{\x}(\x)
}{
p_U(u)p_{\x}(\x)
}
=
\frac{p_{\widetilde U,\x}(u,\x)}{p_{U,\x}(u,\x)}
=
\frac{D^\star(u,\x)}{1-D^\star(u,\x)}.
\]

Thus, the desired ratio can be recovered directly from the
classifier via its odds. In practice, we train a discriminator
$D_\psi$ by minimising the binary cross-entropy loss, yielding
the approximation
\[
\frac{p_{\widetilde U,\x}(u,\x)}{p_{U,\x}(u,\x)}
\approx
\frac{D_{\widehat\psi}(u,\x)}
{1-D_{\widehat\psi}(u,\x)},
\]
where $\widehat\psi$ are the trained parameters. Finally, given $D_{\hat \psi}$, we can obtain Monte Carlo estimation of $\kl_\mathrm{cal}^S$ with sample size $N$ as
\[
\widehat{\kl}_{\mathrm{cal}}^S(q_\phi;P)
=
\frac{1}{N}\sum_{i=1}^N
\log
\frac{D_{\hat{\psi}}(\tilde U_i,\x_i)}
{1-D_{\hat{\psi}}(\tilde U_i,\x_i)}.
\]

In practice, we summarise $\x$ by the total posterior
variance
$\gamma_\phi(\x)=
\sum_{j=1}^{d_\Theta}
\operatorname{Var}_{Q_{\btheta\mid\x}}(\theta_j\mid\x)$
and fit the logistic classifier
\[
D_\psi(u,\gamma)  = [1+\exp(-(\psi_0+\psi_1u+\psi_2\gamma))]^{-1}
\]
We use this low-dimensional summary and a simple classifier to limit estimation complexity. This is particularly important in the small-sample regime, where $D_\psi$ and $\widehat \kl_\mathrm{cal}^S$ must be estimated using an even smaller subset of the available data reserved for validation; a more complex estimator would overfit, leading to unreliable estimates of $\widehat \kl_\mathrm{cal}^S$ and, consequently, a poor choice of $\varepsilon$.

The resulting score is a practical approximation:
the compression of $\x$, the restricted classifier, and finite posterior
sampling can limit its sensitivity to miscalibration.
Alternatively, when sufficient samples are available, more flexible models such as a neural
network could be used to learn $D_\psi(u,\x)$ directly from the full $\x$.

We use $M=1000$ posterior samples per evaluation pair and
$K=5$ classifier folds. Cross-fitting separates classifier
fitting from log-odds evaluation, as detailed in
Algorithm~\ref{alg:kl_cal}. Note that we only

\begin{algorithm}[t]
\caption{Cross-fitted estimation of KL-based miscalibration}
\label{alg:kl_cal}
\begingroup
\begin{algorithmic}[1]
\Require Evaluation pairs
$\{(\btheta_i,\x_i)\}_{i=1}^{N}$,
posterior density estimator $q_\phi$,
summary $S$,
number of posterior samples $\M$,
number of folds $K$
\Ensure Estimate $\widehat{\kl}_{\mathrm{cal}}^S$
of KL-based miscalibration

\Statex \textbf{(a) Construct rank samples}
\For{$i=1,\ldots,N$}
    \State Draw
    $\tilde\btheta_i^{(1)},\ldots,\tilde \btheta_i^{(\M)}
    \overset{\mathrm{iid}}{\sim}
    q_\phi(\cdot\mid\x_i)$
    \State Compute
    \[
    \widetilde u_i
    =
    \frac{1}{\M}\sum_{m=1}^{\M}
    \mathbf{1}\bigl\{
    S(\tilde \btheta_i^{(m)};\x_i)
    \leq S(\btheta_i;\x_i)
   \bigl\}
    \]
\EndFor

\Statex \textbf{(b) Summarise observations by posterior variance}
\For{$i=1,\ldots,N$}
    \State Using the posterior samples, compute
    \[
    \widehat\gamma_i
    =
    \sum_{j=1}^{d_\Theta}
    \widehat{\operatorname{Var}}
    \left[\tilde \theta_j\mid\x_i\right]
    \]
\EndFor

\Statex \textbf{(c) Split evaluation pairs into $K$ folds}
\State Partition $\{1,\ldots,N\}$ into
$\mathcal I_1,\ldots,\mathcal I_K$
\State Draw independent $U_i\sim\mathcal U[0,1]$
for $i=1,\ldots,N$

\For{$k=1,\ldots,K$}
    \Statex \textbf{(d) Fit the classifier on the other $K-1$ folds}
    \State Fit logistic regression by binary cross-entropy
    to distinguish
    \[
    \{(\widetilde U_i,\widehat\gamma_i,1):
       i\notin\mathcal I_k\}
    \quad\text{from}\quad
    \{(U_i,\widehat\gamma_i,0):
       i\notin\mathcal I_k\},
    \]
    \Statex \hspace{\algorithmicindent} using
    $D_\psi(u,\gamma)
    =[1+\exp(-(\psi_0+\psi_1u+\psi_2\gamma))]^{-1}$

    \Statex \textbf{(e) Evaluate on held-out observed-rank pairs}
    \For{$i\in\mathcal I_k$}
        \State Compute
        \[
        \widehat r_i
        =
        \log
        \frac{
        D_{\widehat\psi^{(-k)}}
        (\widetilde U_i,\widehat\gamma_i)
        }{
        1-D_{\widehat\psi^{(-k)}}
        (\widetilde U_i,\widehat\gamma_i)
        }
        \]
    \EndFor
\EndFor

\State \Return
$\displaystyle
\widehat{\kl}_{\mathrm{cal}}^S
=
\frac{1}{N}\sum_{i=1}^{N}\widehat r_i$
\end{algorithmic}
\endgroup
\end{algorithm}

\paragraph{Selecting $\varepsilon$.}
Given an estimate of $\kl_{\mathrm{cal}}^S$, we choose
$\varepsilon$ using a validation set.
We split $n$ samples from $P$ into training and validation sets, reserving $10\%$ for validation, and minimise the validation estimate of $\kl_{\mathrm{cal}}^S$ over $\varepsilon\in[\varepsilon_{\min},\varepsilon_{\max}]$.
For each queried value of $\varepsilon$, we train $q_\phi$
on the training set and evaluate $\kl_{\mathrm{cal}}^S$
on the validation set.
The selected value $\varepsilon^\star$ is then used to
retrain $q_\phi$ on the full dataset.
In our experiments, we perform Bayesian optimisation
\citep{garnett2023bayesian} over $\log\varepsilon$ on the
interval $[\log 0.001,\log 10]$, using a budget of 10 iterations.
Optimising over $\log\varepsilon$ explores each order of
magnitude equally, which is appropriate because $\varepsilon$
controls the strength of regularisation.

\subsection{Alternative approaches to select $\varepsilon$}\label{appendix:alt-approaches-epsilon}

Selecting $\varepsilon$ in DRO typically involves a trade-off between theoretical guarantees, empirical performance, and computational cost.
In our experiments, we prioritise empirical performance by selecting $\varepsilon$ via hold-out validation.
To illustrate this trade-off, we describe two alternative approaches: an explicit theoretical radius, and a bootstrap estimate of $W_2(P_n, P)$.

\paragraph{Explicit theoretical radius.}
The radius $\varepsilon(n, \delta)$ in \Cref{theorem:W-DRO-upper-bound-chain} relies on the concentration bounds of \citet{fournier2015rate}, whose constants are not explicit.
\citet{fournier2023convergence} instead provides explicit, computable constants, at the cost of a looser bound.
Specifically, for $P$ supported in a ball of radius $R$ with respect to the supremum norm $\|\cdot\|_\infty$, \citet[Section~2.5, Table~2]{fournier2023convergence} gives
\[
\sqrt{\mathbb{E}\big[W_2^{(\infty)}(P_n, P)^2\big]} \le B(2, n, d_{\mathcal{Z}}),
\]
where $W_2^{(\infty)}$ denotes the 2-Wasserstein distance with ground cost induced by $\|\cdot\|_\infty$, and the expectation is over $z_{1:n} \overset{\text{i.i.d.}}{\sim} P$.
Since $\|z\|_2 \le \sqrt{d_{\mathcal{Z}}}\, \|z\|_\infty$ for all $z \in \mathcal{Z}$, the Euclidean $W_2$ used throughout satisfies
\[
\mathbb{E}\big[W_2(P_n, P)^2\big] \le d_{\mathcal{Z}}\, \mathbb{E}\big[W_2^{(\infty)}(P_n, P)^2\big] \le d_{\mathcal{Z}}\, B(2, n, d_{\mathcal{Z}})^2.
\]
This bound holds in expectation, whereas \Cref{theorem:W-DRO-upper-bound-chain} requires $P \in \mathcal{A}_2(P_n; \varepsilon)$ with probability at least $1 - \delta$.
Applying Markov's inequality to $W_2(P_n, P)^2$ gives
\[
\mathbb{P}\big(W_2(P_n, P) > \varepsilon\big) = \mathbb{P}\big(W_2(P_n, P)^2 > \varepsilon^2\big) \le \frac{d_{\mathcal{Z}}\, B(2, n, d_{\mathcal{Z}})^2}{\varepsilon^2},
\]
so that choosing
\[
\varepsilon_n := \sqrt{\frac{d_{\mathcal{Z}}}{\delta}}\; B(2, n, d_{\mathcal{Z}})
\]
ensures $P \in \mathcal{A}_2(P_n; \varepsilon_n)$ with probability at least $1 - \delta$, and hence that the upper bound in \Cref{theorem:W-DRO-upper-bound-chain} holds.

As an example, consider Two Moons, for which $d_{\mathcal{Z}} = 4$ and \citet[Table~2]{fournier2023convergence} gives
\[
B(2, n, 4) = 2R\, \frac{\sqrt{0.73 \log n + 1.26}}{n^{1/4}}.
\]
Two Moons is not compactly supported, owing to the Gaussian noise term $r \sim \mathcal{N}(0.1, 0.01^2)$, but its tails decay exponentially fast.
We take $R = 1.12$, the radius of the smallest supremum-norm ball containing at least $95\%$ of the mass of $P$; a fully rigorous treatment would instead use the bounds of \citet{fournier2023convergence} for non-compactly supported distributions.
With $\delta = 0.05$, this gives $\varepsilon_n \approx 8.90$ for $n = 1024$ and $\varepsilon_n \approx 5.11$ for $n = 16384$.

\paragraph{Bootstrapping.}
The second approach estimates $W_2(P_n, P)$ directly by bootstrapping the simulated data.
This choice satisfies the theory up to bootstrapping error, which can be made negligible with a large number of bootstrap samples.
It is also considerably cheaper than hold-out validation, since it does not require retraining the normalising flow.

\paragraph{Comparison.}
For Two Moons with $n = 1024$ and $\delta = 0.05$, the explicit theoretical radius gives $\varepsilon_n \approx 8.90$ at virtually zero cost, whereas bootstrapping gives $\varepsilon_n \approx 0.129$, averaged over five random seeds with $1000$ bootstrap samples, and takes $34$ seconds on average.
Both values are substantially larger than those obtained via hold-out validation, which are of order $10^{-3}$.
However, hold-out validation is considerably more expensive, requiring over $1000$ seconds in this example.
Overall, hold-out validation yields substantially smaller, less conservative choices of $\varepsilon$, but does not guarantee that the theoretical conditions are satisfied.
Nevertheless, it is commonly used in DRO, and our experiments demonstrate its effectiveness in the context of NPE.

\subsection{Description of baseline methods}\label{appendix:baseline-details}

\paragraph{Balanced NPE \citep{delaunoy2023balancing}.}
The method extends the Balanced NRE method proposed by \citet{delaunoy2022towards} to NPE. NRE relies on the  ratio $r(\x \mid \btheta) := \frac{p(\btheta, \x)}{p(\btheta)p(\x)}$ and the Bayes-optimal classifier $d^\star: \Theta \times \X \to [0,1]$, which distinguishes samples from the joint distribution $P_{\btheta, \x}$ and the product of marginals $\Pi \otimes P_{\x}$. The optimal classifier is defined as $d^\star(\btheta, \x) := \frac{p(\btheta, \x)}{p(\btheta, \x) + p(\btheta)p(\x)}$, and relates to $r(\x \mid \btheta)$ via $r(\x \mid \btheta) = \frac{d^\star(\btheta, \x)}{1 - d^\star(\btheta, \x)}$.

In practice, the optimal classifier is not available and is estimated from finite samples, denoted $\hat d$. \citet{delaunoy2023balancing} suggest that overconfident posteriors obtained from NRE originate from overconfidence in $\hat d$. To mitigate this, they introduce a regularisation term that encourages $\hat d$ to be more conservative than $d^\star$. Specifically, they define a balanced classifier $d_b$ satisfying $\mathbb{E}_{\btheta, \x \sim P_{\btheta, \x}} [d_b(\btheta, \x)] + \mathbb{E}_{\btheta, \x \sim \Pi \otimes P_{\x}} [d_b(\btheta, \x)] = 1$. They show that any balanced classifier satisfies $\mathbb{E}_{\btheta, \x \sim P_{\btheta, \x}}\left[ \frac{d^\star(\btheta, \x)}{d_b(\btheta, \x)} \right] \geq 1$ and $\mathbb{E}_{\btheta, \x \sim \Pi \otimes P_{\x}}\left[ \frac{1 - d^\star(\btheta, \x)}{1 - d_b(\btheta, \x)} \right] \geq 1$, implying that $d_b$ is less confident than $d^\star$ in expectation.

To impose this property in the NPE setting, \citet{delaunoy2023balancing} add the balancing regularisation term $\mathbb{E}_{\btheta, \x \sim P_{\btheta, \x}}[\hat d (\btheta, \x)] + \mathbb{E}_{\btheta, \x \sim \Pi \otimes P_{\x}}[\hat d (\btheta, \x)] -1$ on top of the standard NPE loss. This term is estimated using the approximate posterior density by noting that $r(\x \mid \btheta) = \frac{p(\btheta \mid \x)}{p(\btheta)} \approx \frac{q_\phi(\btheta \mid \x)}{p(\btheta)}$ and $\hat d(\btheta, x) \approx r(\x \mid \btheta) / (1 + r(\x \mid \btheta))$. The resulting regularised objective is
\begin{talign*}
\mathcal{L}_{\mathrm{Bal}}(q_\phi; P_{\btheta, \x}, \Pi \otimes P_{\x})
:=
\mathcal{L}_\mathrm{NPE} (q; P_{\btheta, \x})
+ \lambda \left[
\mathbb{E}_{\btheta, \x \sim P_{\btheta, \x}}\left[
\frac{\frac{q_\phi(\btheta \mid \x)}{p(\btheta)}}{1 + \frac{q_\phi(\btheta \mid \x)}{p(\btheta)}}
\right]
+
\mathbb{E}_{\btheta, \x \sim \Pi \otimes  P_{\x}}\left[
\frac{\frac{q_\phi(\btheta \mid \x)}{p(\btheta)}}{1 + \frac{q_\phi(\btheta \mid \x)}{p(\btheta)}}
\right]
-1
\right],
\end{talign*}
where $\lambda >0$ is a regularisation parameter. In their implementation, samples from $\Pi \otimes P_{\x}$ are approximated using joint samples $\{(\boldsymbol{\theta}_i, \x_i)\}_{i=1}^n \sim P_{\boldsymbol{\theta}, \x}$. While these are used for estimating terms involving the joint, samples from the marginals are constructed by cyclically shifting $\boldsymbol{\theta}$, yielding mismatched pairs $\{(\boldsymbol{\theta}_n, \x_1), (\boldsymbol{\theta}_1, \x_2), \dots, (\boldsymbol{\theta}_{n-1}, \x_n)\}$.

\paragraph{Coverage-regularised NPE \citep{falkiewicz2023calibrating}.} This method penalises deviations of the SBC ranks, computed as in \Cref{def:cal}, from the standard uniform distribution, using $S(\btheta; \x) = q_\phi(\btheta \mid \x)$. Given $(\btheta, \x) \sim P_{\btheta,\x}$, define
\begin{align*}
u = u(q; \btheta, \x) = \frac{1}{\M} \sum_{m=1}^\M \mathbf{1}\Bigl\{ q_\phi\left(\tilde\btheta^{(m)} \mid \x \right) < q_\phi \left(\btheta \mid \x \right) \Bigr\} \quad \text{where } \quad \tilde\btheta^{(1)}, \dots, \tilde\btheta^{(\M)} \sim Q_{\btheta \mid \x}.
\end{align*} 
Let $u_{(1)}, \dots, u_{(n)}$ denote the sorted sequence of $\{u_i\}_{i=1}^n$ in ascending order with $u_i = u(q_\phi; \btheta_i, \x_i)$. The empirical CDF evaluated at these points is given by $\hat{F}(u_{(i)}) = \frac{i}{n}$ for $ i \in \{1, \dots, n\}$. This empirical CDF is compared to the CDF of the standard uniform distribution, which satisfies $F(u_{(i)}) = u_{(i)}$. The resulting calibration error is defined as the $L^2$ distance between $F$ and $\hat{F}$, and in practice this is estimated through samples $u_1,\ldots,u_n$:  $\frac{1}{n} \sum_{i=1}^n \left(u_{(i)} -  \frac{i}{n} \right)^2$. 
To encourage conservative approximate posteriors, the authors propose penalising only overconfident deviations. This is achieved by introducing a one-sided penalty that considers only cases where the uniform CDF dominates the empirical CDF, namely
$\frac{1}{n} \sum_{i=1}^n ( \max\{ u_{(i)} - \frac{i}{n}, 0 \} )^2$.

To allow backpropagation with respect to $\phi$, $u$ is estimated via self-normalised importance sampling with the prior as a proposal distribution such that 
$$u  =  \frac{\sum_{m=1}^\M \frac{q_\phi(\tilde\btheta^{(m)} \mid \x)}{p(\tilde\btheta^{(m)})} \mathbf{1}\bigl\{ q_\phi(\tilde\btheta^{(m)} \mid \x) < q_\phi(\btheta \mid \x) \bigr\}}{\sum_{m=1}^\M \frac{q_\phi(\tilde\btheta^{(m)} \mid \x)}{p(\tilde\btheta^{(m)})}},$$
where $\tilde \btheta^{(1)}, \dots, \tilde \btheta^{(\M)} \sim \Pi$. Moreover, differentiable relaxations of both the indicator function and the sorting operation are employed. 
The resulting regularised objective is 
\begin{align*}
 \mathcal{L}_{\mathrm{Cal}}(q_\phi; P_n) :=
\mathcal{L}_\mathrm{NPE} (q_\phi; P_n) 
+ \lambda \frac{1}{n} \sum_{i=1}^n \left(\max\left\{u_{(i)}\left(q_\phi; \btheta_{(i)}, \x_{(i)}\right) - \frac{i}{n}, 0 \right\}\right)^2.
\end{align*}

\subsection{Implementation details}\label{app:arch-hyperparam}

\textbf{Conditional density estimator.} \hspace{1ex}
Across all experiments, we employ a conditional normalising flow model based on coupling layers, as implemented in the \textit{BayesFlow} package \citep{kuhmichel2026bayesflow}, for both DRO-NPE and the baseline methods. We use six coupling layers and adopt affine transformations as the coupling functions, following the default configuration. For the multi-layer perceptron used to predict the parameters of the affine transformations, we retain the default architecture except for the choice of activation function: we replace the standard activation with the tanh, which ensures that the theoretical assumptions underlying our method are satisfied.

\textbf{Training procedure.} \hspace{1ex}
The model is trained for 1000 epochs with a batch size of 64. We train the network using the AdamW optimiser \citep{loshchilov2017decoupled} with a learning rate of $5 \times 10^{-4}$. Unless otherwise stated, we do not use a validation set for early stopping or model selection, in order to isolate the effect of the training objective.
Consequently, the results of our experiments may not exactly reproduce those reported in earlier works where this effect is not explicitly disentangled.
For CR-NPE, we observed numerical instabilities during prolonged training. Therefore, for this method, we train the network until instability occurs (i.e., loss value becoming \textit{nan}) or for 1000 epochs (whichever occurs first). In case of instability, we set the final model to be the one in the last stable training epoch.
Note that while this instability is also present in the original implementation, it is more pronounced in our setting, potentially due to the use of a different conditional density estimator from the original paper. Nevertheless, we retain this choice to ensure a fair and meaningful comparison across methods. For $\varepsilon$-selection in DRO-NPE, we use 10\% of the data as validation set.

\textbf{Hyperparameters for other methods.} \hspace{1ex} Following the recommendations in the respective works, we set $\lambda = 5$ and $M = 16$ in CR-NPE and $\lambda = 100$ in Bal-NPE. 

\subsection{Description of simulators}\label{app:simulator}

\paragraph{SLCP.}
The simulator is deliberately constructed to exhibit a simple likelihood while inducing a complex posterior, making it well-suited for benchmarking. The parameter $\btheta \in \mathbb{R}^5$ governs a two-dimensional Gaussian distribution. We draw four independent samples from this distribution and concatenate them, resulting in an observation $\x \in \mathbb{R}^8$ \citep{Papamakarios2019, Lueckmann2021, Hermans2022}.
\[
\begin{aligned}
\textbf{Prior:} \quad & \btheta \sim \mathcal{U}([-3, 3]^5) \\
\textbf{Model:} \quad & \x_i \sim \mathcal{N}\left(
\begin{bmatrix}
\theta_1 \\
\theta_2
\end{bmatrix},
\begin{bmatrix}
\theta_3^4 & \tanh(\theta_5)\,\theta_3^2 \theta_4^2 \\
\tanh(\theta_5)\,\theta_3^2 \theta_4^2 & \theta_4^4
\end{bmatrix}
\right), \quad i = 1, \dots, 4 \\
& \x = [\x_1^\top, \dots, \x_4^\top]^\top
\end{aligned}
\]

\paragraph{Two moons.}
The two moons simulator is a commonly used toy benchmark in SBI, designed to exhibit a posterior distribution with multimodality. Both the parameter and observation spaces are two-dimensional, with $\btheta \in \mathbb{R}^2$ and $\x \in \mathbb{R}^2$ \citep{Lueckmann2021, Greenberg2019}.

\[
\begin{aligned}
\textbf{Prior:} \quad &\btheta \sim  \mathcal{U}([-1, 1]^2),\\
\textbf{Model:} \quad & \x =
\begin{bmatrix}
-|\theta_1 + \theta_2| /\sqrt{2} \\
(-\theta_1 + \theta_2) / 2
\end{bmatrix}
+
\begin{bmatrix}
r\cos \alpha + 0.25 \\
r\sin \alpha
\end{bmatrix}, \\
& \text{ where } \alpha \sim \mathcal{U}(-\pi/2, \pi/2) \text{ and } r \sim \mathcal{N}(0.1, 0.01^2).
\end{aligned}
\]

\paragraph{Lotka--Volterra.}
The Lotka--Volterra model describes the nonlinear dynamics of interacting prey and predator populations. The objective is to infer the parameter vector $\btheta \in \mathbb{R}^4$, which controls the species interaction dynamics, from noisy observations of both populations over time. Each population is summarised at 10 evenly spaced time points, yielding an observation $\x \in \mathbb{R}^{20}$ \citep{Lueckmann2021, lotka1920analytical}.

\[
\begin{aligned}
\textbf{Prior:} \quad & \theta_1 \sim \mathrm{LogNormal}(-0.125, 0.5), \quad \theta_2 \sim \mathrm{LogNormal}(-3, 0.5) \\
\quad & \theta_3 \sim \mathrm{LogNormal}(-0.125, 0.5), \quad \theta_4 \sim \mathrm{LogNormal}(-3, 0.5) \\
\textbf{Model:} \quad & x_{1, i} \sim  \mathrm{LogNormal}(\log(X(t_i)), 0.1), \quad i = 1,\dots 10 \\ 
&x_{2, i} \sim \mathrm{LogNormal}(\log (Y(t_i)), 0.1), \quad i = 1,\dots 10 \\
& \text{where } (X, Y) \text{ evolve according to} \\
& \frac{dX}{dt} = \theta_1 X - \theta_2 X Y, \quad \frac{dY}{dt} = -\theta_3 Y + \theta_4 X Y, \quad  \\
& \text{with time horizon } T = 20\text{, and initial condition } (X(0), Y(0)) = (30, 1). \\
& \text{Observations are collected at 10 evenly spaced time points } t_i \in [0, T].
\end{aligned}
\]

\paragraph{Inverse Kinematics.}
The inverse kinematics benchmark models a two-dimensional, multi-jointed robotic arm whose forward mapping sends joint configurations to the arm endpoint position. The parameter $\btheta \in \mathbb{R}^4$ specifies the configuration of the arm, consisting of an initial height $\theta_1$ and three subsequent joint angles $\theta_2, \theta_3, \theta_4$. The resulting observation $\x \in \mathbb{R}^2$ corresponds to the two-dimensional coordinates of the arm \citep{kruse2021benchmarking, schmitt2024consistency}.

\[
\begin{aligned}
\textbf{Prior:} \quad & \btheta \sim \mathcal{N}\!\left(\mathbf{0}, \operatorname{diag}\!\left(\tfrac{1}{16}, \tfrac{1}{4}, \tfrac{1}{4}, \tfrac{1}{4}\right)\right)\\
\textbf{Model:} \quad & x_1 = l_1 \sin(\theta_2) + l_2 \sin(\theta_2 + \theta_3) + l_3 \sin(\theta_2 + \theta_3 + \theta_4) + \theta_1 \\
&x_2 = l_1\cos(\theta_2) + l_2 \cos(\theta_2 + \theta_3) + l_3 \cos(\theta_2 + \theta_3 + \theta_4)   \\
& \text{ with segment lengths: } l_1 = l_2 =  \frac{1}{2}, l_3 = 1.
\end{aligned}
\]

\paragraph{Cosmology.}

We use the high-fidelity hydrodynamic simulations from the CAMELS suite ~\citep{CAMELS_presentation, CAMELS_DR1}, a benchmark dataset for machine learning in astrophysics. These simulations model 25 Mpc$/h$ cosmological volumes and include complex astrophysical processes such as feedback. The inference task is to estimate the cosmological parameters $\theta = (\Omega_m, \sigma_8)$, where $\Omega_m$ is the matter density and $\sigma_8$ 
the amplitude of fluctuations, from mock power spectrum measurements $P(k)$.

\subsection{Computational cost comparison}\label{app:cost_comparison}

We now compare the computational overhead of DRO-NPE with that of the baselines. Let $B$ be the batch
size and $C_\phi$ the cost of one evaluation of $-\log q_\phi(\btheta\mid\x)$, which is determined
by the architecture of $q_\phi$, including its input dimension $d_{\mathcal Z}$. We use coupling
flows in our experiments. For a coupling flow with $M\ge1$ coupling blocks, whose subnetworks are
multi-layer perceptrons (MLPs) with $L\ge1$ hidden layers of $H\ge1$ hidden units each (i.e.\
$\h^{(\ell)}\in\mathbb R^H$ for $1\le\ell\le L$, with hidden layers defined analogously to
Appendix \ref{app:theory-wasserstein-DRO}), the cost $C_\phi$ of one evaluation is of order
$MH(d_{\mathcal Z}+(L-1)H)$. We first bound the cost of evaluating each training objective on a
mini-batch.

Standard NPE evaluates the density once per sample, so the NPE objective costs
$\mathcal O(BC_\phi)$. Bal-NPE additionally evaluates the density on the cyclically shifted pairs,
which changes the cost only by a constant factor, and therefore also costs $\mathcal O(BC_\phi)$.
CR-NPE evaluates the density at $\M$ importance samples per observation, so its cost
grows as $\mathcal O(B\mathcal{M}C_\phi)$.

DRO-NPE additionally requires the input gradients $\nabla_{\z}\ell_{\mathrm{NPE}}(q_\phi;\z_i)$ and
the regulariser $\Omega(\phi;P_B)$.
Reverse-mode automatic differentiation computes the gradient of a scalar at a constant multiple of
the cost of evaluating it, without a multiplicative dependence on the number of variables
\citep[\S3.2]{baydin2018automatic}. Hence, the input gradients for $i=1,\dots,B$ cost
$\mathcal O(BC_\phi)$, without an additional factor of $d_{\mathcal Z}$.
Computing $\Omega(\phi;P_B)$ from them requires the squared norm of each
$d_{\mathcal Z}$-dimensional input gradient and an average over the $B$ samples, which costs
$\mathcal O(Bd_{\Z})$. This is also $\mathcal O(BC_\phi)$ since $C_\phi$ is of order $MH(d_{\mathcal Z}+(L-1)H)$, which is at least $d_{\mathcal Z}$ because $M,H\ge1$.
Evaluating the DRO-NPE objective therefore also costs $\mathcal O(BC_\phi)$, as for NPE, but with a
larger constant.

Finally, each objective is a scalar, so by the same reverse-mode argument, its gradient with
respect to $\phi$ costs a constant multiple of its evaluation. A full training step therefore has
the same order of cost as evaluating the objective: $\mathcal O(BC_\phi)$ for NPE, Bal-NPE and
DRO-NPE, and $\mathcal O(B\M C_\phi)$ for CR-NPE.
Constant factors and actual run times depend on the hardware and implementation, so for DRO-NPE we
report measured wall-clock times (\Cref{fig:empirical-anal}(d)): DRO-NPE takes about $2\times$ the
time of NPE on CPU and $1.25\times$ on GPU.

\paragraph{Cost of hyperparameter-selection.}
Our full DRO-NPE procedure performs 10 Bayesian
optimisation (BO) iterations, consisting of (i) training of the posterior estimator $q_\phi$, 
(ii) evaluation of KL-based miscalibration $\widehat{\kl}_{\mathrm{cal}}^q$ on validation set, and (iii) selection of next $\varepsilon$ using BO. These are followed by retraining of $q_\phi$ on
the full dataset using the selected $\varepsilon$.
Table~\ref{tab:hyperparameter-cost} compares this total
cost with the training times of standard NPE and
DRO-NPE with fixed $\varepsilon$. All the experiments are conducted on GPU. 
Including hyperparameter selection increases the total
computational cost to approximately $13$--$16$ times that
of DRO-NPE training with fixed $\varepsilon$. 

\begin{table}[t]
    \centering
    \caption{
    Computational cost on Lotka--Volterra task in seconds.
    DRO-NPE (total) includes 10 BO
    iterations and final retraining.
    Timings are means $\pm$ standard deviations over
    five random seeds.
    }
    \begin{tabular}{rccc}
        \toprule
        $n$ & NPE & DRO-NPE (fixed $\varepsilon$)
            & DRO-NPE (total) \\
        \midrule
        $1024$  & $73\pm1$   & $93\pm3$   & $1521\pm21$ \\
        $2048$  & $128\pm3$  & $159\pm7$  & $2465\pm218$ \\
        $4096$  & $233\pm4$  & $286\pm15$ & $4249\pm742$ \\
        $8192$  & $437\pm11$ & $535\pm21$ & $7221\pm350$ \\
        $16384$ & $826\pm30$ & $976\pm57$ & $13138\pm293$ \\
        \bottomrule
    \end{tabular}
    \label{tab:hyperparameter-cost}
\end{table}

We also examine the cost of evaluating $\widehat{\kl}_{\mathrm{cal}}^q$ within each BO iteration.
On Lotka--Volterra task, a single evaluation of $\widehat{\kl}_{\mathrm{cal}}^q$ takes an average of
$2.64$ seconds for $N=103$ validation pairs and
$5.23$ seconds for $N=205$, over five seeds. These validation sizes correspond to simulation budgets
$n=1024$ and $n=2048$, respectively. The evaluation cost is dominated by rank computation,
which requires repeated density evaluations. Relative to the DRO-NPE training times reported
in Table~\ref{tab:hyperparameter-cost}, evaluating
$\widehat{\kl}_{\mathrm{cal}}^q$ incurs only a small additional cost.

\section{Additional results}\label{app:additional-results}

In this section, we present additional results related to experiments shown in \Cref{sec:experiments}. 

\paragraph{Additional details and results for \Cref{fig:intro-posterior-comparison}.} The posterior contour plots of Lotka--Volterra parameters in \Cref{fig:intro-posterior-comparison} were obtained using $n = 1024$ simulations for both NPE and DRO-NPE.
To produce a reference posterior, we train NPE with $n = 131072$ simulations. The contour plots are computed using kernel density estimation from $5000$ posterior samples. In \Cref{fig:posterior-comparison-app}, we show the results for all four parameters using paired two-dimensional marginals for $(\theta_1, \theta_2)$ and $(\theta_3, \theta_4)$. We observe similar results for $(\theta_3, \theta_4)$ as we did for $(\theta_1, \theta_2)$: the NPE posterior is biased, while DRO-NPE covers the true parameter and yields a conservative posterior.
\begin{figure}
    \centering
    \includegraphics[width=0.6\linewidth]{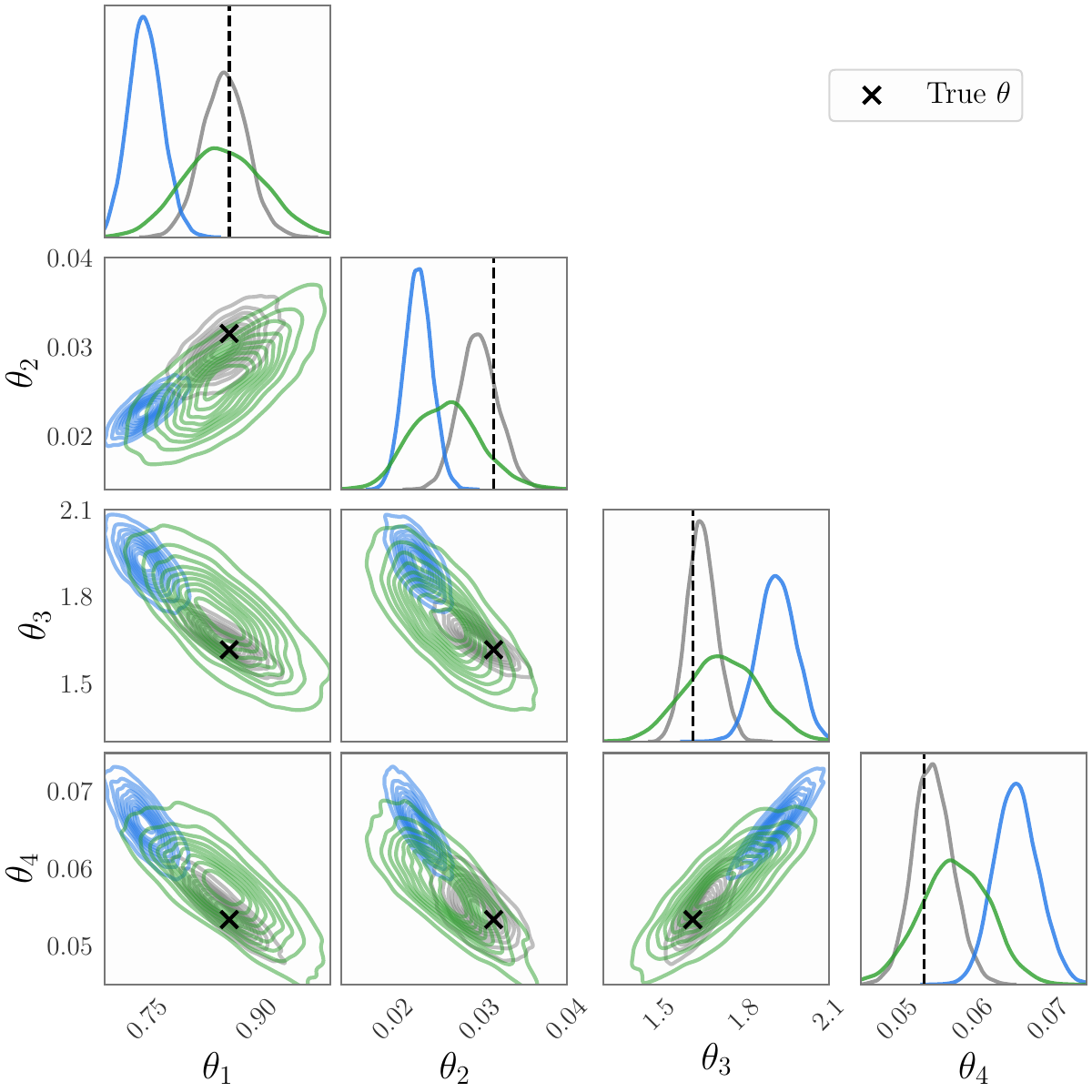}
    \caption{Posterior contour plots of Lotka--Volterra model obtained using NPE (\legendbox{npecolor}) and DRO-NPE (\legendbox{drocolor}) trained with $n = 1024$ simulations. The reference NPE posterior (\legendbox{gray}) is obtained using $n = 131072$.}\label{fig:posterior-comparison-app}
\end{figure}

\paragraph{Selected values of $\varepsilon$.}
\Cref{fig:dro-optimal-eps} reports the selected value $\varepsilon$, which we denote as $\varepsilon^\star$, obtained by minimising the validation estimate of $\kl_\mathrm{cal}^q$ in the benchmarking results of \Cref{sec:benchmarking}. Across most tasks, $\varepsilon^\star$ decreases as the simulation budget grows, as expected since larger simulation budgets require less conservativeness. The selected values also vary substantially across tasks, suggesting that a single choice of $\varepsilon$ determined solely by $n$ is unlikely to be sufficient.

\begin{figure}
    \centering
    \includegraphics[width=\linewidth]{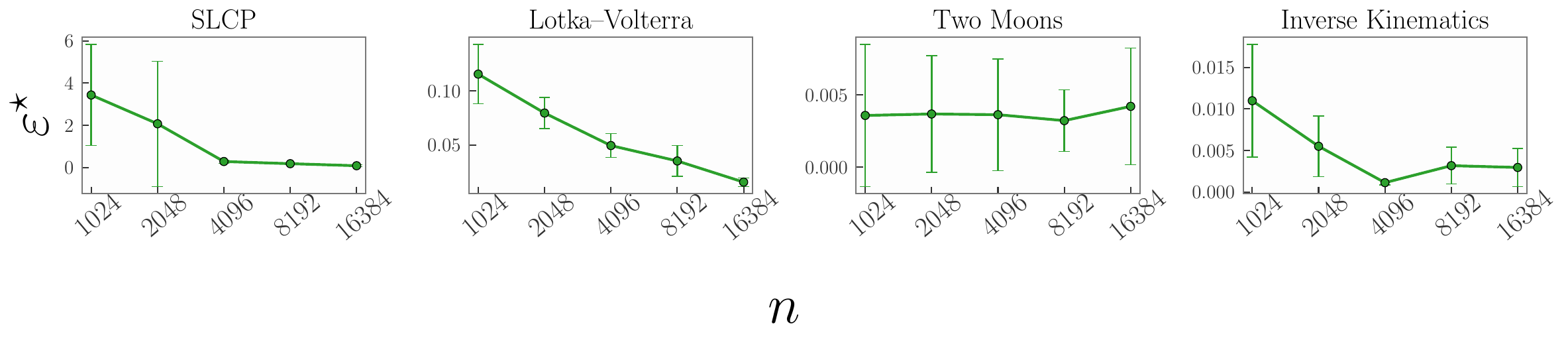}
    \caption{
    Selected DRO radius \(\varepsilon^\star\) across benchmark tasks and simulation budgets.
    Values are chosen by minimising validation \(\kl_\mathrm{cal}^q\); means and standard deviations are shown over five random seeds.
    }
    \label{fig:dro-optimal-eps}
\end{figure}

\begin{figure}
    \centering

    \begin{subfigure}{\linewidth}
        \centering
        \includegraphics[width=\linewidth]{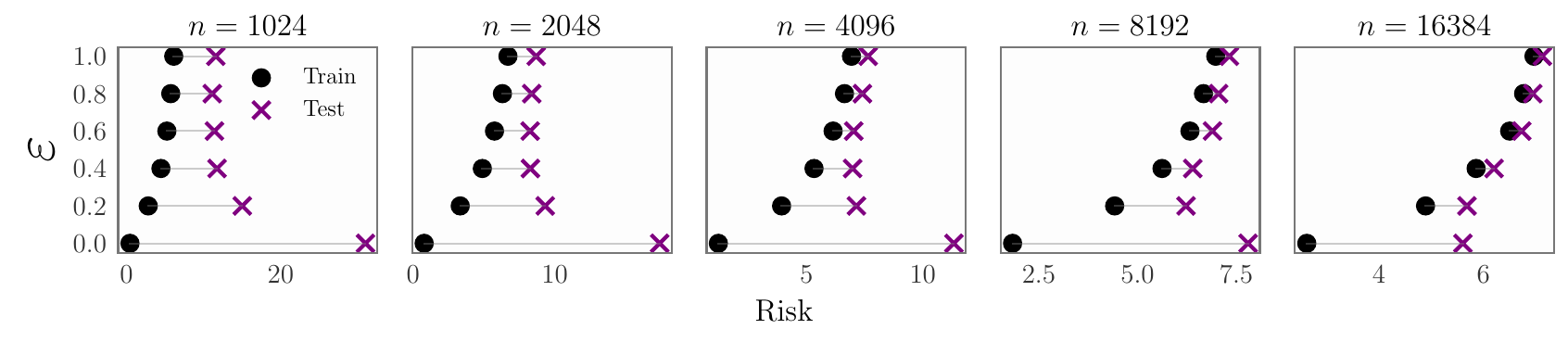}
         \caption{SLCP}
    \end{subfigure}
    \
    \begin{subfigure}{\linewidth}
        \centering
        \includegraphics[width=\linewidth]{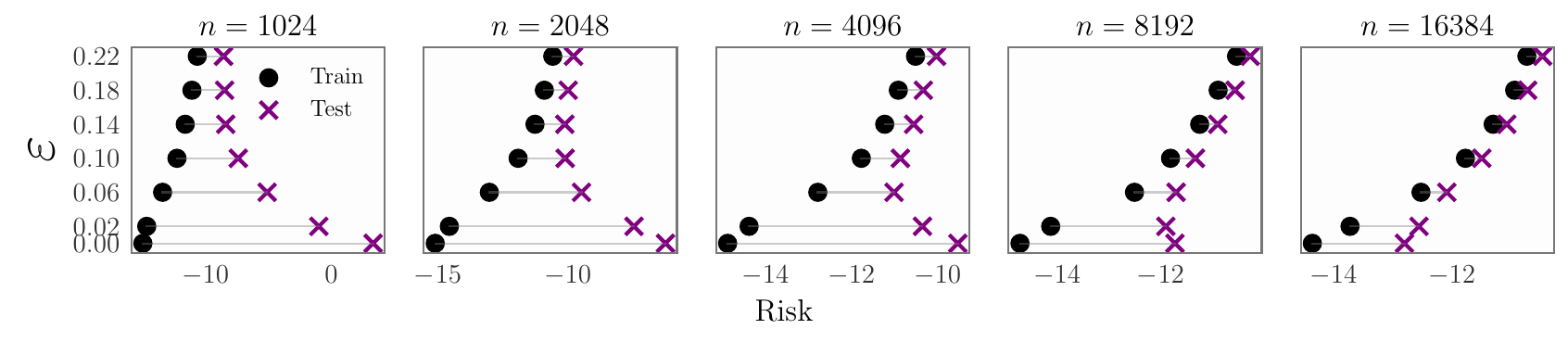}
        \caption{Lotka--Volterra}
    \end{subfigure}

    \caption{Training and test risk across $\varepsilon$ for SLCP and Lotka--Volterra at different simulation budgets. Increasing \(\varepsilon\) reduces the generalisation gap, but overly large values can worsen test risk. Means are shown over five random seeds.}
    \label{fig:dro-sensitivity-risk-app}
\end{figure}

\paragraph{Sensitivity of DRO-NPE to $\varepsilon$.}
We provide additional sensitivity analyses for the DRO radius \(\varepsilon\), extending the results of \Cref{sec:analysis} to different simulation budgets and including the SLCP task. We report training and test risks in \Cref{fig:dro-sensitivity-risk-app}, coverage curves in \Cref{fig:dro-sensitivity-cov-app}, and the behaviour of absolute miscoverage and \(\kl_\mathrm{cal}^q\) in \Cref{fig:dro-sensitivity-kl-app}.

As shown in \Cref{fig:dro-sensitivity-risk-app}, increasing $\varepsilon$ consistently reduces the generalisation gap across tasks and simulation budgets. For small $n$, this reduction is accompanied by improved test risk, whereas for larger $n$, overly large $\varepsilon$ can worsen test risk, indicating excessive conservativeness. 

\Cref{fig:dro-sensitivity-cov-app} shows the corresponding coverage behaviour. Increasing \(\varepsilon\) generally moves coverage curves upward, making the posterior more conservative. The value of \(\varepsilon\) that yields coverage closest to the diagonal decreases as \(n\) increases, consistent with the intuition that larger simulation budgets require less regularisation.

\begin{figure}
    \centering

    \begin{subfigure}{\linewidth}
        \centering
        \includegraphics[width=0.19\linewidth]{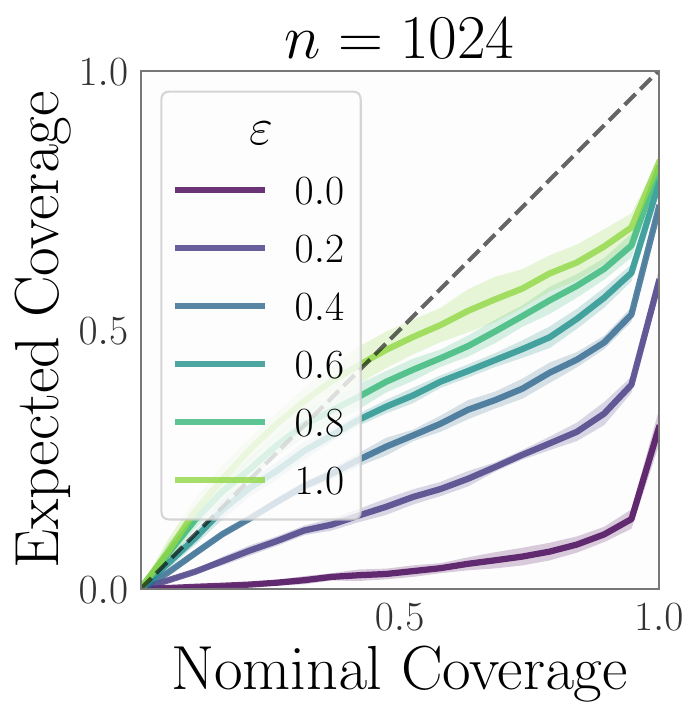}
        \includegraphics[width=0.19\linewidth]{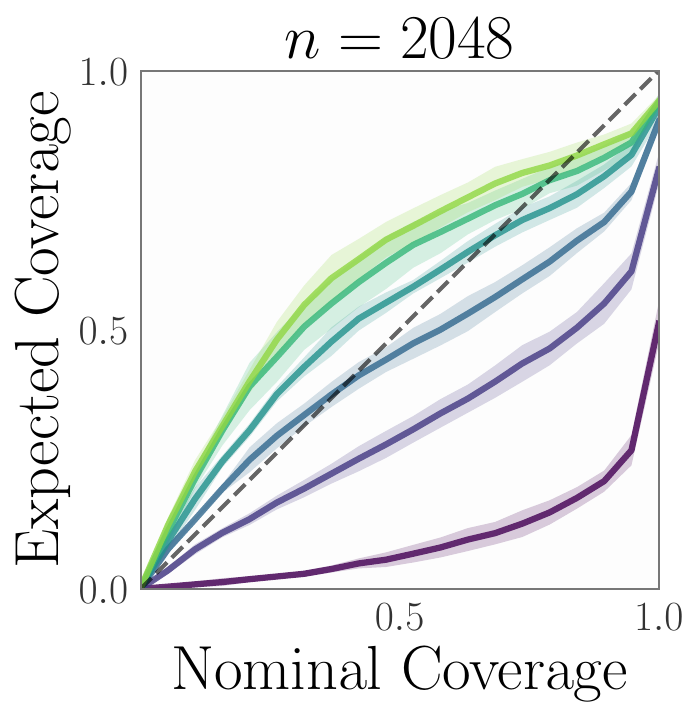}
        \includegraphics[width=0.19\linewidth]{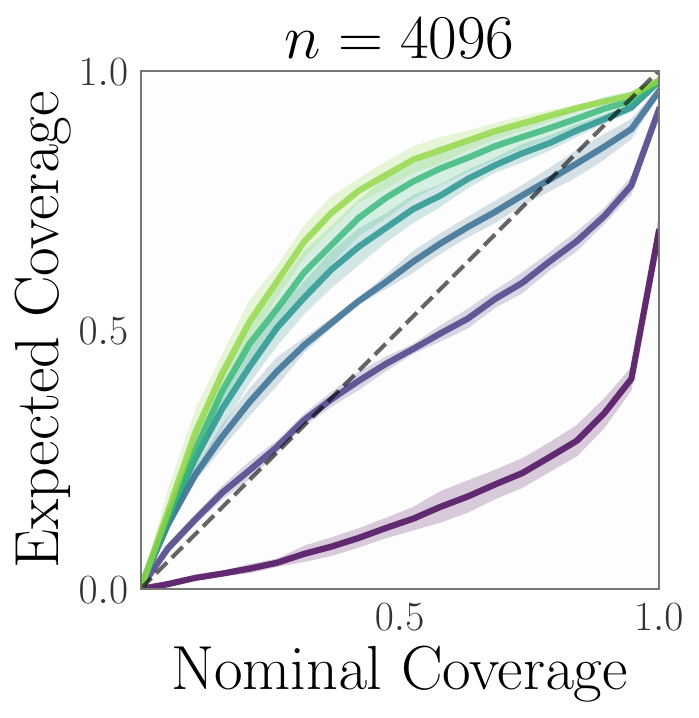}
        \includegraphics[width=0.19\linewidth]{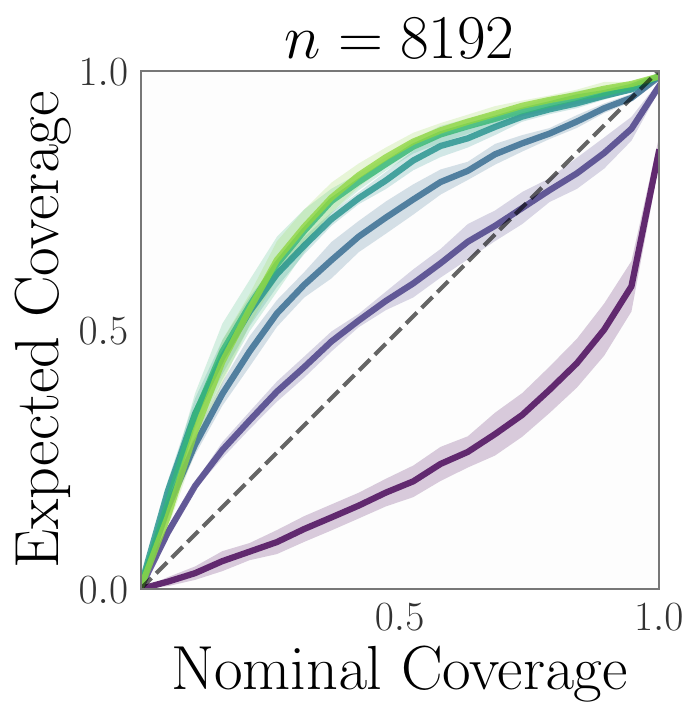}
        \includegraphics[width=0.19\linewidth]{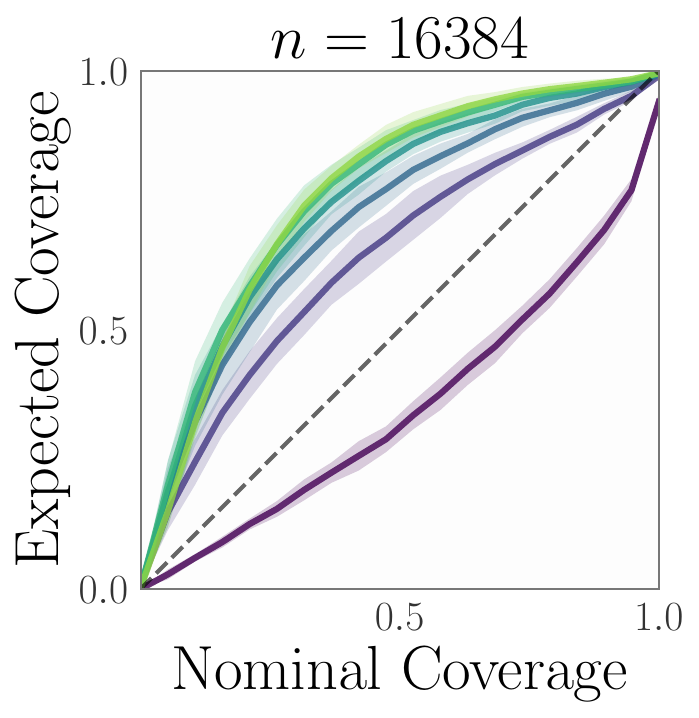}
         \caption{SLCP}
    \end{subfigure}

    \begin{subfigure}{\linewidth}
        \centering
        \includegraphics[width=0.19\linewidth]{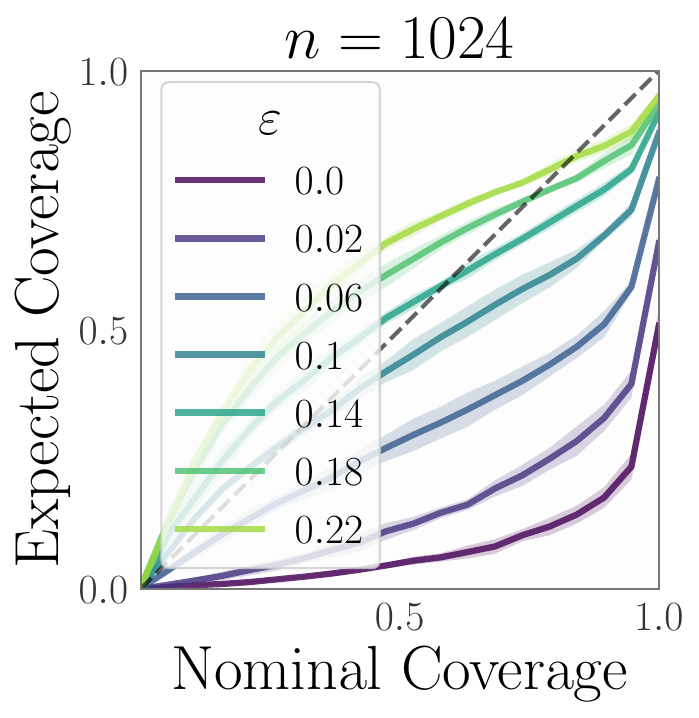}
        \includegraphics[width=0.19\linewidth]{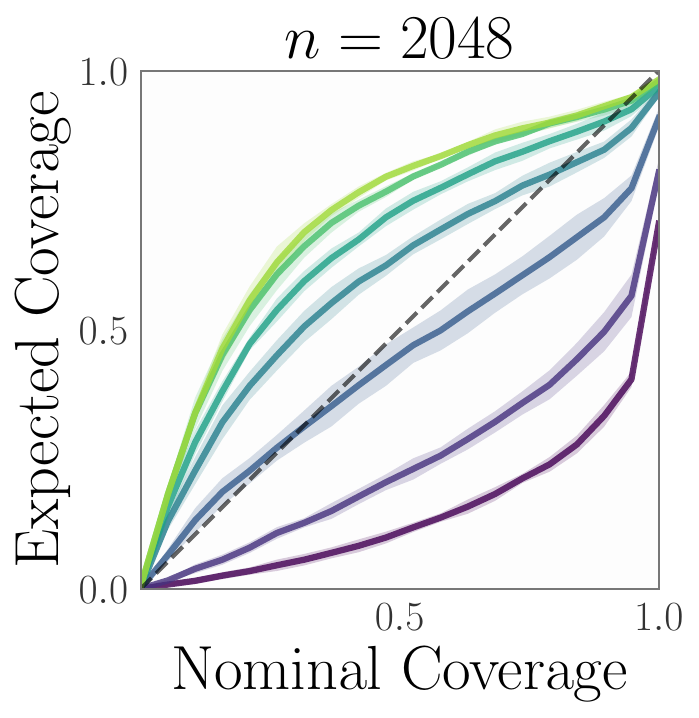}
        \includegraphics[width=0.19\linewidth]{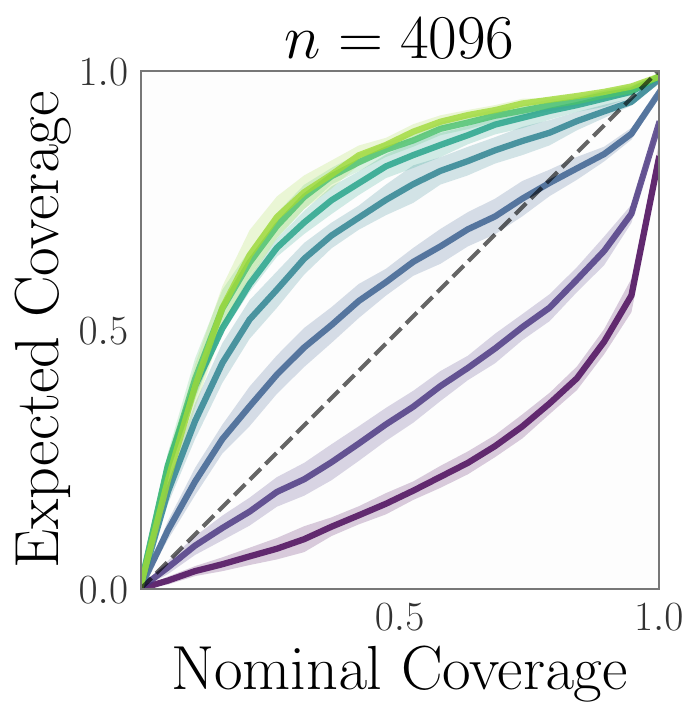}
        \includegraphics[width=0.19\linewidth]{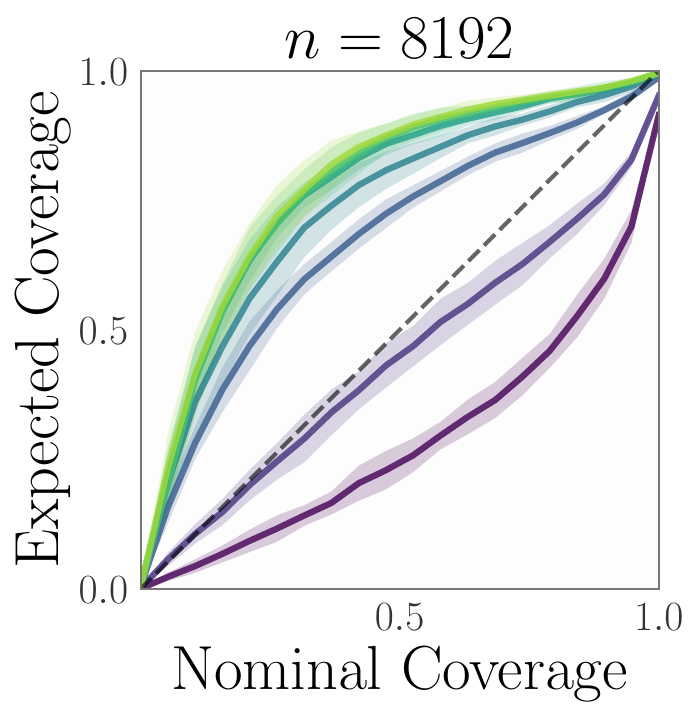}
        \includegraphics[width=0.19\linewidth]{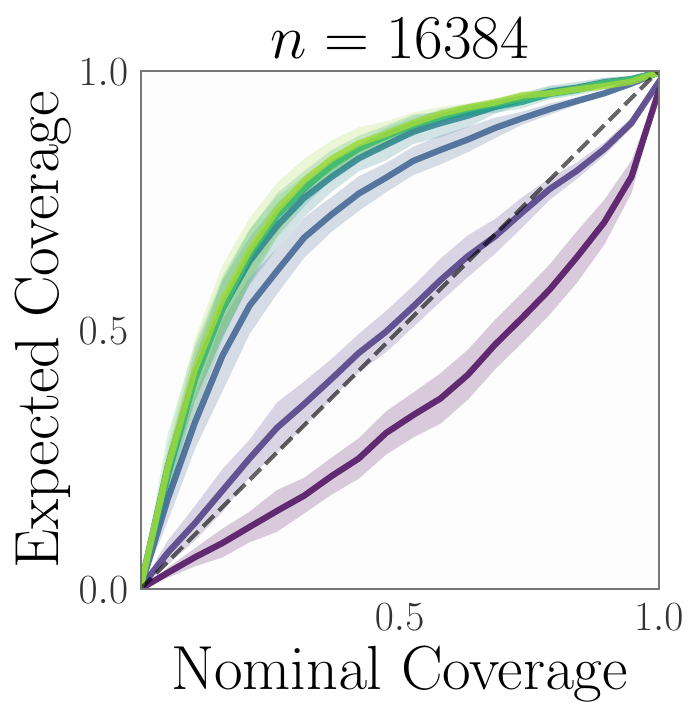}
        \caption{Lotka--Volterra}
    \end{subfigure}

    \caption{Coverage curves across \(\varepsilon\) for SLCP and Lotka--Volterra at different simulation budgets. Larger \(\varepsilon\) generally yields more conservative posteriors, while the radius closest to the diagonal decreases as the simulation budget grows.  Means and standard deviations are shown over five seeds.}
    \label{fig:dro-sensitivity-cov-app}
\end{figure}

Finally, \Cref{fig:dro-sensitivity-kl-app} compares absolute miscoverage at \(\alpha=0.05\) with \(\kl_\mathrm{cal}^q\) for different \(\varepsilon\). Absolute miscoverage decreases nearly monotonically with $\varepsilon$, favouring increasingly conservative posteriors.
By contrast, \(\kl_\mathrm{cal}^q\) is minimised at an intermediate value of \(\varepsilon\), making it more informative for selecting a radius that balances coverage and posterior quality.

\begin{figure}
    \centering

    \begin{subfigure}{\linewidth}
        \centering
        \includegraphics[width=\linewidth]{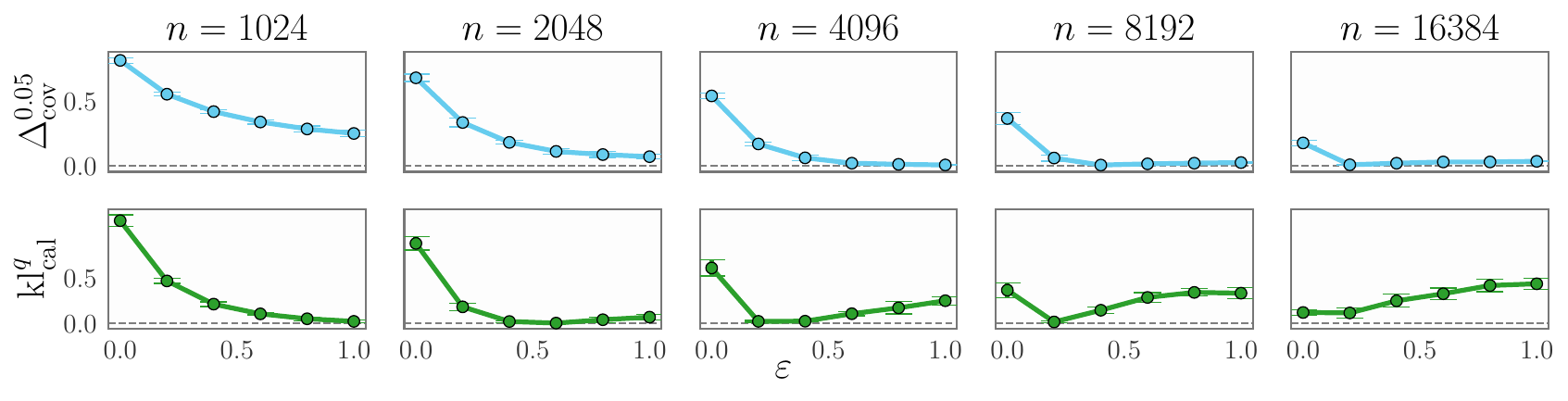}
         \caption{SLCP}
    \end{subfigure}
    \
    \begin{subfigure}{\linewidth}
        \centering
        \includegraphics[width=\linewidth]{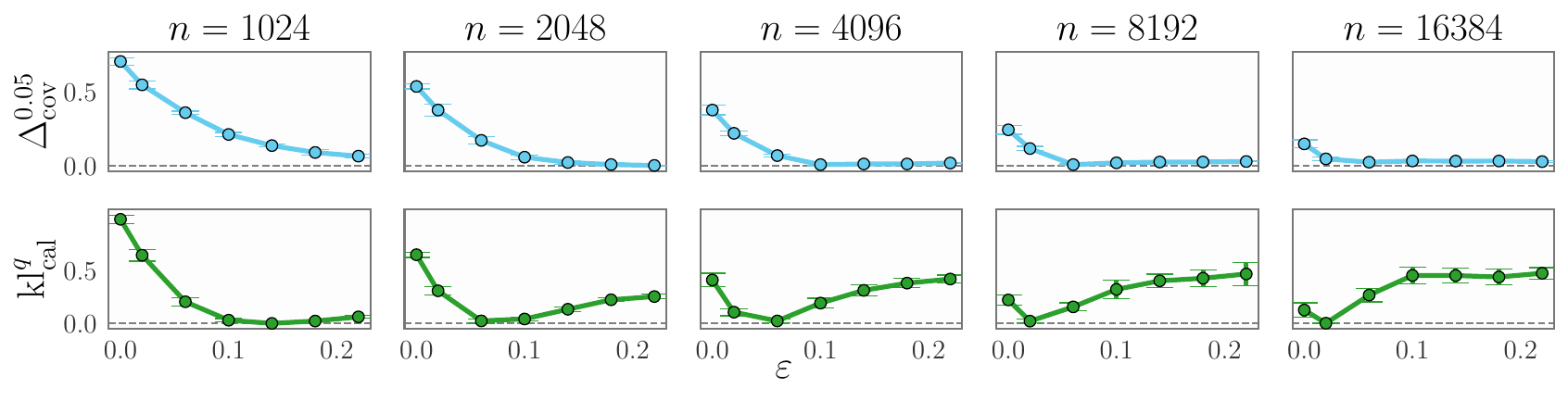}
        \caption{Lotka--Volterra}
    \end{subfigure}

    \caption{Absolute miscoverage at $\alpha = 0.05$ (\legendbox{drocoverror}) and KL-based miscalibration $\kl_\mathrm{cal}^q$ (\legendbox{drocolor}) across $\varepsilon$ for SLCP and Lotka--Volterra. Means and standard deviations are shown over five random seeds.}
    \label{fig:dro-sensitivity-kl-app}
\end{figure}

\paragraph{Choice of hyperparameter-selection metric.}
We report additional results for selecting $\varepsilon$ using validation metrics other than $\kl_\mathrm{cal}^q$. Across SLCP and Lotka--Volterra, selecting $\varepsilon$ via $\kl_\mathrm{cal}^q$ generally yields coverage curves closest to the diagonal while maintaining competitive NLPD.
In contrast, absolute miscoverage tends to select overly conservative posteriors, whereas \(\kl_\mathrm{cal}^{\|\btheta-\btheta_0\|}\) often leads to undercoverage and higher variability across runs. Selection by NLPD also performs well, but tends to produce slightly more conservative coverage curves than \(\kl_\mathrm{cal}^q\).

\begin{figure}
    \centering
    \begin{subfigure}{\linewidth}
        \centering
        \includegraphics[width=0.19\linewidth]{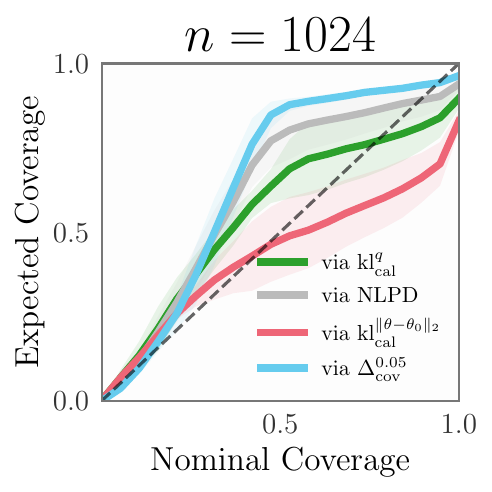}
        \includegraphics[width=0.19\linewidth]{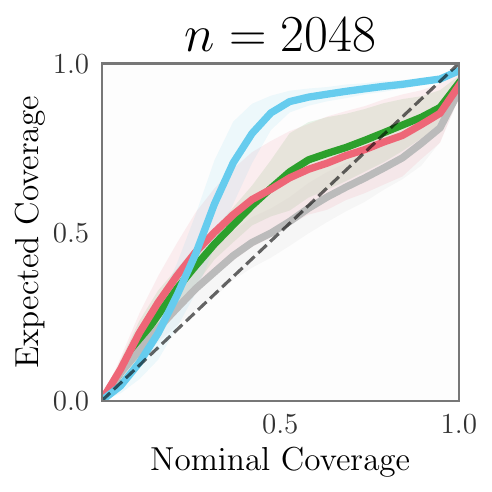}
        \includegraphics[width=0.19\linewidth]{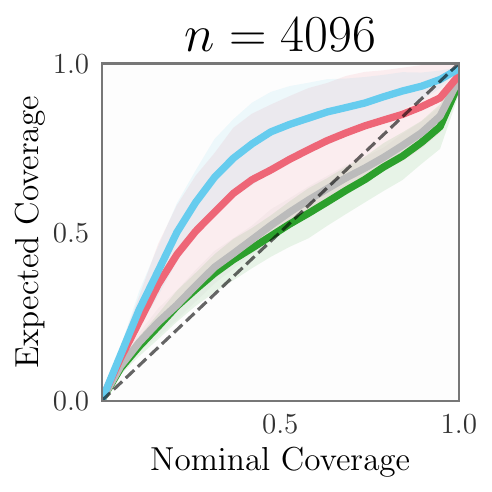}
        \includegraphics[width=0.25\linewidth]{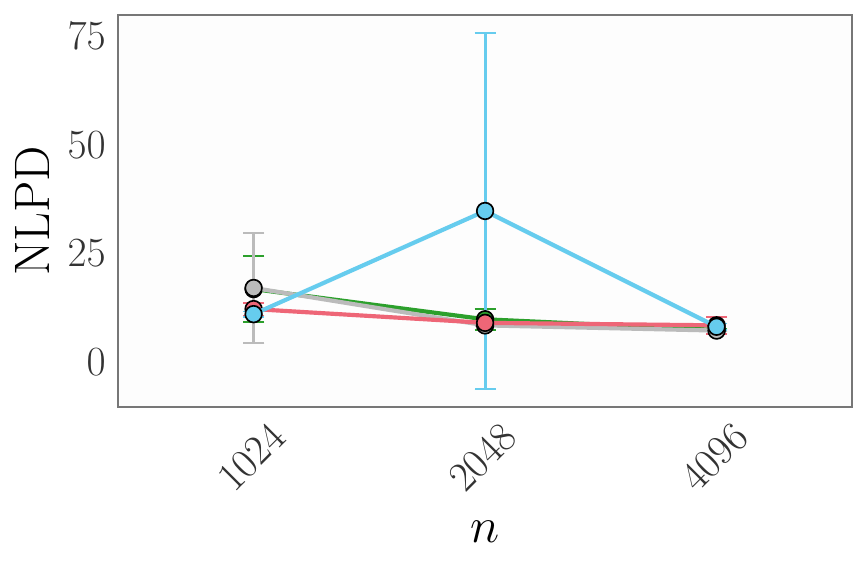}
        \caption{SLCP}
    \end{subfigure}
     \begin{subfigure}{\linewidth}
        \centering
        \includegraphics[width=0.19\linewidth]{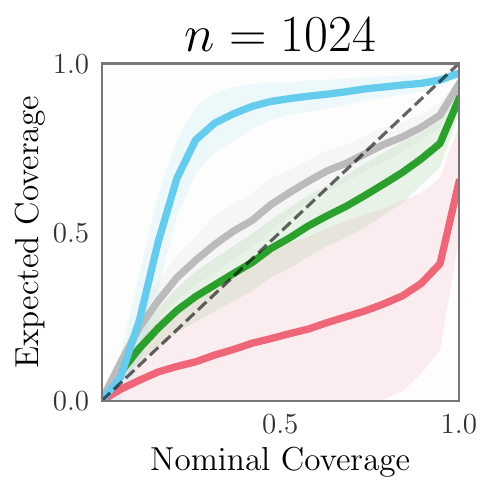}
        \includegraphics[width=0.19\linewidth]{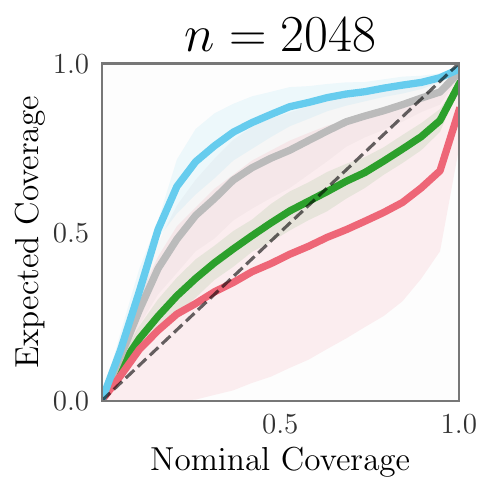}
        \includegraphics[width=0.19\linewidth]{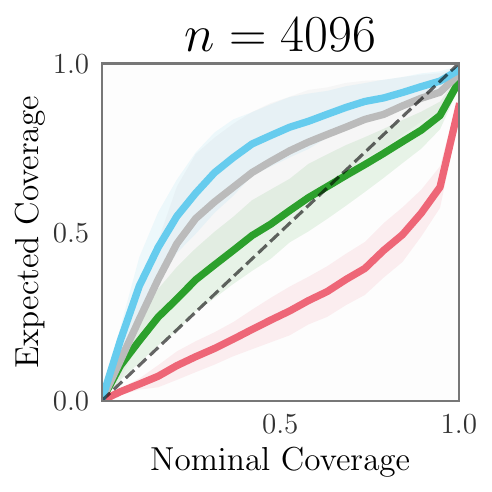}
        \includegraphics[width=0.25\linewidth]{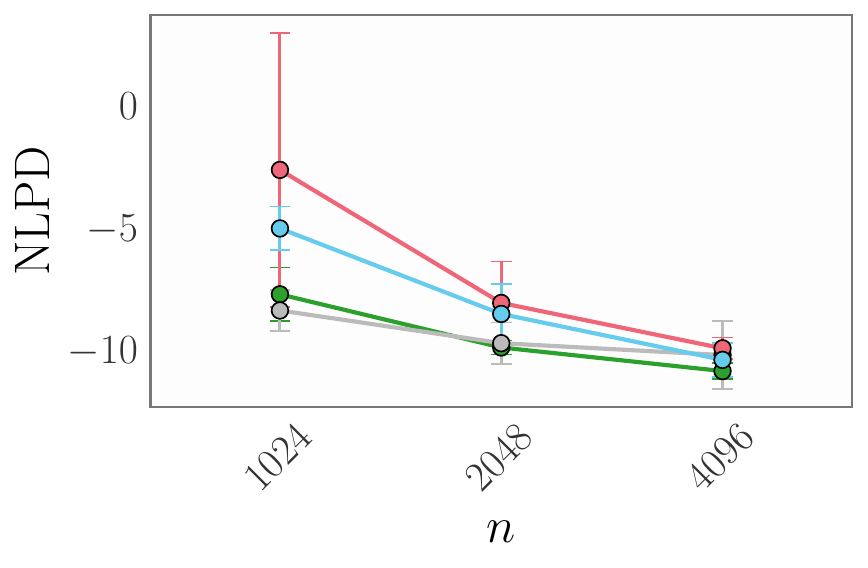}
        \caption{Lotka--Volterra}
    \end{subfigure}

    \caption{Effect of the validation metric used to select \(\varepsilon\) on SLCP and Lotka--Volterra. We compare selection by $\kl_\mathrm{cal}^q$ (\legendbox{drocolor}), NLPD (\legendbox{dronlpdcolor}), $\kl_\mathrm{cal}^{\| \theta- \theta_0\|_2}$ with $\theta_0$ drawn from the prior (\legendbox{drokldistcolor}), and absolute miscoverage at $\alpha = 0.05$ (\legendbox{drocoverror}). Selecting by \(\kl_\mathrm{cal}^q\) generally gives coverage closest to the diagonal while maintaining competitive NLPD. Means and standard deviations are shown over five seeds.}
    \label{fig:coverage_dro_different_metrics}
\end{figure}

\paragraph{Hyperparameter tuning on baseline methods.} 
We now use \(\kl_\mathrm{cal}^q\) to tune the regularisation hyperparameter of the baseline methods in settings where they exhibit poor coverage in \Cref{sec:benchmarking}. For CR-NPE, that is Lotka--Volterra with \(n=1024\), and for Bal-NPE, that is SLCP with \(n=1024\). 

For CR-NPE, we search over \([1,40]\), following the sensitivity range used by \citet{falkiewicz2023calibrating}. Although \citet{falkiewicz2023calibrating} analyse sensitivity using both NLPD and the area under the coverage curve, using two metrics for model selection would require an additional weighting choice; we therefore use \(\kl_\mathrm{cal}^q\) as a single validation criterion. For Bal-NPE, we search over \([100,1000]\), since the default \(\lambda=100\) appears too small in \Cref{fig:benchmarking-main}. The selected values are \(\lambda=29.6\pm15.8\) for CR-NPE and \(\lambda=415.4\pm303.0\) for Bal-NPE.

\Cref{fig:tune-baseline} shows the results over five random seeds. For CR-NPE, tuning the hyperparameter yields the same coverage results as before, while the NLPD performance degrades slightly, both in terms of the average value and the variance across runs. This is consistent with \citet{falkiewicz2023calibrating}, who find that CR-NPE is relatively insensitive to \(\lambda\). For Bal-NPE, tuning substantially improves NLPD but only modestly improves coverage, leaving the method undercovered. Overall, these results suggest that \(\kl_\mathrm{cal}^q\) could potentially be useful for tuning other conservative NPE objectives, whilst highlighting that hyperparameter tuning alone does not remove the coverage limitations of these baselines.

\begin{figure}
    \centering
    \includegraphics[width=0.24\linewidth]{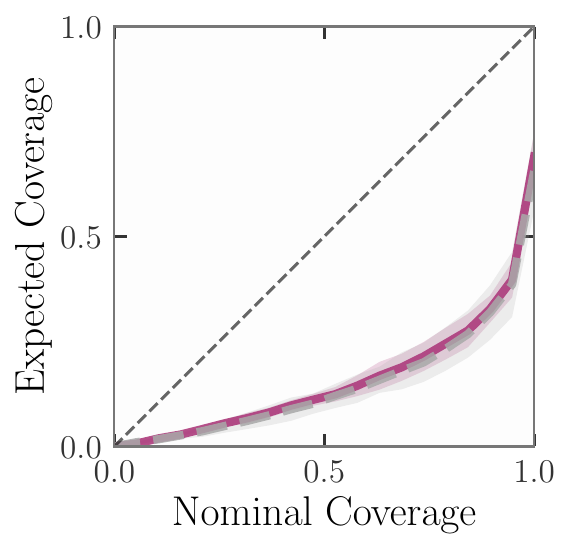}
    \raisebox{0.5em}{\includegraphics[width=0.24\linewidth]{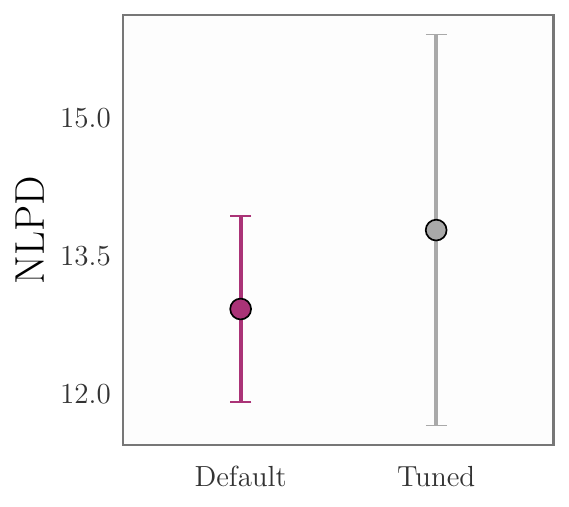}}
    \includegraphics[width=0.24\linewidth]{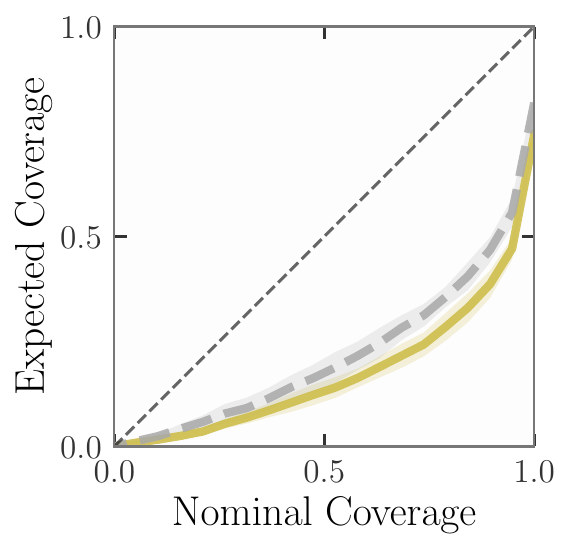}
    \raisebox{0.5em}{\includegraphics[width=0.235\linewidth]{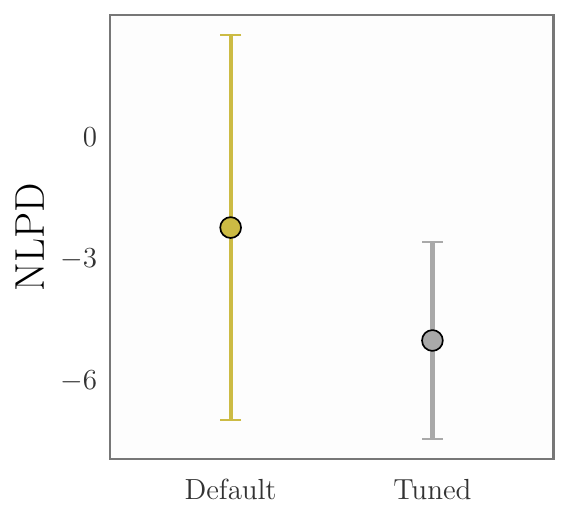}}
    \caption{Effect of tuning  regularisation hyperparameter using KL-based miscalibration \(\kl_\mathrm{cal}^q\). Default settings of CR-NPE (\legendbox{calnpecolor}) and Bal-NPE (\legendbox{balancecolor}) are compared with their tuned variants (\legendbox{dronlpdcolor}). Coverage curves and NLPD are shown with means and standard deviations over five random seeds. }
    \label{fig:tune-baseline}
\end{figure}

\paragraph{Effect of architecture.}
An alternative way to reduce generalisation error is to
adjust the capacity of the model (normalising flow in our case) used as a posterior density estimator.
We therefore investigate the effect of normalising-flow
capacity on Lotka--Volterra task with $n=1024$, using five
random seeds.
We compare three architectures: Small, with MLP widths
$(32,32)$ and three coupling layers; Default, with widths
$(128,128)$ and six coupling layers; and Large, with widths
$(512,512)$ and twelve coupling layers.
Figure~\ref{fig:architecture-ablation} reports coverage and
NLPD for DRO-NPE, standard NPE, Bal-NPE, and CR-NPE.

Smaller architectures tend to yield lower NLPD and 
improve coverage in this setting, possibly because reduced
capacity limits overfitting in the small-sample regime.
Based on the mean results, DRO-NPE and CR-NPE appear less
sensitive to architecture than standard NPE, whose
performance deteriorates more as capacity increases.

\begin{figure}[t]
    \centering
    \begin{subfigure}{0.65\linewidth}
        \centering
        \includegraphics[width=\linewidth]{
        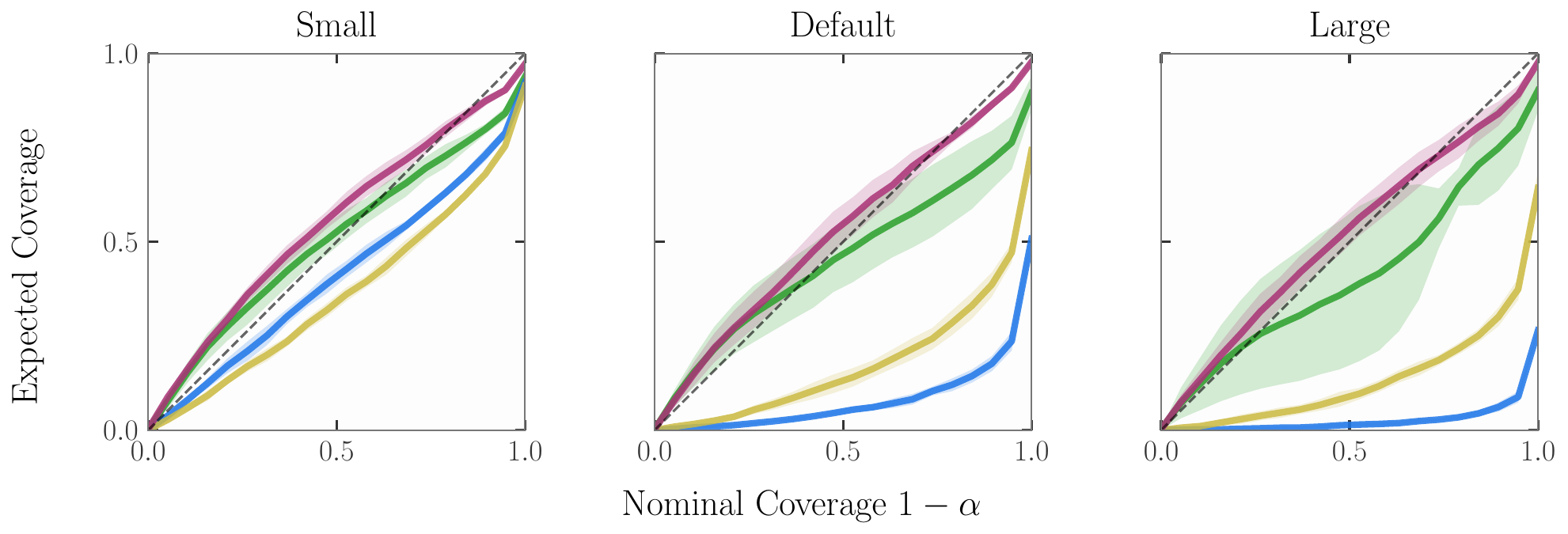}
        \caption{Coverage curves}
    \end{subfigure}
    \hfill
    \begin{subfigure}{0.32\linewidth}
        \centering
        \includegraphics[width=\linewidth]{
        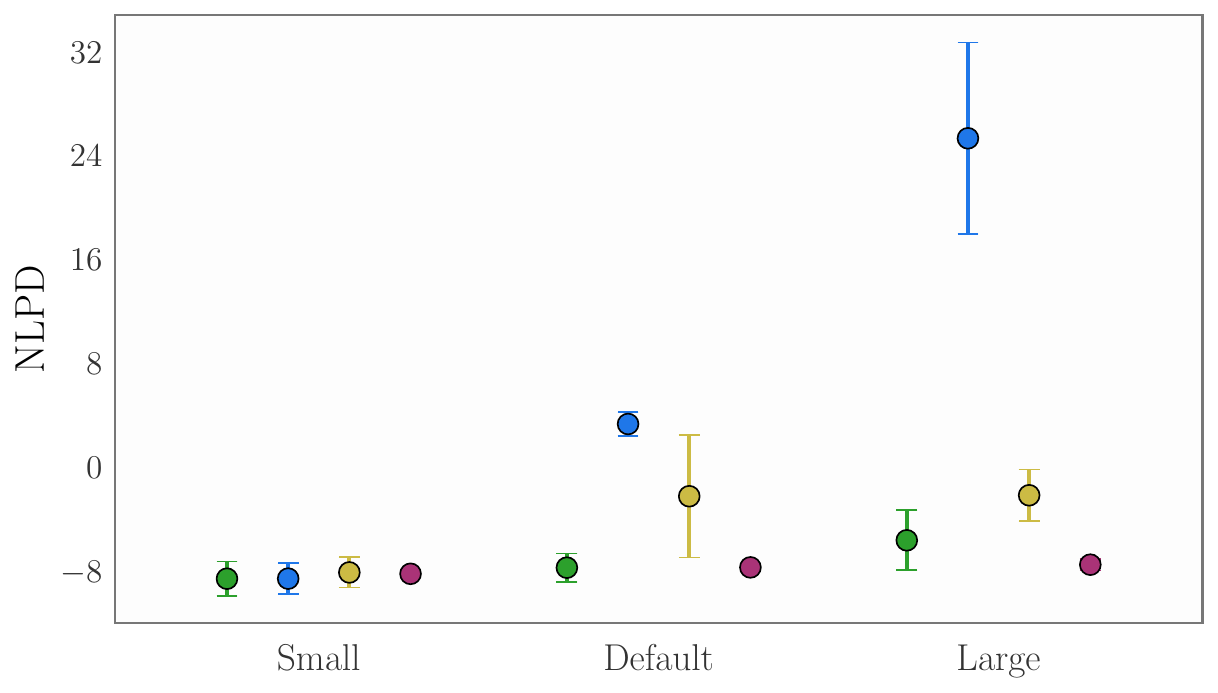}
        \caption{NLPD}
    \end{subfigure}
    \caption{
    Effect of normalising-flow architecture on Lotka--Volterra
    with $n=1024$.
    Small, Default, and Large architectures use MLP widths
    $(32,32)$, $(128,128)$, and $(512,512)$, with $3$, $6$,
    and $12$ coupling layers, respectively.
    We compare DRO-NPE (\legendbox{drocolor}), standard NPE (\legendbox{npecolor}), Bal-NPE (\legendbox{balancecolor}), and CR-NPE (\legendbox{calnpecolor}). Means and standard deviations are shown over five seeds.}
    \label{fig:architecture-ablation}
\end{figure}

\paragraph{Effect of early stopping.}
Another common way to reduce generalisation error and mitigate overfitting is early stopping. Although its effect on regularisation is less well understood theoretically, except in relatively simple settings such as linear models \citep{sonthalia2024regularization}, it may still help in practice by reducing the generalisation gap. To evaluate this empirically, we apply early stopping to all four methods on
Lotka--Volterra task with $n=1024$ and the default architecture.
Figure~\ref{fig:early-stopping-ablation} compares coverage and
NLPD with and without early stopping over five random seeds.

Early stopping shifts the methods towards more conservative
posteriors overall and improves NLPD for all methods except
CR-NPE in this setting.
The improvement is particularly pronounced for standard NPE.
DRO-NPE remains comparatively well calibrated with and
without early stopping, while maintaining competitive NLPD.

\begin{figure}[t]
    \centering
    \begin{subfigure}{0.55\linewidth}
        \centering
        \includegraphics[width=\linewidth]{
        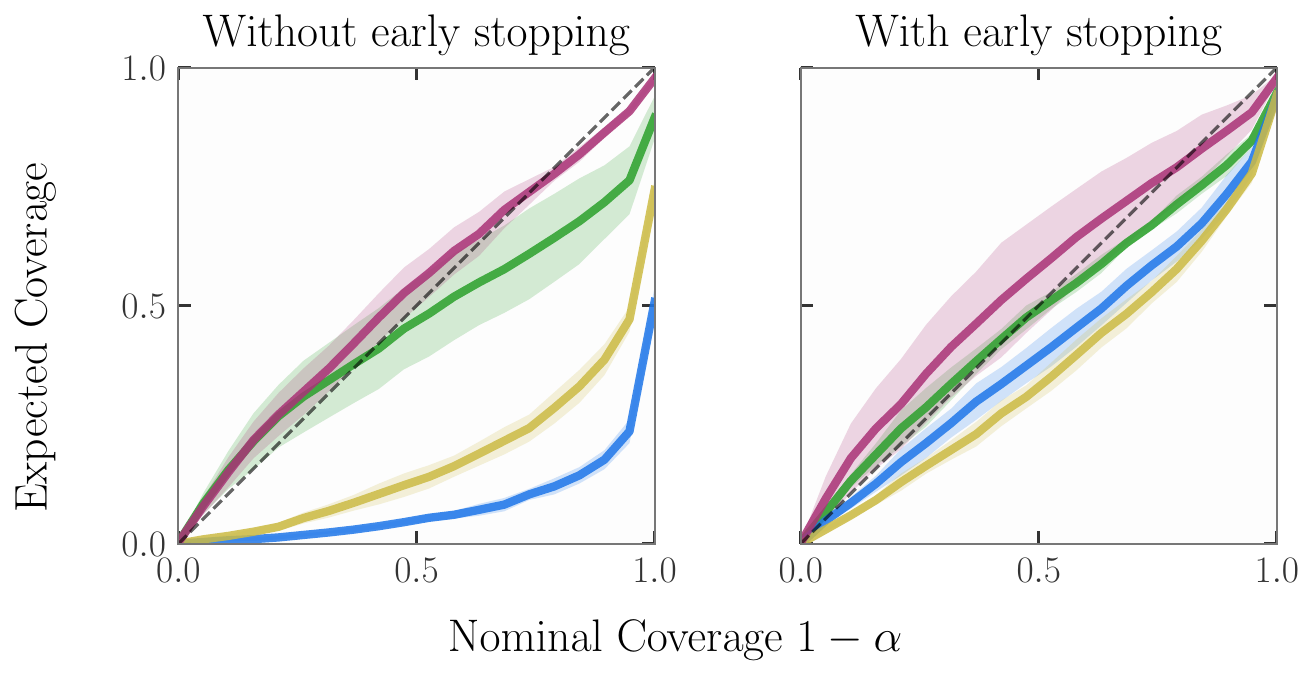}
        \caption{Coverage curves}
    \end{subfigure}
    \hfill
    \begin{subfigure}{0.42\linewidth}
        \centering
        \includegraphics[width=\linewidth]{
        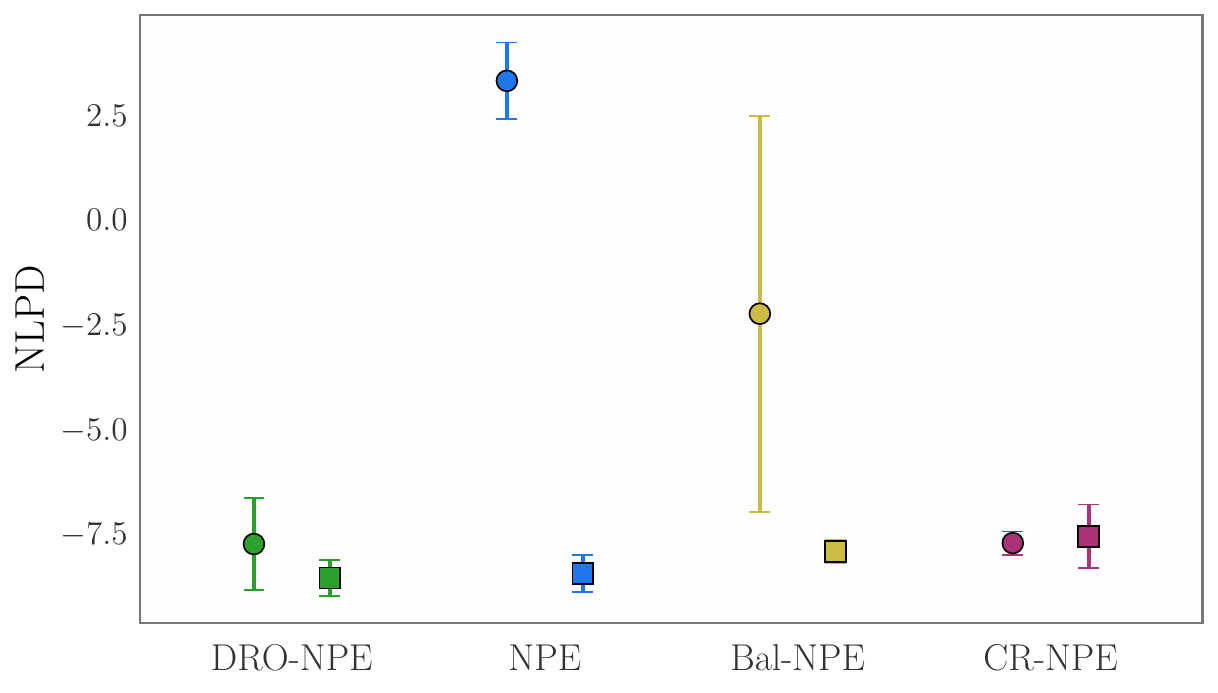}
        \caption{NLPD}
    \end{subfigure}
    \caption{
    Effect of early stopping on Lotka--Volterra with $n=1024$
    and the default architecture.
    We compare DRO-NPE (\legendbox{drocolor}), standard NPE (\legendbox{npecolor}), Bal-NPE (\legendbox{balancecolor}), and CR-NPE (\legendbox{calnpecolor})
    with and without early stopping.
    Means and standard deviations are shown over five seeds.
    }
    \label{fig:early-stopping-ablation}
\end{figure}

\paragraph{A more conservative variant of DRO-NPE.}

In our main experiments, we select $\varepsilon$ by minimising $\kl_\mathrm{cal}^q$.
Although this criterion yields good overall coverage, it does not explicitly encourage coverage curves to lie above the diagonal; rather, it penalises deviations from the diagonal in either direction. We therefore consider a more conservative
variant of DRO-NPE, obtained by slightly modifying the selection criterion. Specifically, instead of minimising $\kl_\mathrm{cal}^q$, we minimise
$\kl_\mathrm{cal}^q - \gamma \tilde{\varepsilon}$, where
$$\tilde{\varepsilon}
=
\frac{\log \varepsilon - \log \varepsilon_\mathrm{min}}
{\log \varepsilon_\mathrm{max} - \log \varepsilon_\mathrm{min}}$$ is the log-normalised value of $\varepsilon$. The additional term favours larger values of $\varepsilon$, and hence more conservative posteriors. Here, $\gamma \geq 0$ is a user-specified hyperparameter controlling the trade-off between calibration and additional conservativeness. 

We illustrate this variant in \Cref{fig:dro-conservative}, reporting the resulting coverage and NLPD for $\gamma = 0.8$. The modified criterion yields a more conservative posterior approximation. Interestingly, it also achieves a slightly lower NLPD, suggesting that in this case the value of $\varepsilon$ selected using the validation set with \(\kl_\mathrm{cal}^q\) alone may have been slightly smaller than optimal.

\begin{figure}
    \centering
    \begin{subfigure}{0.33\linewidth}
        \centering
        \includegraphics[width=0.8\linewidth]{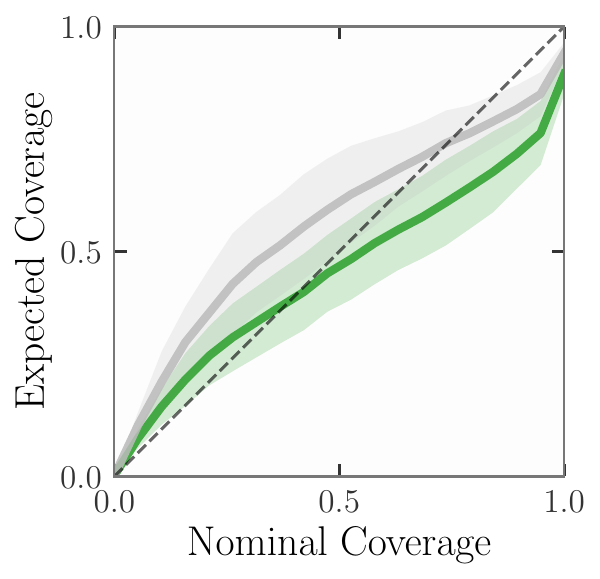}
        \caption{Coverage curve}
    \end{subfigure}
    \begin{subfigure}{0.33\linewidth}
        \centering
        \raisebox{1.5em}{\includegraphics[width=0.72\linewidth]{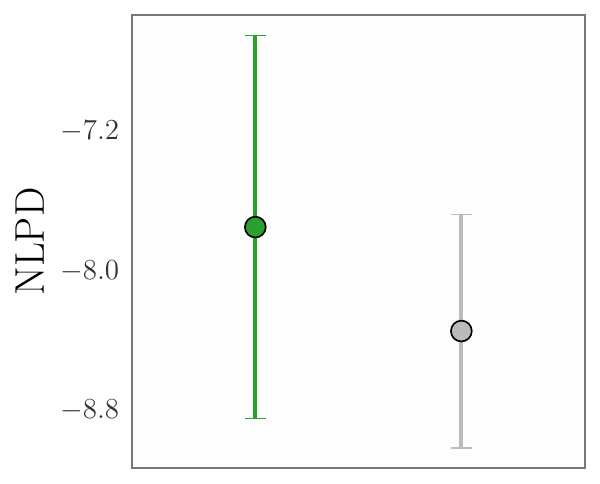}}
        \caption{NLPD}
    \end{subfigure}
    \caption{Comparison of standard DRO-NPE, with $\varepsilon$ selected by $\kl_\mathrm{cal}^q$ (\legendbox{drocolor}), against a more conservative variant with $\varepsilon$ selected by $\kl_\mathrm{cal}^q - \gamma \tilde \varepsilon$ (\legendbox{dronlpdcolor}), using $\gamma = 0.8$.  Results are shown for the Lotka--Volterra task with $n = 1024$. Means and standard deviations over five random seeds are shown.}
    \label{fig:dro-conservative}
\end{figure}



\end{document}